\documentclass[11pt]{article}
\usepackage{acl}

\usepackage{times}
\usepackage{latexsym}
\usepackage[T1]{fontenc}
\usepackage[utf8]{inputenc}
\usepackage{microtype}
\usepackage{inconsolata}
\usepackage{graphicx}
\usepackage{booktabs}
\usepackage{multirow}
\usepackage{amsmath}
\usepackage{amssymb}
\usepackage[table]{xcolor}
\usepackage{enumitem}
\usepackage{algorithm}
\usepackage{algorithmic}
\usepackage{amsthm}
\usepackage{listings}
\usepackage[most]{tcolorbox}
\definecolor{boxBgOurs}{HTML}{EAF3FB}
\definecolor{boxFrOurs}{HTML}{377EB8}
\definecolor{boxBgAlc}{HTML}{F2EFEA}
\definecolor{boxFrAlc}{HTML}{8C7B5A}
\definecolor{boxBgDS}{HTML}{F5EDED}
\definecolor{boxFrDS}{HTML}{B05B5B}
\definecolor{boxBgAna}{HTML}{F0EAF4}
\definecolor{boxFrAna}{HTML}{6A4CA0}
\definecolor{boxBgImp}{HTML}{EAF4EC}
\definecolor{boxFrImp}{HTML}{4B8B5A}
\definecolor{boxBgRef}{HTML}{FBF1E4}
\definecolor{boxFrRef}{HTML}{C97A30}
\definecolor{boxBgGen}{HTML}{EAF3FB}
\definecolor{boxFrGen}{HTML}{377EB8}
\newtcolorbox{methodcard}[2][]{enhanced, breakable,
  colback=#1!8, colframe=#1!85!black, fonttitle=\bfseries\small,
  boxrule=0.6pt, arc=2pt, left=4pt, right=4pt, top=3pt, bottom=3pt,
  title={#2}, fontupper=\small}
\lstdefinestyle{pylf}{language=Python, basicstyle=\ttfamily\scriptsize,
  keywordstyle=\color{blue!60!black}\bfseries,
  stringstyle=\color{red!60!black},
  commentstyle=\color{gray!70!black}\itshape,
  showstringspaces=false, breaklines=true, columns=fullflexible,
  belowskip=-2pt, aboveskip=2pt, escapeinside={(*}{*)}}

\theoremstyle{definition}

\theoremstyle{plain}
\newtheorem{proposition}{Proposition}

\title{EvoPool: Evolutionary Programmatic Annotation \\
       for Label-Efficient Specialized Supervision}

\author{
  Tianyi Xu$^{1,2}$ \quad Yaolun Zhang$^{1}$ \quad Xuan Ouyang$^{2}$ \quad Huazheng Wang$^{1}$\thanks{Corresponding author.} \\[3pt]
  $^{1}$Oregon State University \quad $^{2}$University of Wisconsin--Madison \\[3pt]
  \texttt{\{xut2, zhanyaol, huazheng.wang\}@oregonstate.edu} \quad \texttt{xouyang7@wisc.edu}
}

\begin{document}
\maketitle

\begin{abstract}
  Large language models excel at general tasks but underperform smaller supervised models in specialized, high-stakes domains where training labels are costly. We address this regime with \textbf{EvoPool}, an evolutionary multi-agent framework inspired by Darwinian evolution. Three specialized agents iteratively propose executable annotator code, a small validation set provides a fitness signal, and a deterministic gate keeps only annotators that pass viability, diversity, and marginal-contribution checks across generations. Pool votes are mapped to soft training labels by EvoAgg, a text-aware aggregator combining semantic features with annotator-vote features. The authored pool runs at near-zero per-example cost and is $4500$ to $31000\times$ faster than LLM annotation on 100K examples. Across 7 of 8 LLM-weak specialized and complex tasks spanning biomedical relation extraction, legal-clause classification, complex reasoning, and dense multi-label biomedical classification, EvoPool beats the strongest LLM annotation baseline by an average $+0.141$ macro-F1, peaking at $+0.301$ on ChemProt and $+0.265$ on PubMed. Code is available at: \url{https://github.com/tianyi0216/EvoPool}.
  \end{abstract}

\section{Introduction}

\begin{figure*}[!t]
  \centering
  \includegraphics[width=\linewidth]{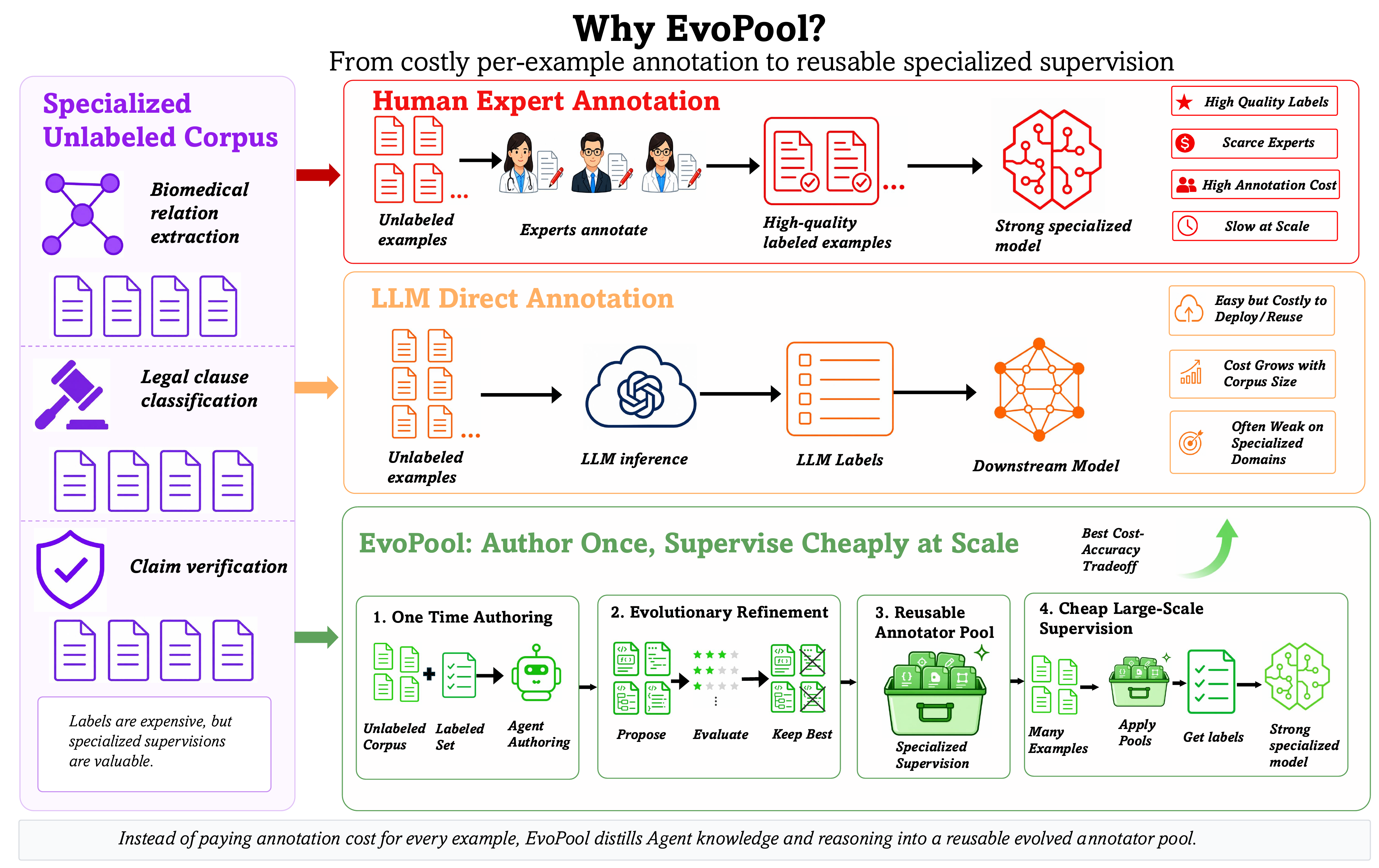}
  \vspace{-15pt}
\caption{\textbf{Why EvoPool?} Three regimes for obtaining specialized supervision when training labels are unavailable. \textbf{(1) Human-expert annotation} provides high-quality labels but depends on scarce experts and scales poorly. \textbf{(2) LLM-direct annotation} is easy to deploy, but requires one LLM call per example, making cost grow with corpus size while often remaining unreliable on specialized domains. \textbf{(3) EvoPool} amortizes LLM reasoning into a reusable pool of executable annotators. Through one-time evolutionary authoring, candidate annotators are proposed, evaluated, and refined into dataset-specific supervision that can be applied cheaply across the full unlabeled corpus to train a strong specialized model.}
  \vspace{-15pt}
  \label{fig:motivation}
\end{figure*}

Specialized, high-stakes domains like biomedical relation extraction \citep{krallinger2017chemprot,herrero2013ddi}, legal-clause classification \citep{zhang2024stronger}, and scientific claim verification \citep{wadden2020scifact} require large amounts of accurately labeled training data. A small model \citep{liu2019roberta} fine-tuned on a few thousand such examples reaches strong macro-F1, but those labels require clinicians, lawyers, or scientific reviewers and routinely cost tens to hundreds of thousands of dollars per dataset. Annotation budgets typically allow only a small validation set, leaving the training corpus unlabeled.

Large language models, with their strong general task knowledge and code-synthesis capability, are a promising route to scalable annotation in such domains. Two natural variants have emerged, but both fall short. \emph{LLM annotation} labels each training example via a per-example LLM inference \citep{zheng2023judging,wang2021gpt3dataaug}. It is slow and expensive at deployment because cost scales linearly with corpus size, and on specialized domains accuracy still trails task-specific models \citep{xu2025specialized}: on ChemProt \citep{krallinger2017chemprot}, gpt-4o-mini annotation reaches only $0.28$ test macro-F1. \emph{LLM annotator synthesis} methods such as Alchemist \citep{alchemist}, DataSculpt \citep{datasculpt}, and EXPONA \citep{guo2024expona} author a pool of executable labeling functions in one LLM shot, building on the data-programming abstraction of Snorkel \citep{ratner2017snorkel,ratner2016dataprogramming}. This is cheap to deploy at inference but the pool inherits the LLM's a priori biases over the label space, ignores observed per-class weakness, and collapses coverage on rare classes.

To address these gaps, we ask: can the annotator pool itself evolve to fit each domain rather than being authored once and frozen? Our answer is \textbf{EvoPool}, an evolutionary multi-agent framework that draws on Darwinian evolution \citep{darwin1859origin} and recent evolutionary methods for data and agents \citep{mi2026datadarwinism,yuan2024evoagent}. Each generation, three specialized agents propose new executable annotators targeting weak classes, the validation set provides a fitness signal, and a deterministic gate keeps annotators that pass viability, diversity, and marginal-contribution checks. Only the surviving pool persists to the next generation, accumulating across $K$ generations as auditable Python rules. This per-class feedback and cross-generation accumulation is what one-shot annotator synthesis and per-example LLM annotation cannot provide, and it lets EvoPool keep improving where they stall. Pool votes are then mapped to soft training labels by EvoAgg, a text-aware aggregator combining a low-dimensional sentence-BERT embedding \citep{reimers2019sbert} with annotator-vote features. Once authored, the pool runs at near-zero per-example cost while LLM annotation scales linearly with the corpus. Figure~\ref{fig:motivation} contrasts these three regimes for specialized supervision.

We evaluate EvoPool on 10 datasets across 4 task categories. On 7 of 8 LLM-weak specialized and complex tasks, EvoPool beats LLM annotation by an average $+0.141$ test macro-F1, with the largest gain exceeding $+0.30$ on biomedical relation extraction. On the remaining 2 tasks, EvoPool matches or beats LLM annotation while still outperforming the stronger of the LLM annotator synthesis methods by $0.36$ to $0.62$. We further evaluate the framework across $4$ LLM backbones and $3$ downstream models that vary in size and architecture, confirming that the cost-accuracy advantage is robust to the choice of agent backbone and downstream architecture. Our contributions are:

\begin{itemize}[leftmargin=13pt,itemsep=2pt]
\item \textbf{EvoPool}, a Darwinian-inspired evolutionary multi-agent framework that authors a pool of executable annotators under deterministic pool-level selection, paired with \textbf{EvoAgg}, a text-aware aggregator that lifts macro-F1 over majority vote by $+0.10$ to $+0.43$ on long-tail and dense multi-label datasets.
\item \textbf{Task-aware programmatic supervision for complex structured tasks}: to our knowledge, the first programmatic-annotation framework evaluated across both complex reasoning with explicit claim/evidence comparison and dense multi-label biomedical classification, two regimes where prior LLM annotator synthesis collapses. Coverage is enabled by task-aware Generators and a text-aware aggregator.
\item \textbf{Best cost-accuracy tradeoff and cross-backbone robustness}: the authored pool is $4500$ to $31000\times$ faster than LLM annotation on 100K examples, and EvoPool outperforms the strongest baseline or stays competitive across 4 LLM backbones and 3 downstream models.
\end{itemize}

\section{Methods}
\label{sec:method}

\begin{figure*}[!t]
  \centering
  \includegraphics[width=\linewidth]{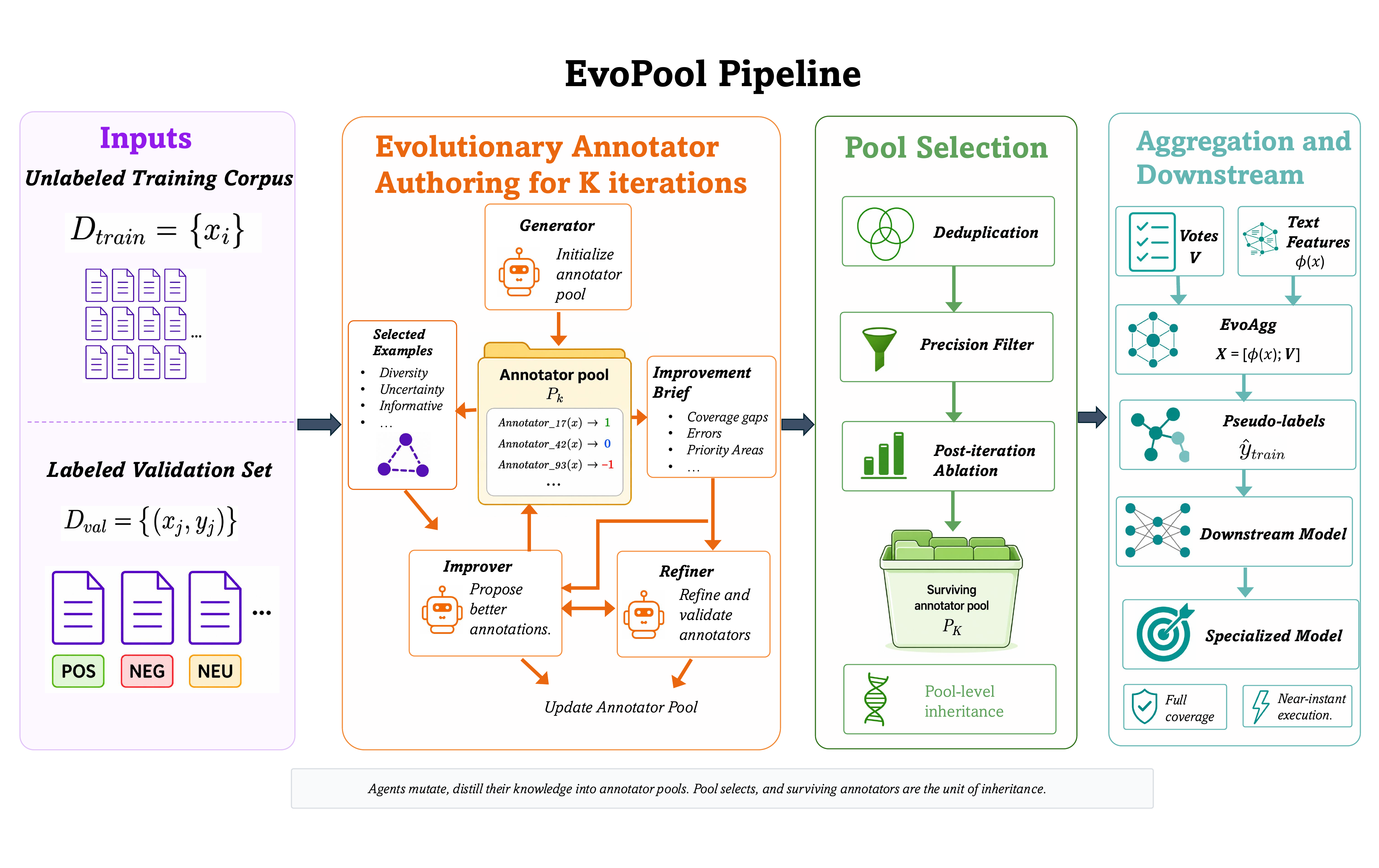}
  \vspace{-30pt}
  \caption{\textbf{Overview of the EvoPool pipeline.} Over $K$ generations, Generator, Improver, and Refiner propose annotators targeting weak classes. A selection gate keeps annotators that pass viability, diversity, and marginal-contribution checks against validation fitness. EvoAgg then aggregates the surviving pool $P_K$ from text features and votes into soft training labels.}
  \vspace{-15pt}
  \label{fig:pipeline}
\end{figure*}

Figure~\ref{fig:pipeline} summarizes our pipeline. EvoPool casts programmatic annotation as evolutionary search over a population of executable annotators: a multi-agent loop introduces variation, a deterministic selection gate enforces fitness, and the surviving pool's votes are aggregated by EvoAgg into soft training labels for a downstream model.

\subsection{Problem Setup}
\label{sec:problem-setup}

Given a corpus of unlabeled training texts $\mathcal{D}_{\text{train}} = \{x_i\}_{i=1}^N$, a small labeled validation set $\mathcal{D}_{\text{val}} = \{(x_j, y_j)\}_{j=1}^M$ with $M \ll N$, and a held-out test set $\mathcal{D}_{\text{test}}$, we aim to author a pool of executable annotators that produces pseudo-labels $\hat{\mathcal{D}}_{\text{train}} = \{(x_i, \hat{y}_i)\}_{i=1}^N$ for the training corpus, on which we then fine-tune a downstream model $f: \mathcal{X} \to \mathcal{Y}$ over $C$ classes. The train-time loop may access $\mathcal{D}_{\text{train}}$ texts and $\mathcal{D}_{\text{val}}$ labels but never $\mathcal{D}_{\text{train}}$ labels.

\subsection{Evolutionary Multi-Agent Loop}

We organize annotator authoring as an evolutionary process with three primitives: \emph{variation} for proposing new annotators, \emph{selection} for retaining the fittest, and \emph{inheritance} for carrying the surviving pool into the next generation. At each generation $k$, three role-specialized LLM agents act on the current pool $P_{k-1}$, and the three roles together partition the exploration-exploitation tradeoff.

\noindent \textbf{Generator (broad variation).} The Generator seeds the initial pool at $k=0$ with diverse candidate annotators that span the label space. Its role is breadth-first exploration of the programmatic-supervision space, providing initial coverage across all classes before any selection is applied.

\noindent \textbf{Improver (targeted variation).} The Improver reads a per-iteration analysis brief that packages the current pool's per-class diagnostics with the query selector's picks. From this brief it identifies the classes where the pool is weakest by validation $F_1$ and concentrates variation there, proposing new annotators for under-performing classes with confusable-class guardrails.

\noindent \textbf{Refiner (exploitation).} The Refiner takes high-precision but low-recall annotators and proposes broader variants that retain the discriminative signal while widening coverage. It exploits already-validated rules to expand coverage around high-fitness candidates rather than introducing fresh diversity.

Each agent is re-invoked from scratch every generation with no persistent state. The pool $P_k$ is the only inheritance channel, making the system strictly Darwinian rather than Lamarckian \citep{lamarck1809philosophie}. Empirically this stateless choice is both cheaper and slightly better than Lamarckian variants with cross-iteration verbal lessons (App.~\ref{app:memory-ablation}), since pool-level selection already supplies the per-class feedback. Algorithm~\ref{alg:evopool} in App.~\ref{app:algorithm} gives the full training-time loop, including the explicit selection gate, the Refiner activation rule, and the per-class targeting logic. App.~\ref{app:design} expands the rationale behind each agent role and selection primitive.

\subsection{Pool-Level Natural Selection}

Selection acts on the augmented pool $P_{k-1} \cup \Delta P_{\text{imp}} \cup \Delta P_{\text{ref}}$ at the end of each generation through three deterministic filters, each motivated by a Darwinian primitive. Formal definitions and thresholds are in App.~\ref{app:selection-impl}.

\noindent \textbf{Niche redundancy.} Annotators whose firing pattern overlaps an existing pool member above a similarity threshold are removed. This keeps the pool diverse without explicit diversity penalties.

\noindent \textbf{Minimum competence.} Annotators that fail a validation-precision floor or fire too rarely are removed. This is the basic viability threshold for entering the pool.

\noindent \textbf{Marginal contribution.} Newly proposed annotators whose removal strictly improves the pool's collective validation score (by at least a small tolerance $\tau_{\text{abl}}$) are dropped through greedy post-iteration ablation. This ensures every newly admitted member contributes positively to population fitness rather than free-riding on other annotators.

The strict-improvement form of this ablation rule gives the post-iteration step a clean monotonicity guarantee: each accepted drop strictly improves the pool's validation macro-F1, so the post-ablation pool always dominates the pre-ablation pool (Prop.~\ref{prop:monotone-quality}, App.~\ref{app:selection-impl}). Empirically the end-of-generation pool is monotone non-decreasing across all our experiments, justifying the Darwinian framing of pool inheritance.

\subsection{Query Selection}
\label{sec:agent-search}

The Improver and Refiner need a signal indicating where the current pool fails, so we design query selection as a principled feedback channel. At each generation we draw a small batch of training examples for the agents to reason over. We use BatchBALD \citep{kirsch2019batchbald} over the annotator ensemble, an information-theoretic acquisition that picks examples maximizing the expected mutual information between unknown labels and annotator predictions, surfacing inputs where the population lacks discriminative power. Unlike uncertainty top-K, BatchBALD measures each example's information given the others in the batch, penalizing internal redundancy. Full derivation is in App.~\ref{app:batchbald-full}.

\subsection{EvoAgg: Text-Aware Aggregation}

The pool $P_K$ casts a vote vector per example, but turning these votes into a label is itself a learning problem. Majority vote assumes annotators are independent and equally reliable and uses no signal beyond the votes themselves. Both assumptions break here: on hard or rare classes the pool may fire only a few times or abstain entirely, and when most annotators fire on a single class, MV collapses to that class regardless of input.

To address these issues, we propose \textbf{EvoAgg}, a learned aggregator that conditions on both votes and the input text. Text is the only signal that does not depend on pool composition: it is present for every example and grounds annotator reliability in actual content. We concatenate a low-dimensional text embedding $\phi(x)$ with the annotator-vote one-hot encoding to form $\mathbf{X} = [\phi(x);\,\mathbf{V}]$, and fit a logistic regression on the validation labels.

The embedding $\phi$ plays three roles. First, it is a fallback signal where the pool is weak or silent, so EvoAgg has coverage $1.0$ by construction. Second, by concatenating rather than predicting separately, the model arbitrates per-example between text and votes. Third, it regularizes the high-dimensional vote space, preventing overfitting on the small validation set.

\begin{table*}[!t]
\caption{\textbf{Annotator-level validation and test macro-F1 across 10 datasets organized into 4 task categories.} All methods use gpt-4o-mini as the LLM backbone. \textbf{Bold} marks the best performance. Full per-dataset Coverage/Acc/Acc-c/F1 breakdown is in App.~\ref{app:full-annotator}.}
\vspace{-10pt}
\label{tab:annotator-main}
\centering
\small
\renewcommand{\arraystretch}{1.06}
\resizebox{0.93\linewidth}{!}{
\begin{tabular}{@{}l|cc|cc|cc|>{\columncolor{cyan!5}}c>{\columncolor{cyan!5}}c@{}}
\toprule
\multirow{2}{*}{Dataset}
 & \multicolumn{2}{c|}{Alchemist}
 & \multicolumn{2}{c|}{DataSculpt}
 & \multicolumn{2}{c|}{LLM annotation (gpt-4o-mini)}
 & \multicolumn{2}{c}{\textbf{EvoPool (ours)}} \\
\cmidrule(lr){2-3}\cmidrule(lr){4-5}\cmidrule(lr){6-7}\cmidrule(lr){8-9}
 & Val & Test & Val & Test & Val & Test & Val & Test \\
\midrule
\multicolumn{9}{c}{\textbf{(1) General classification}} \\
\midrule
AGNews            & 0.4724 & 0.4668 & 0.5151 & 0.4994 & 0.8560 & 0.8628 & \textbf{0.8599} & \textbf{0.8642} \\
Banking77         & 0.0356 & 0.0382 & 0.0871 & 0.0824 & 0.6086 & 0.6261 & \textbf{0.8720} & \textbf{0.7023} \\
\midrule
\multicolumn{9}{c}{\textbf{(2) High-stakes specialized}} \\
\midrule
ChemProt          & 0.2447 & 0.2204 & 0.2035 & 0.1965 & 0.3212 & 0.2817 & \textbf{0.6059} & \textbf{0.5826} \\
DDI               & 0.1239 & 0.1189 & 0.1249 & 0.1134 & 0.2059 & 0.2345 & \textbf{0.5827} & \textbf{0.4539} \\
Claude9          & 0.1607 & 0.1633 & 0.2338 & 0.1458 & \textbf{0.4409} & 0.3320 & 0.3461 & \textbf{0.3645} \\
\midrule
\multicolumn{9}{c}{\textbf{(3) Complex reasoning}} \\
\midrule
FEVER             & 0.1665 & 0.1667 & 0.0464 & 0.0339 & 0.7932 & 0.7599 & \textbf{0.8467} & \textbf{0.8117} \\
VitaminC          & 0.1692 & 0.1681 & 0.0603 & 0.0473 & 0.5993 & 0.6482 & \textbf{0.6999} & \textbf{0.6667} \\
SciFact          & 0.3438 & 0.2855 & 0.2609 & 0.2000 & 0.7040 & \textbf{0.7617} & \textbf{0.7115} & 0.5890 \\
\midrule
\multicolumn{9}{c}{\textbf{(4) Multi-label biomedical}} \\
\midrule
Ohsumed        & 0.1872 & 0.1901 & 0.0722 & 0.0738 & 0.4264 & 0.4433 & \textbf{0.8086} & \textbf{0.5420} \\
PubMed  & 0.1498 & 0.1461 & 0.0064 & 0.0069 & 0.4410 & 0.4375 & \textbf{0.7587} & \textbf{0.7026} \\
\bottomrule
\end{tabular}
}
\vspace{-15pt}
\end{table*}

\section{Experiment Setup}
\label{sec:setup}

\noindent \textbf{Datasets and splits.}
We evaluate EvoPool on 10 datasets organized into four task categories. \textbf{(1) General classification} covers LLM-friendly topic and intent tasks (AGNews \citep{zhang2015agnews}, Banking77 \citep{casanueva2020banking77}). \textbf{(2) High-stakes specialized} targets domains where expert labels are costly and LLM annotation is weak (biomedical relation extraction on ChemProt \citep{krallinger2017chemprot} and DDI \citep{herrero2013ddi}, legal-clause classification on Claude9 \citep{zhang2024stronger}). \textbf{(3) Complex reasoning} covers verification tasks that require explicit claim/evidence comparison (FEVER \citep{thorne2018fever}, VitaminC \citep{schuster2021vitaminc}, SciFact \citep{wadden2020scifact}). For these tasks we precompute a small set of claim-evidence comparison features and attach them to each example's metadata for all methods, as described in App.~\ref{app:enrichment}. \textbf{(4) Multi-label biomedical} probes dense per-class supervision (Ohsumed \citep{hersh1994ohsumed}, PubMed \citep{kaggle_pubmed}). For all datasets, train has no ground-truth labels, validation is used by EvoPool only for selection and aggregator fitting, and test is held out. Full dataset sizes and statistics are in App.~\ref{app:datasets}.

\noindent \textbf{Baselines.}
We compare against five families. \textit{LLM-Eval} performs direct LLM evaluation at test time without annotating training data. \textit{LLM annotation} uses per-example LLM inference to label the training set. \textit{One-shot LLM annotator synthesis} includes Alchemist \citep{alchemist} and DataSculpt \citep{datasculpt}. \textit{Self-Training} \citep{xie2020selftraining} trains a teacher on the validation labels, uses it to pseudo-label the training corpus, and re-trains the downstream model on those pseudo-labels. \textit{Supervised (Golden)} uses available labeled data from the validation set. All methods share the same data, validation labels, and metadata, with pool-based baselines such as Alchemist and DataSculpt using majority vote as their native aggregator. 

\noindent \textbf{Training.}
The EvoPool loop runs for $K$ generations with gpt-4o-mini \citep{openai2023gpt4} as the primary agent backbone. Cross-backbone variants (gpt-5-mini \citep{singh2026openaigpt5card}, gemini-3.1-pro-preview \citep{google2026gemini31pro}, deepseek-v4-flash \citep{deepseekv4}) are evaluated on ChemProt and FEVER. We evaluate 3 downstream models: RoBERTa-large \citep{liu2019roberta}, Qwen3-1.7B \citep{qwen2024qwen3}, and Llama-3.1-8B \citep{touvron2024llama3}. Full training details and hyperparameters in App.~\ref{app:hparams}.

\noindent \textbf{Evaluation.}
We report test macro-F1 on the downstream model and annotator-level macro-F1 on both validation and test. Coverage, accuracy, accuracy on the covered subset, and weighted-F1 are reported alongside in App.~\ref{app:full-annotator}.

\section{Results}
\label{sec:results}

\begin{table*}[tb!]
\caption{\textbf{Downstream model test macro-F1 across 10 datasets $\times$ 3 downstream models, organized into 4 task categories.} Each method's pool produces pseudo-labels for the full train set. Downstream model fine-tuned end-to-end. Models: RoB$=$RoBERTa-large, Qwen$=$Qwen3-1.7B, Llama$=$Llama-3.1-8B. All pools constructed using gpt-4o-mini backbone. Full metrics in App.~\ref{app:full-downstream}.}
\vspace{-8pt}
\label{tab:downstream-main}
\centering

\begin{minipage}[t]{0.48\linewidth}
\centering\footnotesize
\textbf{(1) General classification}\\[2pt]
\setlength{\tabcolsep}{3pt}
\renewcommand{\arraystretch}{0.95}
\begin{tabular}{@{}ll|ccc@{}}
\toprule
Dataset & Method & RoB & Qwen & Llama \\
\midrule
 \multirow{7}{*}{AGNews}
 & Golden    & \textbf{0.9016} & 0.8422 & 0.8658 \\
 & LLM-Eval & 0.8628 & 0.8628 & 0.8628 \\
\cmidrule(lr){2-5}
 & Alchemist           & 0.5732 & 0.5541 & 0.5298 \\
 & DataSculpt            & 0.8222 & 0.7922 & 0.7248 \\
 & LLM-Annotation  & 0.8746 & 0.8701 & 0.8735 \\
 & Self-Train    & 0.8925 & 0.8588 & 0.8803 \\
\cmidrule(lr){2-5}
 & \cellcolor{cyan!5}\textbf{EvoPool} & \cellcolor{cyan!5}0.8860 & \cellcolor{cyan!5}\textbf{0.8951} & \cellcolor{cyan!5}\textbf{0.8859} \\
\cmidrule(lr){1-5}
 \multirow{7}{*}{Banking77}
 & Golden    & 0.4330 & 0.0679 & 0.5117 \\
 & LLM-Eval & 0.6261 & 0.6261 & 0.6261 \\
\cmidrule(lr){2-5}
 & Alchemist           & 0.0491 & 0.0443 & 0.0476 \\
 & DataSculpt            & 0.0929 & 0.0540 & 0.0713 \\
 & LLM-Annotation  & 0.6550 & 0.6361 & 0.6410 \\
 & Self-Train    & \textbf{0.7607} & 0.1220 & 0.6371 \\
\cmidrule(lr){2-5}
 & \cellcolor{cyan!5}\textbf{EvoPool} & \cellcolor{cyan!5}0.7332 & \cellcolor{cyan!5}\textbf{0.7150} & \cellcolor{cyan!5}\textbf{0.7241} \\
\bottomrule
\end{tabular}

\vspace{6pt}

\textbf{(3) Complex reasoning}\\[2pt]
\setlength{\tabcolsep}{3pt}
\renewcommand{\arraystretch}{0.95}
\begin{tabular}{@{}ll|ccc@{}}
\toprule
Dataset & Method & RoB & Qwen & Llama \\
\midrule
 \multirow{7}{*}{FEVER}
 & Golden    & 0.7805 & 0.6477 & 0.6042 \\
 & LLM-Eval & 0.7599 & 0.7599 & 0.7599 \\
\cmidrule(lr){2-5}
 & Alchemist           & 0.1667 & 0.1667 & 0.1667 \\
 & DataSculpt            & 0.3285 & 0.3541 & 0.3434 \\
 & LLM-Annotation  & 0.7428 & 0.7626 & 0.7645 \\
 & Self-Train    & 0.7556 & 0.7099 & 0.6879 \\
\cmidrule(lr){2-5}
 & \cellcolor{cyan!5}\textbf{EvoPool} & \cellcolor{cyan!5}\textbf{0.8132} & \cellcolor{cyan!5}\textbf{0.8262} & \cellcolor{cyan!5}\textbf{0.8435} \\
\cmidrule(lr){1-5}
 \multirow{7}{*}{VitaminC}
 & Golden        & 0.4543 & 0.5065 & 0.4481 \\
 & LLM-Eval & 0.6482 & 0.6482 & 0.6482 \\
\cmidrule(lr){2-5}
 & Alchemist           & 0.1667 & 0.1667 & 0.1667 \\
 & DataSculpt            & 0.2646 & 0.2257 & 0.2600 \\
 & LLM-Annotation  & 0.5919 & 0.6562 & 0.6508 \\
 & Self-Train    & 0.5282 & 0.5174 & 0.4784 \\
\cmidrule(lr){2-5}
 & \cellcolor{cyan!5}\textbf{EvoPool} & \cellcolor{cyan!5}\textbf{0.6476} & \cellcolor{cyan!5}\textbf{0.6659} & \cellcolor{cyan!5}\textbf{0.6797} \\
\cmidrule(lr){1-5}
 \multirow{7}{*}{SciFact}
 & Golden    & 0.2864 & 0.2929 & 0.4096 \\
 & LLM-Eval & 0.7617 & 0.7617 & 0.7617 \\
\cmidrule(lr){2-5}
 & Alchemist           & 0.2051 & 0.2743 & 0.2903 \\
 & DataSculpt            & 0.3159 & 0.3111 & 0.2759 \\
 & LLM-Annotation  & \textbf{0.5919} & \textbf{0.4355} & \textbf{0.4788} \\
 & Self-Train    & 0.3553 & 0.4205 & 0.4060 \\
\cmidrule(lr){2-5}
 & \cellcolor{cyan!5}\textbf{EvoPool} & \cellcolor{cyan!5}0.2890 & \cellcolor{cyan!5}0.3797 & \cellcolor{cyan!5}0.4442 \\
\bottomrule
\end{tabular}
\end{minipage}\hfill
\begin{minipage}[t]{0.48\linewidth}
\centering\footnotesize
\textbf{(2) High-stakes specialized}\\[2pt]
\setlength{\tabcolsep}{3pt}
\renewcommand{\arraystretch}{0.95}
\begin{tabular}{@{}ll|ccc@{}}
\toprule
Dataset & Method & RoB & Qwen & Llama \\
\midrule
 \multirow{7}{*}{ChemProt}
 & Golden    & 0.5194 & 0.2775 & 0.3178 \\
 & LLM-Eval & 0.2817 & 0.2817 & 0.2817 \\
\cmidrule(lr){2-5}
 & Alchemist           & 0.2193 & 0.2194 & 0.2129 \\
 & DataSculpt            & 0.2267 & 0.2084 & 0.1994 \\
 & LLM-Annotation  & 0.2840 & 0.2945 & 0.2809 \\
 & Self-Train    & \textbf{0.5535} & 0.3049 & 0.3253 \\
\cmidrule(lr){2-5}
 & \cellcolor{cyan!5}\textbf{EvoPool} & \cellcolor{cyan!5}0.5396 & \cellcolor{cyan!5}\textbf{0.5383} & \cellcolor{cyan!5}\textbf{0.5909} \\
\cmidrule(lr){1-5}
 \multirow{7}{*}{DDI}
 & Golden    & 0.5774 & 0.4491 & 0.3863 \\
 & LLM-Eval & 0.2345 & 0.2345 & 0.2345 \\
\cmidrule(lr){2-5}
 & Alchemist           & 0.1806 & 0.2155 & 0.2121 \\
 & DataSculpt            & 0.1746 & 0.1756 & 0.1778 \\
 & LLM-Annotation  & 0.2234 & 0.2206 & 0.2305 \\
 & Self-Train    & \textbf{0.5782} & 0.4592 & 0.4581 \\
\cmidrule(lr){2-5}
 & \cellcolor{cyan!5}\textbf{EvoPool} & \cellcolor{cyan!5}0.5340 & \cellcolor{cyan!5}\textbf{0.5293} & \cellcolor{cyan!5}\textbf{0.5379} \\
\cmidrule(lr){1-5}
 \multirow{7}{*}{Claude9}
 & Golden    & 0.1057 & 0.1053 & 0.2111 \\
 & LLM-Eval & 0.3320 & 0.3320 & 0.3320 \\
\cmidrule(lr){2-5}
 & Alchemist           & 0.1867 & 0.2207 & 0.1941 \\
 & DataSculpt            & 0.1057 & 0.1909 & 0.1500 \\
 & LLM-Annotation  & \textbf{0.3469} & 0.3579 & 0.3639 \\
 & Self-Train    & 0.1057 & 0.0765 & 0.1131 \\
\cmidrule(lr){2-5}
 & \cellcolor{cyan!5}\textbf{EvoPool} & \cellcolor{cyan!5}0.3413 & \cellcolor{cyan!5}\textbf{0.3827} & \cellcolor{cyan!5}\textbf{0.3908} \\
\bottomrule
\end{tabular}

\vspace{6pt}

\textbf{(4) Multi-label biomedical}\\[2pt]
\setlength{\tabcolsep}{3pt}
\renewcommand{\arraystretch}{0.95}
\begin{tabular}{@{}ll|ccc@{}}
\toprule
Dataset & Method & RoB & Qwen & Llama \\
\midrule
 \multirow{7}{*}{Ohsumed}
 & Golden    & 0.0671 & 0.0042 & 0.0530 \\
 & LLM-Eval & 0.4433 & 0.4433 & 0.4433 \\
\cmidrule(lr){2-5}
 & Alchemist           & 0.1940 & 0.2012 & 0.2006 \\
 & DataSculpt            & 0.0713 & 0.0685 & 0.0692 \\
 & LLM-Annotation  & 0.4566 & 0.4634 & 0.4726 \\
 & Self-Train    & 0.2526 & 0.0652 & 0.1869 \\
\cmidrule(lr){2-5}
 & \cellcolor{cyan!5}\textbf{EvoPool} & \cellcolor{cyan!5}\textbf{0.5565} & \cellcolor{cyan!5}\textbf{0.5473} & \cellcolor{cyan!5}\textbf{0.5566} \\
\cmidrule(lr){1-5}
 \multirow{7}{*}{PubMed}
 & Golden    & 0.6229 & 0.4477 & 0.5593 \\
 & LLM-Eval & 0.4375 & 0.4375 & 0.4375 \\
\cmidrule(lr){2-5}
 & Alchemist           & 0.1399 & 0.1390 & 0.1397 \\
 & DataSculpt            & 0.0043 & 0.0047 & 0.0053 \\
 & LLM-Annotation  & 0.4396 & 0.4313 & 0.4353 \\
 & Self-Train    & 0.6975 & 0.4566 & 0.5618 \\
\cmidrule(lr){2-5}
 & \cellcolor{cyan!5}\textbf{EvoPool} & \cellcolor{cyan!5}\textbf{0.7126} & \cellcolor{cyan!5}\textbf{0.7095} & \cellcolor{cyan!5}\textbf{0.7110} \\
\bottomrule
\end{tabular}
\end{minipage}

\vspace{-12pt}
\end{table*}

\noindent \textbf{Annotator-level results.}
Table~\ref{tab:annotator-main} reports performance of the annotator pools on $10$ datasets across four task categories. EvoPool wins $8$ of $10$ datasets on both validation and test, with the largest test-set gap on the $10$-class biomedical ChemProt task at $+0.301$ over LLM annotation, followed by the multi-label biomedical PubMed at $+0.265$ and DDI at $+0.219$. Across 7 of 8 LLM-weak specialized and complex datasets, EvoPool beats LLM annotation by an average $+0.141$. SciFact is the one test-set loss, where LLM annotation leads EvoPool by $0.173$, consistent with gpt-4o-mini's strong pretraining on scientific claim language.

\noindent \textbf{Downstream model results.}
Table~\ref{tab:downstream-main} reports test macro-F1 of three downstream models RoBERTa-large, Qwen3-1.7B, and Llama-3.1-8B fine-tuned on each method's pseudo-labels. The downstream ranking mirrors the annotator-level ranking. On the biomedical tasks EvoPool dominates LLM-Annotation by $+0.24$ to $+0.31$ on ChemProt and by approximately $+0.31$ on DDI across all three downstream models, and beats every baseline on DDI under Llama-3.1-8B. EvoPool also leads on FEVER by $+0.06$ to $+0.08$, on VitaminC by $+0.01$ to $+0.06$, and on Banking77 by approximately $+0.08$, while matching LLM-Annotation on Claude9 where both methods maintain non-collapsed multi-class coverage as the other baselines collapse to a dominant class. SciFact remains the boundary case where LLM-Annotation wins downstream as well. On the multi-label biomedical tasks, EvoPool wins every cell, lifting test macro-F1 by $+0.015$ to $+0.25$ over the strongest baseline per model, with Self-Training collapsing on several LoRA-tuned cells. App.~\ref{app:cft} extends the downstream evaluation with continued fine-tuning on validation, and App.~\ref{app:human} compares EvoPool to the hand-crafted human annotator pools.

\begin{table*}[!t]
\centering
\small
\setlength{\tabcolsep}{4pt}
\renewcommand{\arraystretch}{1.04}
\caption{\textbf{Pipeline component ablation.} Leave-one-out using gpt-4o-mini backbone: starting from the final configuration with all components, remove a single component and re-run the full pipeline on ChemProt and FEVER. Four metrics are measured per dataset: MV and EvoAgg annotator-level test macro-F1, and Llama-3.1-8B downstream test macro-F1 trained on MV and EvoAgg pseudo-labels.}
\vspace{-10pt}
\label{tab:ablation-loo}
\begin{tabular}{@{}l|cccc|cccc@{}}
\toprule
& \multicolumn{4}{c|}{\textbf{ChemProt}} & \multicolumn{4}{c}{\textbf{FEVER}} \\
\cmidrule(lr){2-5}\cmidrule(lr){6-9}
Configuration & MV & EvoAgg & Down.\ MV & Down.\ EvoAgg & MV & EvoAgg & Down.\ MV & Down.\ EvoAgg \\
\midrule
\rowcolor{cyan!5}
\textbf{Full pipeline}      & \textbf{0.4787} & \textbf{0.5826} & \textbf{0.4817} & \textbf{0.5909} & \textbf{0.6696} & \textbf{0.8117} & \textbf{0.6735} & \textbf{0.8435} \\
$-$Improver                 & 0.4210 & 0.5207 & 0.4573 & 0.5176 & 0.5407 & 0.7456 & 0.5604 & 0.7790 \\
$-$Refiner                  & 0.4508 & 0.5712 & 0.4592 & 0.5671 & 0.6426 & 0.7881 & 0.6485 & 0.8192 \\
$-$Selection gate           & 0.4389 & 0.5822 & 0.4363 & 0.5881 & 0.5572 & 0.8112 & 0.5519 & 0.8396 \\
$-$Query selection          & 0.4403 & 0.5783 & 0.4527 & 0.5832 & 0.5846 & 0.7984 & 0.5916 & 0.8336 \\
Generator only              & 0.3289 & 0.5375 & 0.3636 & 0.5365 & 0.5084 & 0.7335 & 0.5092 & 0.6064 \\
\bottomrule
\end{tabular}
\vspace{-8pt}
\end{table*}

\begin{table*}[h]
\centering
\small
\setlength{\tabcolsep}{4pt}
\caption{\textbf{Cross-backbone transfer.} We run identical pipeline across backbones, with only the LLM author differing. We report downstream performance using Llama-3.1-8B on ChemProt and FEVER.}
\vspace{-8pt}
\label{tab:cross-backbone}
\resizebox{\linewidth}{!}{%
\begin{tabular}{@{}l|ccccc|ccccc@{}}
\toprule
& \multicolumn{5}{c|}{\textbf{ChemProt}} & \multicolumn{5}{c}{\textbf{FEVER}} \\
\cmidrule(lr){2-6}\cmidrule(lr){7-11}
Backbone & LLM-Eval & LLM-Annotation & Alchemist & DataSculpt & EvoPool & LLM-Eval & LLM-Annotation & Alchemist & DataSculpt & EvoPool \\
\midrule
gpt-4o-mini             & 0.282 & 0.281 & 0.213 & 0.199 & 0.591 & 0.760 & 0.765 & 0.167 & 0.343 & \textbf{0.844} \\
gpt-5-mini              & 0.282 & 0.288 & 0.308 & 0.237 & 0.586 & 0.760 & 0.760 & 0.366 & 0.323 & 0.840 \\
gemini-3.1-pro-preview  & 0.430 & 0.439 & 0.425 & 0.149 & \textbf{0.602} & 0.871 & 0.847 & 0.393 & 0.337 & 0.841 \\
deepseek-v4-flash       & 0.317 & 0.313 & 0.291 & 0.184 & 0.569 & 0.854 & 0.843 & 0.268 & 0.371 & 0.842 \\
\bottomrule
\end{tabular}}
\vspace{-12pt}
\end{table*}

\noindent \textbf{Pipeline component ablation.}
Table~\ref{tab:ablation-loo} reports a leave-one-out ablation on ChemProt and FEVER, removing the Improver, the Refiner, the selection gate, the query selector, or iteration entirely from the default configuration and re-running the full pipeline. Every component contributes positively on both datasets, with the Improver carrying the largest single load: it is the only agent that targets specific weak classes from per-iteration validation diagnostics, whereas the Generator front-loads diversity at $k{=}0$ and the Refiner exploits already-validated rules. Stripping iteration so only the Generator pool remains drops MV macro-F1 by $0.150$ on ChemProt and $0.161$ on FEVER, isolating the evolutionary loop as the main source of the lift. The drops are more pronounced under majority vote than under EvoAgg, since each component directly shapes the pool composition that the vote aggregates over. Switching aggregation from majority vote to EvoAgg adds $+0.099$ to $+0.210$ on ChemProt and $+0.142$ to $+0.254$ on FEVER on top of the per-component lifts, indicating that pool evolution and text-aware aggregation complement rather than substitute for one another. App.~\ref{app:agg-on-baselines} extends the comparison to the Alchemist and DataSculpt pools and a text-only baseline. At fixed MV aggregation EvoPool ChemProt reaches $0.479$ macro-F1, $+0.255$ above Alchemist and $+0.298$ above DataSculpt, and still wins by $+0.067$ at fixed EvoAgg aggregator and by $+0.119$ over text-only. Further ablations cover agent memory in App.~\ref{app:memory-ablation}, query-selector choice in App.~\ref{app:selector-ablation}, pool-size Pareto efficiency of the selection gate in App.~\ref{app:pool-efficiency}, and hyperparameter sensitivity in App.~\ref{app:hparam-ablation}. The full pipeline wins both datasets.

\begin{figure}[!t]
\centering
\includegraphics[width=\linewidth]{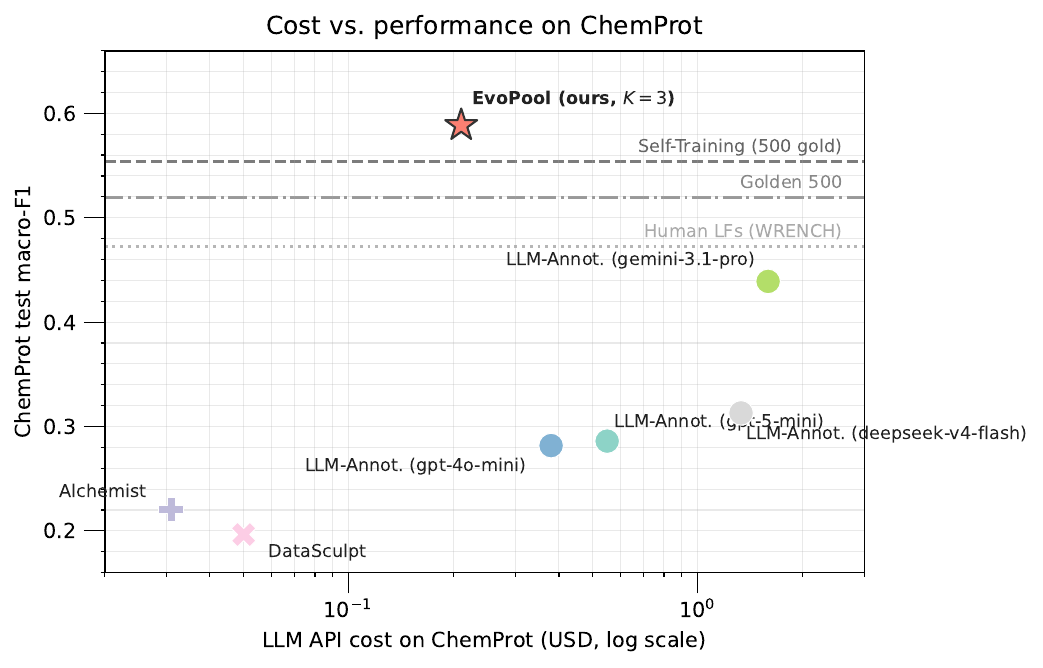}
\vspace{-15pt}
\caption{\textbf{Cost vs.\ performance.} Each point is one annotation method on ChemProt. Cost is total USD spend measured from gpt-4o-mini API calls. Horizontal lines are clean-label references dominated by human annotation cost. EvoPool with $K{=}3$ sits at the top-left, beating Self-Training, Golden, and the expert hand-authored pool at roughly \$0.21.}
\vspace{-20pt}
\label{fig:cost-perf-chemprot}
\end{figure}

\begin{figure*}[!t]
\centering
\includegraphics[width=0.95\linewidth]{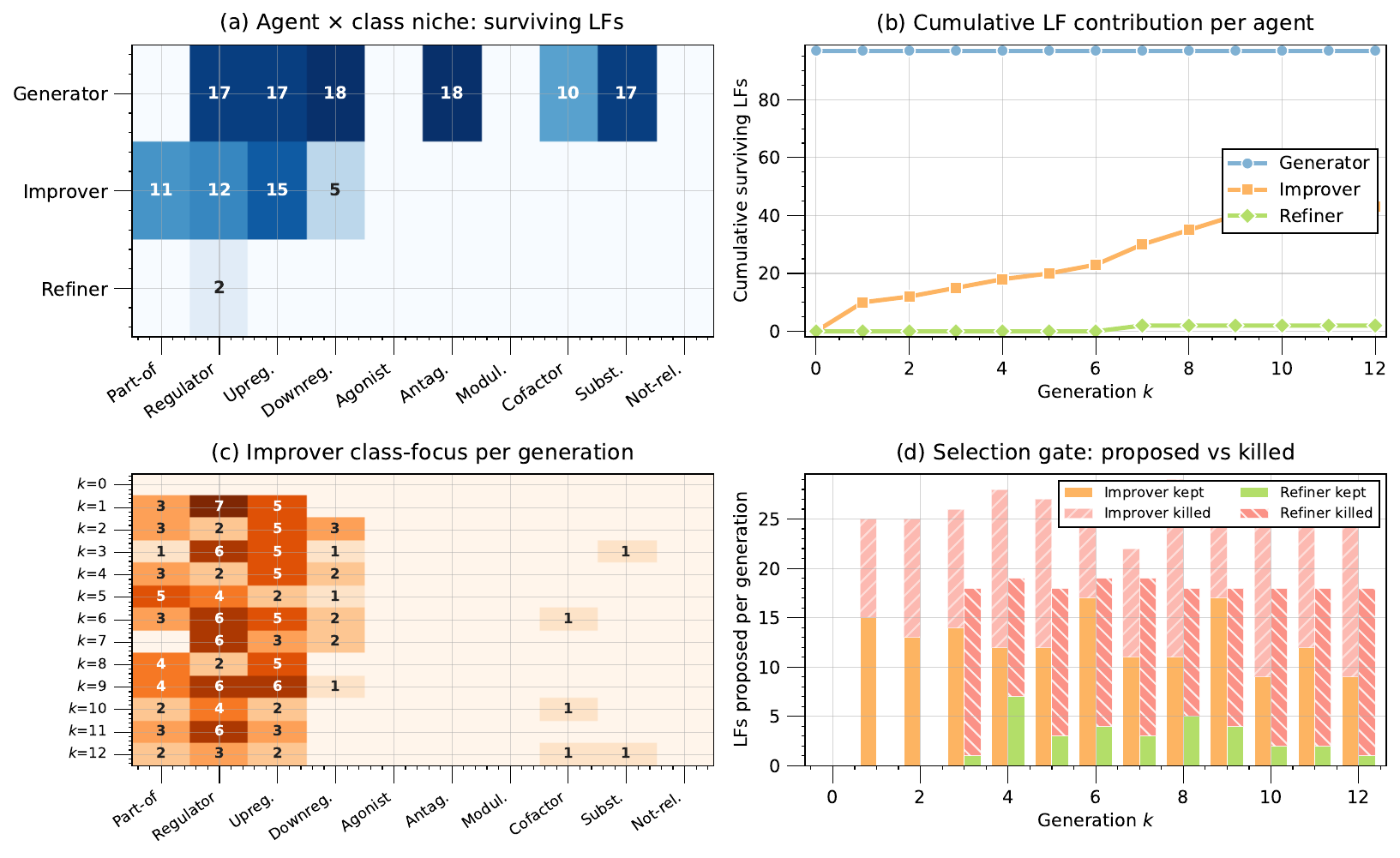}
\vspace{-10pt}
\caption{\textbf{Pool co-evolution signals.} Analysis done on ChemProt with gpt-4o-mini backbone. (a) Agent $\times$ class niche in the final pool. The Generator covers the broad classes, and the Improver fills Part-of, Agonist, Modulator, and Not-relation. (b) Cumulative surviving annotators by agent. The Generator front-loads at $k{=}0$, the Improver ramps from $k{=}1$, and the Refiner activates at $k{=}7$ once $\tau_{\text{ref-min}}{=}3$ allows it. (c) Improver class-focus per generation. (d) Selection gate retention: roughly $40$--$65$\% of Improver and $70$--$95$\% of Refiner candidates are killed each generation by the niche-redundancy and precision filters.}
\vspace{-18pt}
\label{fig:coevolution-chemprot}
\end{figure*}

\noindent \textbf{Cross-backbone transfer.}
Table~\ref{tab:cross-backbone} evaluates all five methods with four LLM backbones on ChemProt and FEVER, reporting Llama-3.1-8B downstream test macro-F1 trained on each method's pseudo-labels. EvoPool downstream performance is highly consistent across backbones, landing in $0.569$--$0.602$ on ChemProt and $0.840$--$0.844$ on FEVER, while one-shot synthesis baselines fluctuate by up to $0.21$ across backbones on ChemProt. On ChemProt EvoPool leads every backbone by $+0.17$ to $+0.31$ over direct LLM prediction. On FEVER the gap closes, but EvoPool's pool stays within $0.4$pp across the four backbones while LLM-Eval itself swings by $11.1$pp on the same axis. The cheapest backbone authors a pool indistinguishable from the most expensive at near-zero per-example deployment cost.

\noindent \textbf{Cost-performance trade-off.}
Figure~\ref{fig:cost-perf-chemprot} plots annotator test macro-F1 against total LLM API cost on ChemProt, measured directly from cached usage logs. EvoPool with $K{=}3$ generations, a cost-efficient variant of the default $K{=}12$, reaches $0.589$ macro-F1 at \$$0.21$, dominating every baseline on both axes. It exceeds direct LLM annotation with gemini-3.1-pro-preview by $+0.15$ at roughly $1/8$ of its cost, and the one-shot Alchemist and DataSculpt pools by more than $+0.35$ at only $5$--$7\times$ their cost. The Self-Training, Golden, and expert-LF clean-label references all sit below EvoPool, indicating that a few iterations of agent-authored programmatic supervision substitute for clean-label supervision in this regime. App.~\ref{app:latency} reports per-example deployment latency over $100$K examples for EvoPool versus four LLM annotation APIs.

\noindent \textbf{Pool co-evolution analysis.}
Figure~\ref{fig:coevolution-chemprot} shows how the pool evolves on ChemProt with gpt-4o-mini. Starting from identical prompt templates, the three agents differentiate into emergent class niches in panel (a): the Generator covers broad relation classes and the Improver fills long-tail biomedical classes the Generator missed. An activation schedule also emerges in panel (b), with the Generator front-loading at $k{=}0$, the Improver ramping from $k{=}1$, and the Refiner activating only at $k{=}7$ once the pool has stabilized. Panels (c) and (d) show the gate killing $40$--$65$\% of Improver and most Refiner candidates each generation, and the Improver giving up on classes the gate has ruled out. None of these patterns are pre-specified. They emerge from validation-driven selection across generations, supporting that the pool, not the agents, is the unit of inheritance. The FEVER plot is in App.~\ref{app:coevolution-fever}, and App.~\ref{app:qualitative} collects agent prompts, the analysis-to-Improver dialogue, and the highest-F1 surviving annotators per method on ChemProt and FEVER.

\section{Conclusion}

We introduced EvoPool, an evolutionary multi-agent framework that authors a pool of executable annotators under pool-level selection, paired with EvoAgg, a text-aware aggregator. Across 10 datasets, EvoPool beats LLM annotation by an average $+0.141$ test macro-F1 on 7 of 8 LLM-weak tasks at near-zero per-example deployment cost, and the advantage holds across 4 LLM backbones and 3 downstream models. EvoPool moves programmatic annotation from one-shot synthesis to an adaptive process where the pool itself evolves under feedback from a small validation set, opening specialized supervision to settings where expert annotation is unaffordable and direct LLM annotation does not scale.

\section*{Limitations}
\label{sec:limitations}
Our evaluation spans 10 English datasets and four LLM backbones. Broader languages and multi-seed variance are deferred. EvoPool targets structured prediction tasks expressible as discrete-label Python rules such as text classification, relation extraction, and claim verification. Open-ended generation, extractive question answering, and multi-step reasoning trajectories fall outside this scope and represent natural extensions to non-programmatic supervision. The pipeline also relies on an LLM author for annotator generation. We validate robustness across four backbones, but the authored pool inherits each author's coverage of the target domain. EvoAgg requires roughly ten validation examples per class for stable cross-validated fitting. On Claude9, several rare classes have fewer than ten validation examples due to severe class imbalance, so we use MV as the final aggregator there. Training-time LLM calls remain, so the cost story is amortization over a large unlabeled corpus rather than zero LLM usage.

\section*{Ethical Considerations}
EvoPool targets specialized, high-stakes domains where expert annotation is expensive, potentially broadening programmatic-supervision access for under-resourced labs and clinics. However, the authored annotators inherit biases of the LLM author and may produce disparate failure modes on rare clinical entities or minority legal-clause types that we do not exhaustively audit. The Python-rule format is auditable by design, and we recommend that deployments in consequential settings such as clinical decision support or legal triage treat EvoPool labels as a starting point requiring human review rather than a substitute for expert judgment.

\bibliography{custom}

\appendix

\section{Related Work}
\label{app:related-work}

\noindent \textbf{Weak supervision and programmatic annotation.}
Programmatic weak supervision authors labeling functions that vote on each example and aggregates the noisy votes into pseudo-labels. Data programming \citep{ratner2016dataprogramming} and Snorkel \citep{ratner2017snorkel} formalized this around hand-authored labeling functions, with Snorkel DryBell \citep{bach2019drybell} extending the paradigm to organizational scale. Subsequent work scaled aggregation \citep{ratner2019metal,fu2020fastandthree}, standardized benchmarking \citep{zhang2021wrench,zhang2024stronger}, reduced human effort through natural-language explanations \citep{hancock2018babble}, and automatically synthesized labeling functions from a small labeled set \citep{varma2018snuba}. Alfred \citep{guan2023alfred} replaced hand-coded labeling functions with prompted LLM sources, and interactive systems combine labeling-function authoring with feedback on individual instances \citep{choi2024interactive} or active learning \citep{wei2026dall}. EvoPool removes the human author entirely. LLM agents author the annotators themselves under selection pressure on the pool, and EvoAgg conditions on the input text rather than the votes alone.

\noindent \textbf{LLM-as-Judge and LLM-synthesized annotators.}
A complementary line uses the LLM itself as the supervision source \citep{tan2024llmdatasurvey}. LLM-as-Judge and LLM-as-Annotator approaches call the model once per training example \citep{zheng2023judging,dubois2023alpacafarm,wang2021gpt3dataaug}, with recent refinements distilling uncertainty over candidate annotations \citep{prompt2025candidates} or automating guideline improvement \citep{bibal2025automating}. Cost scales linearly with the corpus, which is impractical at deployment scale. A cheaper variant synthesizes training data or executable annotators offline. Self-Instruct \citep{wang2023selfinstruct}, Unnatural Instructions \citep{honovich2023unnatural}, and SuperGen \citep{meng2022supergen} generate synthetic examples, while Alchemist \citep{alchemist}, DataSculpt \citep{datasculpt}, and EXPONA \citep{guo2024expona} produce labeling functions in a single shot, drawing on broader LLM program-synthesis capability \citep{austin2021programsynthesis}. WeShap \citep{kang2024weshap} complements these by attributing per-source contributions to the downstream pipeline via Shapley values. Without feedback on the actual class distribution, the resulting pool inherits the LLM's prior biases and abandons the long tail. EvoPool closes this gap by iterating the pool under explicit validation pressure.

\noindent \textbf{Multi-agent systems.}
Multi-agent LLM frameworks coordinate specialized agents for complex tasks. MetaGPT \citep{hong2024metagpt}, AutoGen \citep{wu2023autogen}, ChatDev \citep{qian2024chatdev}, CAMEL \citep{li2023camel}, and DyLAN \citep{liu2023dylan} orchestrate role-specialized LLMs for software engineering, dialogue, and reasoning. ReAct \citep{yao2023react} interleaves reasoning traces with grounded actions inside a single trajectory, and a second strand equips agents with persistent state across episodes, where Reflexion \citep{shinn2023reflexion}, Voyager \citep{wang2023voyager}, Generative Agents \citep{park2023generative}, and MemGPT \citep{packer2023memgpt} store verbal lessons across iterations under the assumption that experience compounds. EvoPool extends this multi-agent paradigm to efficient code-based annotation, where agents collaboratively author and iteratively refine executable annotator code that then labels the entire corpus at near-zero per-example cost. Within this loop, deterministic selection on the surviving annotators substitutes for cross-iteration agent memory.

\noindent \textbf{Agent evolution and data Darwinism.}
As data quality becomes the binding constraint on modern LLM training \citep{wang2026opus}, a recent line treats data engineering itself as an evolutionary, automated process. Data Darwinism \citep{mi2025datadarwinism1,mi2026datadarwinism} evolves pretraining-data curation strategies through a closed generate-evaluate-refine loop, EvoAgent \citep{yuan2024evoagent} applies evolutionary operators to a single-agent template to spawn diverse multi-agent variants, Self-Refine \citep{madaan2023selfrefine} iteratively revises a model's outputs via self-feedback, Self-Discover \citep{zhou2024selfdiscover} composes task-specific reasoning structures from atomic reasoning modules, and EvoChamber \citep{wu2024evochamber} co-evolves a multi-agent system at the individual, team, and population scales at test time. The Darwin G\"odel Machine \citep{zhang2026darwin} pushes this further by letting agents self-modify their own code under archival selection, and Hyperagents \citep{zhang2026hyperagents} extends this with metacognitive self-modification where the modification procedure itself is editable. In each case the evolving unit is the agent, the data, or the reasoning chain. EvoPool sits at a complementary layer by evolving the executable annotators themselves, so the surviving artifacts rather than the surviving agents form the heritable unit and produce a reusable cheap-to-execute pool.

\section{Detailed Algorithm}
\label{app:algorithm}

\begin{algorithm*}[h]
\small
\caption{EvoPool training-time loop}
\label{alg:evopool}
\begin{algorithmic}[1]
\STATE \textbf{Input:} train $\mathcal{D}_{\text{train}}$, val $\mathcal{D}_{\text{val}}$, test $\mathcal{D}_{\text{test}}$; LLM $M$ at temperature $T$; agents $G, I, R$.
\STATE \textbf{Hyperparameters:} generations $K$, Generator batches $n_G$, Improver batches $n_I$, positives per class $n_{\text{pos}}$, query size $k_{\max}$, precision $\tau_{\text{prec}}$, min fires $\tau_{\text{fires}}$, Jaccard cap $\tau_{\text{jac}}$, train-val gap $\tau_{\text{tv}}$, ablation tolerance $\tau_{\text{abl}}$, Refiner activation $\tau_{\text{ref-min}}$, dropout streak $\tau_{\text{drop}}$, subsumption $\tau_{\text{sub}}$.
\STATE \textbf{Initialize:} $P_0 \leftarrow G(\mathcal{D}_{\text{train}}, M, T, n_G)$, filtered by val precision $\geq \tau_{\text{prec}}$ and fires $\geq \tau_{\text{fires}}$. Class fail streak $c_j \leftarrow 0$.
\FOR{$k = 1$ to $K$}
\STATE \textbf{1. Analyze:} evaluate $P_{k-1}$ on $\mathcal{D}_{\text{val}}$, build improvement brief $\mathcal{B}_k$ over low-recall regions of $\mathcal{D}_{\text{train}}$.
\STATE \textbf{2. Query selection:} BatchBALD picks $S_k \subset \mathcal{D}_{\text{train}}$ with $|S_k| = k_{\max}$; cluster by facility location, attach to $\mathcal{B}_k$.
\STATE \textbf{3. Class targeting:} pick bottom-$n_I$ classes with val F1 $< 0.40$ and $c_j < \tau_{\text{drop}}$ as Improver targets.
\STATE \textbf{4. Improver:} $\Delta P_{\text{imp}}^{\text{raw}} \leftarrow I(M, T, P_{k-1}, \mathcal{B}_k, n_{\text{pos}})$.
\STATE \textbf{5. Improver gate:} keep candidates with val precision $\geq \tau_{\text{prec}}$, fires $\geq \tau_{\text{fires}}$, train-val gap $\leq \tau_{\text{tv}}$, Jaccard with $P_{k-1}$ $< \tau_{\text{jac}}$. Output $\Delta P_{\text{imp}}$.
\STATE \textbf{6. Refiner:} \textbf{if} $k \geq \tau_{\text{ref-min}}$ pick members of $P_{k-1} \cup \Delta P_{\text{imp}}$ with val precision $\geq 0.65$ and val cov $\in [0.005, 0.05]$, call $\Delta P_{\text{ref}}^{\text{raw}} \leftarrow R(\cdot)$ for $2$ to $3$ broader variants each. \textbf{else} $\Delta P_{\text{ref}}^{\text{raw}} \leftarrow \emptyset$.
\STATE \textbf{7. Refiner gate:} apply the step-5 filter to $\Delta P_{\text{ref}}^{\text{raw}}$. Output $\Delta P_{\text{ref}}$.
\STATE \textbf{8. Ablation:} $P_k^{\text{merge}} \leftarrow P_{k-1} \cup \Delta P_{\text{imp}} \cup \Delta P_{\text{ref}}$. Greedily drop any \emph{newly-added} annotator $\ell \in \Delta P_{\text{imp}} \cup \Delta P_{\text{ref}}$ whose removal strictly improves val macro F1 by at least $\tau_{\text{abl}}$.
\STATE \textbf{9. Subsumption prune:} drop pool members subsumed by a new one at Jaccard $\geq \tau_{\text{sub}}$. Output $P_k$.
\STATE \textbf{10. Viability:} $c_j \leftarrow c_j + 1$ if class $j$ F1 stagnated, else $c_j \leftarrow 0$.
\ENDFOR
\STATE \textbf{Aggregator:} fit EvoAgg LR-CV on $\mathcal{D}_{\text{val}}$ from $P_K$ votes and sentence-BERT features. Pseudo-label $\hat{y}_{\text{train}} \leftarrow \text{EvoAgg}(\mathcal{D}_{\text{train}}, P_K)$. Fine-tune the downstream model on $(\mathcal{D}_{\text{train}}, \hat{y}_{\text{train}})$.
\end{algorithmic}
\end{algorithm*}

\noindent Algorithm~\ref{alg:evopool} shows the full EvoPool training-time loop, with the query selector in line~2 detailed in App.~\ref{app:batchbald-full}, the selection gate in lines~8 and~9 in App.~\ref{app:selection-impl}, and per-task hyperparameter values in App.~\ref{app:hparams}.

\section{Design Justification}
\label{app:design}

This section expands the rationale behind key design choices in EvoPool that Sec.~\ref{sec:method} states but does not fully argue. Each choice corresponds to a Darwinian primitive applied to the annotator population. For completeness we briefly state the Lamarckian alternative and why we did not adopt it.

\paragraph{Pool as the unit of inheritance, not agents.}
A natural default in multi-agent LLM frameworks is the Lamarckian view, in which each agent carries verbal lessons, per-class personas, or hard-negative caches across iterations, and these compounding traces shape the next attempt. Reflexion, Voyager, and MemGPT \citep{shinn2023reflexion,wang2023voyager,packer2023memgpt} are canonical instances of this design and the approach works well in many settings. EvoPool deliberately diverges for two reasons. First, the agent prompt grows linearly in iteration count under the Lamarckian view, so each later-iteration call is materially more expensive in input tokens than the corresponding stateless call. Second, pool-level selection already provides the per-class feedback signal directly in our setting. An annotator that fails on the long tail is filtered out, whereas a verbal lesson about the long tail mostly amplifies what the agent already wanted to do. Agents in EvoPool are therefore stateless mutation operators re-invoked from scratch every iteration, and selection acts on the pool of surviving annotators rather than on the agents that authored them. App.~\ref{app:memory-ablation} reports the empirical comparison on ChemProt. Stateless is not just cheaper but also slightly better, beating Reflector memory by $+0.014$ and Full memory by $+0.040$.

\paragraph{Three role-specialized agents rather than one.}
The Generator, Improver, and Refiner partition the exploration-exploitation tradeoff of evolutionary search. The Generator provides breadth-first coverage of the label space, the Improver concentrates variation on weak classes using per-class feedback, and the Refiner broadens already-validated rules to expand coverage around high-fitness candidates. A single agent with a unified prompt has to time-share these objectives, and ablating any one of the three strictly degrades pool quality on both datasets in Table~\ref{tab:ablation-loo}.

\paragraph{Deterministic selection gate.}
We compose three orthogonal filters, namely niche-redundancy deduplication, validation-precision floor, and post-iteration marginal-contribution ablation, instead of learning a candidate ranker. Each filter implements a distinct Darwinian primitive of niche differentiation, viability threshold, and free-rider elimination, and the deterministic form gives a stable, auditable inheritance rule that does not itself drift across iterations. A learned ranker would couple the selection gate to the noise model of the validation set and would remove the strict-improvement guarantee on the pool.

\paragraph{Executable code as the unit of inheritance.}
The pool stores Python functions that emit a class label or abstain on each example, rather than continuous embeddings or learned models. Discrete code preserves auditability, runs at near-zero per-example cost at deployment, and admits direct comparison to hand-authored labeling functions in App.~\ref{app:human}. A continuous artifact would partly close the cost gap to a small fine-tuned model and lose the Python-rule auditability that high-stakes deployments require.

\paragraph{A complementary posterior-refinement reading.}
The Darwinian framing makes the mechanism interpretable, and a complementary Bayesian reading makes the loop's objective explicit and links the design choices above to a single principled guarantee. The current pool $\mathcal{A}^{(t)}$ induces a pseudo-posterior $p(\theta \mid \mathcal{A}^{(t)})$ over the latent labeling function $\theta$, where consistency with the pool's votes on the validation set serves as evidence. Under this reading, the three role-specialized agents are amortized samplers from $p(\theta \mid \cdot)$ that propose new annotators, the deterministic selection gate is a posterior conditioning that retains only candidates passing the three filters, and BatchBALD query selection in App.~\ref{app:batchbald-full} is the greedy expected-information-gain step of sequential Bayesian experimental design. The selection gate is the only step where the validation set enters the loop, so it is the operational analog of conditioning the posterior on observed evidence. The strict-improvement ablation step yields a principled per-step monotonicity guarantee on the validation fitness, $F(\mathcal{A}^{(t+1)}) \geq F(\mathcal{A}_+^{(t)})$, given in Prop.~\ref{prop:monotone-quality} and App.~\ref{app:selection-impl}. The two readings agree on what survives across generations: under the Darwinian view, the fittest annotators, and under the Bayesian view, the annotators most consistent with the validation evidence.

\section{Pool-Level Selection: Formal Definitions}
\label{app:selection-impl}

This appendix gives the formal definition of each filter used by the deterministic selection gate of Sec.~\ref{sec:method} (Pool-Level Natural Selection). Let $\text{fires}(\ell) \subseteq \mathcal{D}_{\text{val}}$ denote the set of validation examples on which annotator $\ell$ emits a non-abstain prediction.

\paragraph{Niche redundancy (deduplication).}
A candidate annotator $\ell$ is dropped if its firing pattern overlaps an existing pool member $\ell' \in P_{k-1}$ above a Jaccard threshold $\tau_{\text{jac}}$:
\begin{equation}
\max_{\ell' \in P_{k-1}} \frac{|\text{fires}(\ell) \cap \text{fires}(\ell')|}{|\text{fires}(\ell) \cup \text{fires}(\ell')|} > \tau_{\text{jac}}.
\label{eq:niche-redundancy}
\end{equation}

\paragraph{Minimum competence (precision floor).}
Define the validation precision of annotator $\ell$ as
\begin{equation}
\text{prec}(\ell) = \mathbb{P}_{x \sim \mathcal{D}_{\text{val}}}\bigl[\ell(x){=}y(x) \mid \ell \text{ fires on } x\bigr].
\label{eq:precision}
\end{equation}
A candidate is dropped if $\text{prec}(\ell) < \tau_{\text{prec}}$ or $|\text{fires}(\ell)| < \tau_{\text{fires}}$.

\paragraph{Marginal contribution (post-iteration ablation).}
At the end of each generation we greedily drop a newly-added annotator $\ell \in \Delta P_{\text{imp}} \cup \Delta P_{\text{ref}}$ when its removal \emph{strictly improves} the pool's aggregated validation macro-F1 $F$ by at least the ablation tolerance:
\begin{equation}
F\bigl(P \setminus \{\ell\};\,\mathcal{D}_{\text{val}}\bigr) \;>\; F\bigl(P;\,\mathcal{D}_{\text{val}}\bigr) + \tau_{\text{abl}}.
\label{eq:marginal-contribution}
\end{equation}
This step only considers newly-proposed annotators as drop candidates. Parent-pool annotators that survived earlier generations are never re-evaluated. The aggregator $F$ is majority vote by default, and the aggregator-aware variant substitutes EvoAgg.

\paragraph{Threshold values.}
We use $\tau_{\text{jac}}{=}0.95$, $\tau_{\text{prec}}{=}0.25$, $\tau_{\text{fires}}{=}5$, $\tau_{\text{abl}}{=}0.001$ across all main-paper experiments. Full pipeline hyperparameters are listed in App.~\ref{app:hparams}.

\paragraph{Framework guarantee.}
The strict-improvement form of Eq.~\ref{eq:marginal-contribution} gives the ablation phase a clean monotonicity property: every drop strictly improves the pool's validation macro-F1.

\begin{proposition}[Ablation-phase monotonicity]
\label{prop:monotone-quality}
Let $F(\mathcal{A})$ denote the pool's aggregated validation macro-F1 under the default aggregator. Write a generation as a two-phase update $\mathcal{A}^{(t)} \to \mathcal{A}_+^{(t)} \to \mathcal{A}^{(t+1)}$, where $\mathcal{A}_+^{(t)} = \mathcal{A}^{(t)} \cup \Delta\mathcal{A}$ is the gate-filtered union of the parent pool and the new candidates passing Eqs.~\ref{eq:niche-redundancy}--\ref{eq:precision}, and $\mathcal{A}^{(t+1)} \subseteq \mathcal{A}_+^{(t)}$ is the post-ablation pool under Eq.~\ref{eq:marginal-contribution}. The ablation phase is strictly non-decreasing in $F$:
\[
F(\mathcal{A}^{(t+1)}) \;\geq\; F(\mathcal{A}_+^{(t)}),
\]
with each drop contributing a gain of at least $\tau_{\mathrm{abl}}$.
\end{proposition}

\begin{proof}
The ablation phase removes annotators sequentially. At each step, the greedy rule of Eq.~\ref{eq:marginal-contribution} drops $\ell$ from the current intermediate pool $\mathcal{P}$ only when $F(\mathcal{P} \setminus \{\ell\}) > F(\mathcal{P}) + \tau_{\mathrm{abl}}$, so $F$ strictly increases by at least $\tau_{\mathrm{abl}}$ on every drop. The phase terminates when no remaining candidate passes the strict-improvement condition. Telescoping over the $T$ accepted drops gives
\begin{align*}
F(\mathcal{A}^{(t+1)}) \;&=\; F(\mathcal{A}_+^{(t)}) + \sum_{i=1}^{T} \Delta_i \\
                       \;&\geq\; F(\mathcal{A}_+^{(t)}) + T \cdot \tau_{\mathrm{abl}},
\end{align*}
where each $\Delta_i > \tau_{\mathrm{abl}}$. \qedhere
\end{proof}

The addition phase $\mathcal{A}^{(t)} \to \mathcal{A}_+^{(t)}$ has no comparable formal guarantee: under majority vote and macro-F1, an admitted annotator with precision above $\tau_{\mathrm{prec}}$ may still tip individual votes the wrong way or shift per-class precision-recall balance, so $F(\mathcal{A}_+^{(t)})$ can fall below $F(\mathcal{A}^{(t)})$ in principle. The strict-improvement ablation step then claws back any such losses by removing the offending newly-added annotators, and we observe empirically across all main-paper experiments that the end-of-generation macro-F1 $F(\mathcal{A}^{(t+1)})$ is monotone non-decreasing across $t$ (cf.\ Fig.~\ref{fig:gate-pareto} in App.~\ref{app:pool-efficiency}). Prop.~\ref{prop:monotone-quality} formally guarantees the ablation half of this empirical monotonicity.

\section{Query Selection: BatchBALD}
\label{app:batchbald-full}

Sec.~\ref{sec:agent-search} introduces BatchBALD as the small example-selection tool used inside each EvoPool iteration. This appendix gives its definition, the per-example weighting used in our implementation, and three short propositions that justify the choice. The selector is not a paper contribution and we describe it only for reproducibility.

\paragraph{Setup.}
At each iteration $t$ the Improver and Refiner need a small batch $S \subset \mathcal{X}$ of training examples to reason over. We treat the current annotator pool $\mathcal{A}^{(t)} = \{a_1, \dots, a_K\}$ as an ensemble of $K$ hypotheses with smoothed predictives $p_k(y \mid x) \in \Delta^{|\mathcal{C}|-1}$, uniform when $a_k$ abstains and Laplace-smoothed indicator otherwise. The marginal predictive is $\bar{p}(y \mid x) = \tfrac{1}{K}\sum_k p_k(y \mid x)$.

\paragraph{Batch mutual information.}
The BatchBALD \citep{kirsch2019batchbald} acquisition score for a candidate batch $S$ is
\begin{equation}
\mathcal{I}(S) \;=\; H\!\bigl(Y_S \mid X_S\bigr) \;-\; \frac{1}{K}\sum_{k=1}^{K} H\!\bigl(Y_S \mid X_S, a_k\bigr),
\label{eq:batchbald-def}
\end{equation}
where $Y_S = (Y_i)_{i \in S}$ and $H$ is the Shannon entropy of $\bar{p}$ over joint label configurations on $S$. The first term is the ensemble's predictive entropy on $S$. The second is the expected per-hypothesis entropy. Their difference is the information $Y_S$ would give about which hypothesis explains the data, i.e.\ which rule structure the pool is missing. The joint-entropy term enforces batch-level non-redundancy: a second example near-duplicating one already in $S$ adds little marginal MI. Exact evaluation of $H(Y_S)$ scales as $|\mathcal{C}|^{|S|}$ and we approximate it by Monte-Carlo sampling under $\bar{p}$, matching the original implementation.

\paragraph{Per-example importance weights.}
Greedy selection is biased by a per-example weight $w_i = u_i \cdot \alpha_i$ that decomposes into label uncertainty $u_i$ and class priority $\alpha_i$. With vote depth $d_i = |\{j : L_{ij} \neq \bot\}|$ and normalized vote distribution $\mathbf{v}_i \in \Delta^{|\mathcal{C}|-1}$,
\begin{equation}
u_i = \begin{cases}
  1 & d_i = 0, \\
  H(\mathbf{v}_i) / \log |\mathcal{C}| & d_i > 0,\ \text{votes disagree}, \\
  1 / (d_i + 1) & d_i > 0,\ \text{votes agree}.
\end{cases}
\end{equation}
Class priority $\alpha_i$ is normalized inverse per-class validation recall, $\alpha_c \propto \widehat{R}_c^{-\gamma}$, assigned per example via the aggregated pool label.

\paragraph{Algorithm and complexity.}
We use standard greedy maximization. Starting from $S_0 = \emptyset$, we iteratively add the example with largest marginal MI gain
\[
x^{*} \;=\; \arg\max_{x \notin S_t} \,w_x \cdot \bigl[\mathcal{I}(S_t \cup \{x\}) - \mathcal{I}(S_t)\bigr]
\]
until $|S| = k_{\max}$. One round costs $O(nKC)$ for marginal entropies plus $O(k\,n\,M\,K)$ for the joint-entropy MC updates, with $K \leq 300$, $C \leq 10$, $M = 500$, $n$ the candidate pool size. Wall-clock is well under a minute on a single CPU and the LLM calls dominate iteration cost.

\begin{table*}[!t]
\centering
\small
\setlength{\tabcolsep}{4pt}
\renewcommand{\arraystretch}{1.1}
\caption{\textbf{Dataset statistics across the 10 datasets.} \emph{\#C} is the number of label classes. \emph{Type} is single-label (S) or multi-label (M). \emph{Val} is the post-cap validation budget used in all experiments.}
\label{tab:dataset-stats}
\begin{tabular}{@{}l|l|c|c|rrr|p{2.6in}@{}}
\toprule
Dataset    & Category                & \#C & Type & Train & Val & Test & Task description \\
\midrule
AGNews     & General classification  & 4   & S    & 9{,}000   & 500 & 7{,}600 & News topic classification across World, Sports, Business, and Sci/Tech \\
Banking77  & General classification  & 77  & S    & 9{,}003   & 500 & 3{,}080 & Fine-grained intent classification in retail-banking customer service \\
\midrule
ChemProt   & High-stakes specialized & 10  & S    & 12{,}861  & 500 & 1{,}607 & Biomedical chemical-protein relation extraction over PubMed abstracts \\
DDI        & High-stakes specialized & 5   & S    & 32{,}670  & 500 & 8{,}673 & Drug-drug interaction relation extraction over DrugBank and MedLine sentences \\
Claude9    & High-stakes specialized & 9   & S    & 5{,}937   & 225 & 1{,}564 & Unfair-clause classification in consumer Terms-of-Service contracts \\
\midrule
FEVER      & Complex reasoning       & 3   & S    & 10{,}000  & 500 & 7{,}600 & Open-domain claim verification against Wikipedia evidence \\
VitaminC   & Complex reasoning       & 3   & S    & 10{,}000  & 500 & 7{,}600 & Counterfactual claim verification under contrastive evidence pairs \\
SciFact    & Complex reasoning       & 3   & S    & 809       & 210 & 90     & Scientific claim verification against biomedical research abstracts \\
\midrule
Ohsumed    & Multi-label biomedical  & 23  & M    & 24{,}074  & 500 & 6{,}877 & MeSH disease-topic tagging on cardiovascular biomedical abstracts \\
PubMed     & Multi-label biomedical  & 14  & M    & 40{,}000  & 500 & 5{,}000 & Topic tagging on biomedical literature abstracts \\
\bottomrule
\end{tabular}
\end{table*}

\subsection{Theoretical Properties}

\begin{proposition}[Monotone submodularity and greedy guarantee]
\label{prop:batchbald-submod}
For any finite candidate set and finite ensemble, $\mathcal{I}: 2^{\mathcal{X}} \to \mathbb{R}_{\geq 0}$ in Eq.~\eqref{eq:batchbald-def} is monotone non-decreasing and submodular. Greedy maximization achieves, after $k$ steps,
\[
\mathcal{I}(S_{\mathrm{greedy}}) \;\geq\; (1 - 1/e)\,\max_{|S| \leq k} \mathcal{I}(S).
\]
\end{proposition}

\begin{proof}
Let $A \in \{a_1, \dots, a_K\}$ be a uniform-random ensemble index, so that $\mathcal{I}(S) = I(Y_S;\, A \mid X_S)$ rewrites Eq.~\eqref{eq:batchbald-def}. Conditional on a specific hypothesis $A = a_k$, distinct annotators label examples independently, so $\{Y_i\}_{i \in S}$ are conditionally independent given $A$. This gives the key decomposition
\[
H(Y_S \mid X_S, A) \;=\; \sum_{i \in S} H(Y_i \mid X_i, A).
\]

\emph{Monotonicity.} For any $S$ and $x \notin S$, the chain rule yields
\[
\mathcal{I}(S \cup \{x\}) - \mathcal{I}(S) \;=\; I(Y_x;\, A \mid X_x, Y_S, X_S) \;\geq\; 0,
\]
since conditional mutual information is non-negative.

\emph{Submodularity.} For any $T \subseteq S$ and $x \notin S$, write the marginal gain at $T$ as
\begin{align*}
\mathcal{I}(T \cup \{x\}) - \mathcal{I}(T) \;&=\; H(Y_x \mid X_x, Y_T) \\
&\quad - H(Y_x \mid X_x, A),
\end{align*}
where the conditional-independence decomposition drops $Y_T$ from the second term. Similarly,
\begin{align*}
\mathcal{I}(S \cup \{x\}) - \mathcal{I}(S) \;&=\; H(Y_x \mid X_x, Y_S) \\
&\quad - H(Y_x \mid X_x, A).
\end{align*}
Conditioning on a superset cannot increase entropy, so $H(Y_x \mid X_x, Y_T) \geq H(Y_x \mid X_x, Y_S)$ for $T \subseteq S$. Subtracting the two marginal gains,
\[
\mathcal{I}(T \cup \{x\}) - \mathcal{I}(T) \;\geq\; \mathcal{I}(S \cup \{x\}) - \mathcal{I}(S),
\]
which is the diminishing-returns property defining submodularity.

\emph{Greedy bound.} $\mathcal{I}$ is monotone submodular and non-negative, so Nemhauser--Wolsey--Fisher \citep{nemhauser1978analysis} directly gives
\[
\mathcal{I}(S_{\mathrm{greedy}}) \;\geq\; (1 - 1/e)\,\max_{|S| \leq k} \mathcal{I}(S). \qedhere
\]
\end{proof}

\section{Feature Enrichment for Complex-Reasoning Tasks}
\label{app:enrichment}

Text-classification tasks like ChemProt give an annotator everything it needs in the raw input. Claim-verification tasks like FEVER, SciFact, and VitaminC do not. What actually matters there is the relationship between a claim and a piece of evidence, and a regex over either side alone cannot recover that relationship. In an early FEVER pilot the agent loop saw only the raw text and produced a pool that always predicted NotEnoughInfo.

We extend EvoPool to these tasks without changing the agent loop. For every claim and evidence pair we precompute a small set of comparison features once at dataset preparation and attach them to the example metadata. The agents then write annotators that combine these ready-made features instead of trying to reconstruct them from text.

The simplest features need no model and capture what a careful reader would check by eye. They include the share of vocabulary the claim and evidence have in common, whether they mention the same named entities, whether one side carries a negation that the other does not, whether the numbers agree, and the length ratio between the two. Three model-backed signals are added on top at a one-time preprocessing cost. A small MiniLM sentence encoder \citep{reimers2019sbert} gives a semantic similarity score. A DeBERTa-v3-base NLI model \citep{he2021deberta} pretrained on MNLI and FEVER-NLI gives entail, neutral, and contradict probabilities. A small English dependency parser pulls out the subject, verb, and object of the claim and checks each one inside the evidence, and a WordNet lookup extends the antonym table.

The lexical features alone are already enough for the agent loop to author non-trivial Supports, Refutes, and NotEnoughInfo annotators on FEVER as reported in Table~\ref{tab:top-lf-fever}. The model-backed signals are added on top for FEVER and SciFact in the main table. Because the enriched data is a strict superset of the original, annotators written against the base schema stay valid, and the enrichment models are all publicly available pretrained checkpoints with no recurring cost at training or deployment.

\paragraph{Fair access across methods.}
All baselines (Alchemist, DataSculpt, LLM-Annotation, LLM-Eval) and EvoPool receive identical pre-computed features through the example metadata, so differences in performance reflect each method's ability to compose those features into rules rather than differential feature access. On FEVER, Alchemist and DataSculpt reach only $0.167$ and $0.034$ test macro-F1 with the same enriched metadata, while EvoPool reaches $0.812$ and LLM-Annotation $0.760$, isolating compositional multi-agent authoring as the source of EvoPool's lift on complex-reasoning tasks.

\section{Experiment Details}
\label{app:hparams}

This appendix consolidates the practical details needed to reproduce EvoPool: dataset statistics, training configuration for both the EvoPool training-time loop and the downstream models, compute resources, and the evaluation protocol.

\subsection{Dataset Statistics}
\label{app:datasets}

Table~\ref{tab:dataset-stats} reports the per-dataset statistics across the 10 evaluation datasets organized into the four task categories of Sec.~\ref{sec:setup}.

\paragraph{Train and test splits.}
Train uses the full processed train split with labels removed. Test uses the full test split with no subsampling. For multi-label datasets (Ohsumed, PubMed) we evaluate against the multi-hot ground-truth target.

\paragraph{Validation budget.}
We standardize the validation split to a fixed budget of $|\mathcal{D}_{\text{val}}|=500$ examples across all datasets. Two datasets have a validation set smaller than $500$ (SciFact $=210$, Claude9 $=225$). The same fixed $\mathcal{D}_{\text{val}}$ is reused across the EvoPool selection gate, EvoAgg's LR-CV fitting, and CFT clean-label refinement.

\paragraph{Licenses and terms of use.}
All datasets and pretrained models we use are publicly released for research, and we access each LLM API under its standard terms of service. Our use is consistent with each artifact's intended research purpose, and any code and annotator pools we release are for research use only.

\subsection{Training Details}
\label{app:training}

\paragraph{EvoPool training-time loop.}
The loop runs for $K=12$ generations with gpt-4o-mini as the primary agent backbone, decoding temperature $0.5$, and single run with fixed seed $42$. Cross-backbone variants such as gpt-5-mini, gemini-3.1-pro-preview, and deepseek-v4-flash hold every other configuration constant and only swap the LLM author. The Generator, Improver, and Refiner agents are re-invoked from scratch every generation with no persistent state. All LLM calls go through the OpenAI, Google, and DeepSeek APIs, and we cache responses by prompt hash so that resumed runs do not re-pay for already-computed calls. Table~\ref{tab:framework-hparams} lists the full configuration.

\begin{table}[h]
\centering
\small
\setlength{\tabcolsep}{4pt}
\renewcommand{\arraystretch}{1.04}
\caption{Default EvoPool framework hyperparameters used for all experiments.}
\label{tab:framework-hparams}
\begin{tabular}{@{}ll|l@{}}
\toprule
Block & Hyperparameter & Value \\
\midrule
\multirow{3}{*}{LLM}
 & Primary backbone         & gpt-4o-mini \\
 & Decoding temperature     & 0.5 \\
 & Seed                     & 42 \\
\midrule
Loop
 & \# generations $K$       & 12 \\
\midrule
\multirow{3}{*}{Generator}
 & \# calls per generation  & 18 \\
 & Min val precision        & 0.30 \\
 & Min fires                & 5 \\
\midrule
\multirow{5}{*}{Improver}
 & \# calls per generation  & 6 \\
 & Min val precision        & 0.25 \\
 & Min fires                & 5 \\
 & Per-class F1 target      & $<\!0.40$ \\
 & \# pos examples/class    & 8 \\
\midrule
\multirow{3}{*}{Refiner}
 & Min precision to refine  & 0.55 \\
 & Min iter to activate     & 3 \\
 & Max Jaccard overlap      & 0.95 \\
\midrule
\multirow{4}{*}{Selection gate}
 & $\tau_{\text{jac}}$ (Jaccard)    & 0.95 \\
 & $\tau_{\text{prec}}$ (precision) & 0.25 \\
 & $\tau_{\text{fires}}$ (fires)    & 5 \\
 & $\tau_{\text{abl}}$ (ablation)   & 0.001 \\
\midrule
\multirow{3}{*}{BatchBALD}
 & Batch size $k_{\max}$    & 60 \\
 & MC samples $M$           & 500 \\
 & Class-priority exp.\ $\gamma$ & 1.0 \\
\midrule
\multirow{3}{*}{EvoAgg}
 & Text embedder            & all-MiniLM-L6-v2 \\
 & LR-CV $Cs$               & $\{0.01, 0.1, 1, 10\}$ \\
 & CV folds                 & 5 (stratified) \\
\bottomrule
\end{tabular}
\end{table}

\paragraph{Downstream models.}
We fine-tune three downstream models on the EvoAgg pseudo-labels: RoBERTa-large with full fine-tuning, Qwen3-1.7B and Llama-3.1-8B with LoRA \citep{hu2022lora}. Table~\ref{tab:downstream-hparams} summarizes the per-model training configuration. All three use AdamW with a linear warmup followed by linear decay schedule. The classification head is a single $[\mathrm{CLS}]$ representation for RoBERTa and a last-token representation for the decoder LoRA models, mapped through a softmax over $|\mathcal{C}|$ classes. Multi-label datasets replace the softmax with $|\mathcal{C}|$ independent sigmoid heads trained with binary cross-entropy on the multi-hot EvoAgg pseudo-target.

\begin{table}[h]
\centering
\small
\setlength{\tabcolsep}{3pt}
\renewcommand{\arraystretch}{1.04}
\caption{Downstream model training hyperparameters. RoBERTa-large is fully fine-tuned. Qwen3-1.7B and Llama-3.1-8B use LoRA with the listed adapter configuration. Effective batch size for the LoRA models is $\text{batch\_size} \times \text{grad\_accum}$.}
\label{tab:downstream-hparams}
\resizebox{\linewidth}{!}{%
\begin{tabular}{@{}l|ccc@{}}
\toprule
Hyperparameter      & RoBERTa-large & Qwen3-1.7B & Llama-3.1-8B \\
\midrule
Tuning method        & Full FT  & LoRA          & LoRA \\
LoRA rank $r$        & ---      & 16            & 16 \\
LoRA $\alpha$        & ---      & 32            & 32 \\
LoRA dropout         & ---      & 0.05          & 0.05 \\
LoRA targets         & ---      & q/k/v/o\_proj & q/k/v/o\_proj \\
Optimizer            & AdamW    & AdamW         & AdamW \\
Pre-train LR         & 2e-5     & 1e-4          & 1e-4 \\
CFT LR               & 2e-5     & 5e-5          & 5e-5 \\
Weight decay         & 0.01     & 0.0           & 0.0 \\
Batch size           & 32       & 8             & 8 \\
Grad accum.\ steps   & 1        & 2             & 2 \\
Warmup ratio         & 0.10     & 0.06          & 0.06 \\
Pre-train epochs     & 3        & 3             & 3 \\
CFT epochs           & 5        & 5             & 5 \\
Max seq.\ length     & 128      & 256           & 256 \\
Precision            & FP16     & BF16          & BF16 \\
Seed                 & 42       & 42            & 42 \\
\bottomrule
\end{tabular}}
\end{table}

\paragraph{Multi-label setting.}
For Ohsumed and PubMed, we replace the softmax classification head with $|\mathcal{C}|$ independent sigmoid heads trained with binary cross-entropy on the multi-hot EvoAgg pseudo-target. Per-class decision threshold defaults to $0.5$, with no threshold tuning on validation.

\paragraph{Self-Training baseline.}
Following Noisy Student \citep{xie2020selftraining}, Self-Training is a two-stage procedure. We first train the downstream model on the available validation set with clean labels, then use that teacher model to pseudo-label the unlabeled train set, and finally re-train the same downstream architecture on the pseudo-labeled train set. We pair this with a separate \emph{Golden} cell (sometimes shown as \emph{Golden only}) that trains on the validation labels alone and serves as the validation-only supervised reference. Both cells use the same optimizer, learning rate, and epoch schedule as the pseudo-label pre-train configuration in Table~\ref{tab:downstream-hparams}.

\paragraph{Continued fine-tuning.}
After pseudo-label pre-training, we additionally fine-tune the downstream model on the full available validation set following the protocol of \citet{zhang2024stronger}. CFT uses the same optimizer and schedule as the pre-training stage but with the listed cft\_epochs and CFT learning rate.

\subsection{Compute Resources}
\label{app:compute}

All downstream training and Self-Training experiments run on a single NVIDIA A100 80GB GPU. The EvoPool training-time loop itself is CPU-only: agents call the LLM API, and BatchBALD query selection runs in well under a minute per generation on a single CPU (see App.~\ref{app:batchbald-full}).

\subsection{Evaluation Details}
\label{app:eval}

\paragraph{Annotator-level metrics.}
We report five metrics on both validation and test for the authored pool's predictions: coverage, accuracy, accuracy-on-covered, sample-F1, and macro-F1. Coverage is the fraction of rows on which at least one annotator fires. Accuracy is computed over the full split with abstention treated as an incorrect prediction. Accuracy-on-covered is the same accuracy restricted to the covered subset and is reported as a diagnostic only. All F1 metrics are abstain-aware: when an annotator pool abstains on an example, that example still contributes to the denominator of recall for its ground-truth class, so a high-coverage pool with moderate accuracy can outperform a low-coverage pool with high accuracy-on-covered. EvoAgg by design has coverage $=1.0$ and so the abstain-aware and abstain-naive macro-F1 coincide for our headline cells.

\paragraph{Downstream metrics.}
For downstream models we report accuracy, sample-F1, and macro-F1 on test, with macro-F1 as the headline metric used throughout the main tables. For multi-label datasets we additionally report sample-F1, the average per-example F1 across all label positions, since classical accuracy degenerates to exact-match subset accuracy under multi-hot targets.

\paragraph{Macro-F1 reporting convention.}
Throughout the paper we report macro-F1, i.e.\ the unweighted mean of per-class F1, as the primary metric since the high-stakes specialized and multi-label datasets exhibit strong class imbalance. Weighted-F1 columns are reported only in App.~\ref{app:full-annotator} and App.~\ref{app:full-downstream}.

\section{Additional Results}
\label{app:additional}

\begin{table}[h]
\centering
\small
\setlength{\tabcolsep}{6pt}
\caption{\textbf{Darwinian vs.\ two Lamarckian variants.} Experiment done on ChemProt with gpt-4o-mini backbone. We report annotator-level macro-F1 on test set.}
\label{tab:memory-ablation}
\begin{tabular}{@{}l|cc@{}}
\toprule
Variant & Test macro-F1 & $\Delta$ \\
\midrule
\rowcolor{cyan!5}
\textbf{Darwinian (EvoPool)}      & \textbf{0.5826} & --- \\
Lamarckian (Reflector memory)     & 0.5688 & $-0.014$ \\
Lamarckian (Full memory)          & 0.5428 & $-0.040$ \\
\bottomrule
\end{tabular}
\end{table}

\subsection{Darwinian vs.\ Lamarckian Agents}
\label{app:memory-ablation}

The default EvoPool configuration is Darwinian: each agent (Generator, Improver, Refiner) is re-invoked from scratch every generation, with the annotator pool $P_k$ as the only inheritance channel. The multi-agent literature \citep{shinn2023reflexion,wang2023voyager,packer2023memgpt} has converged on cross-iteration agent state (verbal lessons, per-class personas, hard-negative caches) as a Lamarckian alternative. We ablate our Darwinian configuration against two Lamarckian variants on ChemProt:

\begin{itemize}[leftmargin=13pt,nosep]
\item \textbf{Darwinian}, ours: no cross-iteration agent state. The pool itself carries all information across generations.
\item \textbf{Lamarckian, Reflector memory}: after each generation, a separate Reflector agent summarizes the pool's per-class performance into a short verbal ``lessons'' note such as ``class 4 is under-served, add Agonist patterns'', which is appended to the next generation's Generator and Improver prompts.
\item \textbf{Lamarckian, Full memory}: per-class persona for the Improver, a verified-rule cache across iterations, competence streaks counting consecutive generations a class was below threshold, and hard-negative examples surfaced from prior generations.
\end{itemize}

Table~\ref{tab:memory-ablation} reports the empirical comparison. The Darwinian configuration achieves the best macro-F1 of the three, with a modest gap of $+0.014$ over Reflector and $+0.040$ over Full memory. It is also materially cheaper at training time. Because the agent prompt does not grow with iteration count, each later-generation call costs the same number of input tokens as the first, while the Lamarckian variants accumulate verbal state and pay an input-token premium that scales linearly with $K$. Along both axes of slightly higher macro-F1 and lower per-call cost, the Darwinian framing is the default choice. The mechanism that explains both observations is the same: pool-level selection already provides the per-class feedback signal directly, so additional agent state mostly amplifies what selection has already conditioned on, rather than introducing new information.

\subsection{Query Selector Ablation}
\label{app:selector-ablation}

The Improver and Refiner act on a small batch of training examples picked by a query selector each generation. EvoPool uses BatchBALD by default (App.~\ref{app:batchbald-full}). Here we ablate the selector while holding every other component of the pipeline fixed. Table~\ref{tab:selector-ablation} reports the ChemProt majority-vote macro-F1 for BatchBALD against two natural baselines, classical uncertainty sampling and uniform random sampling. BatchBALD leads by roughly $0.04$--$0.05$, comparable in magnitude to the EvoAgg-vs-MV gap reported in the main paper, indicating that selector quality is a load-bearing component rather than a free choice. Random performs comparably to uncertainty, consistent with the observation that margin-based uncertainty alone tends to collapse on already-covered easy regions and provides little signal beyond a uniform draw.

\begin{table}[h]
\centering
\small
\setlength{\tabcolsep}{8pt}
\caption{\textbf{Selector ablation}. All experiments are done on ChemProt using gpt-4o-mini backbone and share the default EvoPool configuration with the Generator, Improver, Refiner, and selection gate held fixed, and differ only in the query selector. We report majority-vote test set macro-F1.}
\label{tab:selector-ablation}
\begin{tabular}{@{}l|cc@{}}
\toprule
Selector & MV macro-F1 & $\Delta$ \\
\midrule
\rowcolor{cyan!5}
\textbf{BatchBALD (EvoPool)}   & \textbf{0.4787} & --- \\
Uncertainty sampling           & 0.4300 & $-0.049$ \\
Random                         & 0.4336 & $-0.045$ \\
\bottomrule
\end{tabular}
\end{table}

\subsection{Pool Co-evolution Signals on FEVER}
\label{app:coevolution-fever}

Figure~\ref{fig:coevolution-fever} reproduces the four-panel co-evolution analysis of Figure~\ref{fig:coevolution-chemprot} on FEVER with the gpt-4o-mini backbone. The story is qualitatively similar but quantitatively softer than on ChemProt. Because FEVER has only three classes and the initial pool already covers all three, the agent niches in panel (a) are less differentiated: the Generator authors $19/34/32$ surviving annotators across Supports/Refutes/NEI rather than leaving any class empty, and the Improver is a balanced rebalancer ($25/14/16$) rather than a niche-filler. Panel (b) shows the same staircase activation timing as on ChemProt, but the Refiner is more productive on FEVER, contributing $11$ surviving annotators versus $2$ on ChemProt because the broadened verification-rule variants pass the Jaccard floor more easily on a $3$-class label space. Panel (c)'s stacked bars show the Improver's per-generation target distribution shifting from a Supports-heavy mix at $k{=}1$--$2$ toward more balanced Refutes/NEI contributions in later generations as the gate consumes the easy Supports patterns first. Panel (d) shows broadly similar gate behavior, with $40$--$60$\% of Improver candidates killed and a higher Refiner kill rate.

\begin{figure*}[!t]
\centering
\includegraphics[width=0.95\linewidth]{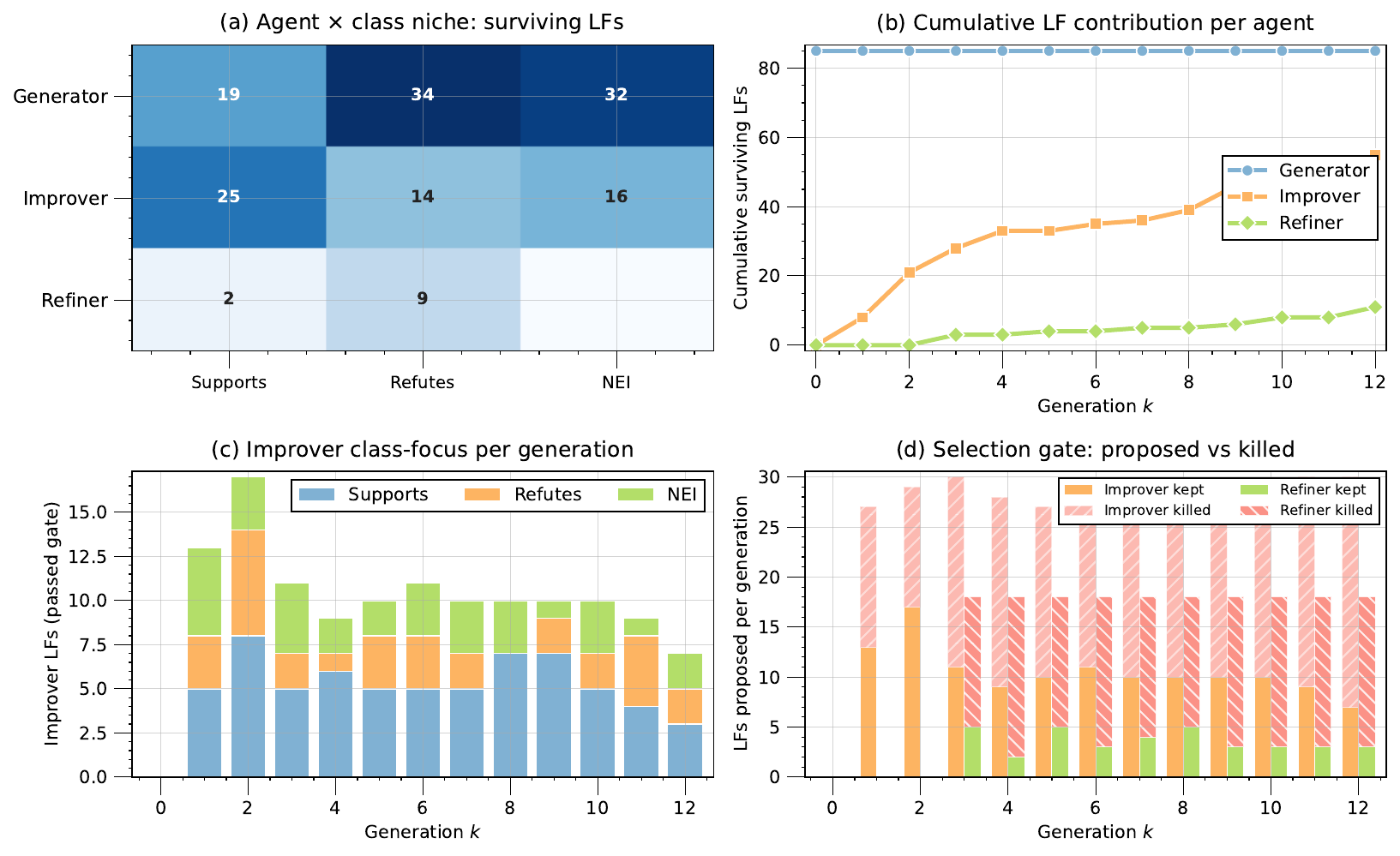}
\vspace{-6pt}
\caption{\textbf{Pool co-evolution signals on FEVER.} Experiment done with gpt-4o-mini backbone. Same panel layout as Figure~\ref{fig:coevolution-chemprot}: (a) agent $\times$ class niche of surviving annotators, (b) cumulative annotator contribution per agent, (c) Improver class-focus per generation (stacked by class since FEVER has only three classes), (d) selection gate proposed vs.\ killed per generation. With only three classes the initial pool already covers all of them, so the agent niches in (a) are less differentiated than on ChemProt. The Refiner is more productive ($11$ surviving annotators vs.\ $2$ on ChemProt).}
\label{fig:coevolution-fever}
\end{figure*}

\begin{table*}[tb!]
\caption{\textbf{Downstream test macro-F1 after continued fine-tuning on the full validation set, across 10 datasets $\times$ 3 downstream models, organized into 4 task categories.} Each method's pool first produces pseudo-labels for the train set. The downstream model is then further fine-tuned on the clean validation labels. Models: RoB$=$RoBERTa-large, Qwen$=$Qwen3-1.7B, Llama$=$Llama-3.1-8B. All pools constructed using gpt-4o-mini backbone.}
\label{tab:cft}
\centering

\begin{minipage}[t]{0.48\linewidth}
\centering\footnotesize
\textbf{(1) General classification}\\[2pt]
\setlength{\tabcolsep}{3pt}
\renewcommand{\arraystretch}{0.95}
\begin{tabular}{@{}ll|ccc@{}}
\toprule
Dataset & Method & RoB & Qwen & Llama \\
\midrule
\multirow{4}{*}{AGNews}
 & Alchemist            & 0.8579 & 0.8307 & 0.8705 \\
 & DataSculpt           & 0.8927 & 0.8765 & 0.8711 \\
 & LLM-Annotation       & \textbf{0.8948} & 0.8926 & \textbf{0.8956} \\
 & \cellcolor{cyan!5}\textbf{EvoPool} & \cellcolor{cyan!5}0.8842 & \cellcolor{cyan!5}\textbf{0.8955} & \cellcolor{cyan!5}0.8899 \\
\cmidrule(lr){1-5}
\multirow{4}{*}{Banking77}
 & Alchemist            & 0.2759 & 0.0682 & 0.2678 \\
 & DataSculpt           & 0.3515 & 0.0787 & 0.3715 \\
 & LLM-Annotation       & 0.7376 & 0.7398 & 0.7915 \\
 & \cellcolor{cyan!5}\textbf{EvoPool} & \cellcolor{cyan!5}\textbf{0.7520} & \cellcolor{cyan!5}\textbf{0.7661} & \cellcolor{cyan!5}\textbf{0.7962} \\
\bottomrule
\end{tabular}

\vspace{6pt}

\textbf{(3) Complex reasoning}\\[2pt]
\setlength{\tabcolsep}{3pt}
\renewcommand{\arraystretch}{0.95}
\begin{tabular}{@{}ll|ccc@{}}
\toprule
Dataset & Method & RoB & Qwen & Llama \\
\midrule
\multirow{4}{*}{FEVER}
 & Alchemist            & 0.2969 & 0.7109 & 0.6390 \\
 & DataSculpt           & 0.3672 & 0.6283 & 0.5083 \\
 & LLM-Annotation       & 0.7767 & 0.7988 & 0.8236 \\
 & \cellcolor{cyan!5}\textbf{EvoPool} & \cellcolor{cyan!5}\textbf{0.8227} & \cellcolor{cyan!5}\textbf{0.8360} & \cellcolor{cyan!5}\textbf{0.8518} \\
\cmidrule(lr){1-5}
\multirow{4}{*}{VitaminC}
 & Alchemist            & 0.1667 & 0.7267 & 0.4642 \\
 & DataSculpt           & 0.3997 & 0.5366 & 0.4319 \\
 & LLM-Annotation       & \textbf{0.7097} & \textbf{0.7411} & 0.7805 \\
 & \cellcolor{cyan!5}\textbf{EvoPool} & \cellcolor{cyan!5}0.7038 & \cellcolor{cyan!5}0.7383 & \cellcolor{cyan!5}\textbf{0.7840} \\
\cmidrule(lr){1-5}
\multirow{4}{*}{SciFact}
 & Alchemist            & 0.2051 & 0.2613 & 0.3248 \\
 & DataSculpt           & 0.2112 & 0.2901 & 0.3609 \\
 & LLM-Annotation       & \textbf{0.5134} & \textbf{0.4697} & 0.4647 \\
 & \cellcolor{cyan!5}\textbf{EvoPool} & \cellcolor{cyan!5}0.2981 & \cellcolor{cyan!5}0.3871 & \cellcolor{cyan!5}\textbf{0.5088} \\
\bottomrule
\end{tabular}
\end{minipage}\hfill
\begin{minipage}[t]{0.48\linewidth}
\centering\footnotesize
\textbf{(2) High-stakes specialized}\\[2pt]
\setlength{\tabcolsep}{3pt}
\renewcommand{\arraystretch}{0.95}
\begin{tabular}{@{}ll|ccc@{}}
\toprule
Dataset & Method & RoB & Qwen & Llama \\
\midrule
\multirow{4}{*}{ChemProt}
 & Alchemist            & 0.2576 & 0.3262 & 0.3803 \\
 & DataSculpt           & 0.3877 & 0.3177 & 0.4082 \\
 & LLM-Annotation       & 0.4366 & 0.3725 & 0.3993 \\
 & \cellcolor{cyan!5}\textbf{EvoPool} & \cellcolor{cyan!5}\textbf{0.5557} & \cellcolor{cyan!5}\textbf{0.5394} & \cellcolor{cyan!5}\textbf{0.5432} \\
\cmidrule(lr){1-5}
\multirow{4}{*}{DDI}
 & Alchemist            & 0.2687 & 0.2671 & 0.2716 \\
 & DataSculpt           & 0.1746 & 0.2955 & 0.3463 \\
 & LLM-Annotation       & 0.2707 & 0.3946 & 0.4549 \\
 & \cellcolor{cyan!5}\textbf{EvoPool} & \cellcolor{cyan!5}\textbf{0.5974} & \cellcolor{cyan!5}\textbf{0.5745} & \cellcolor{cyan!5}\textbf{0.6112} \\
\cmidrule(lr){1-5}
\multirow{4}{*}{Claude9}
 & Alchemist            & 0.2160 & 0.3240 & 0.2526 \\
 & DataSculpt           & 0.1057 & 0.1203 & 0.1610 \\
 & LLM-Annotation       & \textbf{0.5031} & \textbf{0.4890} & \textbf{0.4660} \\
 & \cellcolor{cyan!5}\textbf{EvoPool} & \cellcolor{cyan!5}0.3528 & \cellcolor{cyan!5}0.4500 & \cellcolor{cyan!5}0.4259 \\
\bottomrule
\end{tabular}

\vspace{6pt}

\textbf{(4) Multi-label biomedical}\\[2pt]
\setlength{\tabcolsep}{3pt}
\renewcommand{\arraystretch}{0.95}
\begin{tabular}{@{}ll|ccc@{}}
\toprule
Dataset & Method & RoB & Qwen & Llama \\
\midrule
\multirow{4}{*}{Ohsumed}
 & Alchemist            & 0.2012 & 0.2071 & 0.2335 \\
 & DataSculpt           & 0.2002 & 0.2051 & 0.2661 \\
 & LLM-Annotation       & 0.5442 & 0.5316 & 0.5636 \\
 & \cellcolor{cyan!5}\textbf{EvoPool} & \cellcolor{cyan!5}\textbf{0.6131} & \cellcolor{cyan!5}\textbf{0.6111} & \cellcolor{cyan!5}\textbf{0.6375} \\
\cmidrule(lr){1-5}
\multirow{4}{*}{PubMed}
 & Alchemist            & 0.4967 & 0.4706 & 0.5684 \\
 & DataSculpt           & 0.3774 & 0.4824 & 0.5431 \\
 & LLM-Annotation       & 0.6561 & 0.6622 & 0.6727 \\
 & \cellcolor{cyan!5}\textbf{EvoPool} & \cellcolor{cyan!5}\textbf{0.7346} & \cellcolor{cyan!5}\textbf{0.7265} & \cellcolor{cyan!5}\textbf{0.7354} \\
\bottomrule
\end{tabular}
\end{minipage}

\vspace{-6pt}
\end{table*}

\subsection{Continued Fine-Tuning (CFT)}
\label{app:cft}

After fitting each method's downstream model on its train-set pseudo-labels, we additionally continue fine-tuning the downstream model on the validation set, following the protocol of \citet{zhang2024stronger}. Table~\ref{tab:cft} reports test macro-F1 with CFT applied across the three downstream models. EvoPool+CFT achieves the best or tied-best macro-F1 across the majority of model-dataset cells. SciFact and Claude9 are the boundary cases where LLM annotation+CFT retains the edge, consistent with the pre-CFT ordering. CFT closes a larger fraction of the gap for the weaker one-shot baselines such as Alchemist and DataSculpt than for EvoPool, since EvoPool's pseudo-labels are already close to oracle quality and leave less headroom for clean-label refinement.

\begin{table*}[tb!]
\centering
\small
\renewcommand{\arraystretch}{1.06}
\caption{Downstream test macro-F1 of EvoPool with gpt-4o-mini backbone vs.\ human-LF pseudo-labels across three downstream models on AGNews, ChemProt, Claude9, and Banking77. \textbf{Bold} marks the higher of Human and EvoPool at each no CFT vs.\ CFT pairing per model per task.}
\label{tab:human-comparison-downstream}
\resizebox{0.99\linewidth}{!}{
\begin{tabular}{@{}l|ccc|ccc|ccc|ccc@{}}
\toprule
\multirow{2}{*}{Method}
 & \multicolumn{3}{c|}{\textbf{AGNews}}
 & \multicolumn{3}{c|}{\textbf{ChemProt}}
 & \multicolumn{3}{c|}{\textbf{Claude9}}
 & \multicolumn{3}{c}{\textbf{Banking77}} \\
\cmidrule(lr){2-4}\cmidrule(lr){5-7}\cmidrule(lr){8-10}\cmidrule(lr){11-13}
 & RoB & Qwen & Llama & RoB & Qwen & Llama & RoB & Qwen & Llama & RoB & Qwen & Llama \\
\midrule
Human                                & 0.8280 & 0.8383 & 0.7952 & 0.4679 & 0.4718 & 0.4679 & 0.1945 & \textbf{0.5203} & \textbf{0.5190} & 0.3986 & 0.3520 & 0.3941 \\
Human + CFT                          & 0.8852 & 0.8612 & 0.8640 & 0.5170 & 0.5192 & 0.5170 & 0.3773 & 0.4358 & \textbf{0.5750} & 0.5696 & 0.4831 & 0.6471 \\
\rowcolor{cyan!5} EvoPool (ours)     & \textbf{0.8860} & \textbf{0.8951} & \textbf{0.8859} & \textbf{0.5396} & \textbf{0.5383} & \textbf{0.5909} & \textbf{0.3413} & 0.3827 & 0.3908 & \textbf{0.7332} & \textbf{0.7150} & \textbf{0.7241} \\
\rowcolor{cyan!5} EvoPool + CFT      & \textbf{0.8842} & \textbf{0.8955} & \textbf{0.8899} & \textbf{0.5557} & \textbf{0.5394} & \textbf{0.5432} & 0.3528 & \textbf{0.4500} & 0.4259 & \textbf{0.7520} & \textbf{0.7661} & \textbf{0.7962} \\
\bottomrule
\end{tabular}
}
\vspace{-10pt}
\end{table*}

\subsection{Comparison to Hand-Crafted Human LFs}
\label{app:human}

WRENCH \citep{zhang2021wrench} provides a hand-authored labeling-function pool for AGNews with 9 LFs. \citet{zhang2024stronger} provide hand-authored pools for ChemProt with 25 LFs, Claude9 with 41 LFs, and Banking77 with 335 LFs covering all 77 intents. These pools represent expert manual effort to author labeling rules by inspection.

\paragraph{Annotator-level.}
Table~\ref{tab:human-comparison-annotator} compares EvoPool to the hand-crafted pools on test macro-F1. EvoPool beats the human pool on AGNews by $+0.203$, on ChemProt by $+0.110$, and on Banking77 by $+0.398$. On Claude9 the human pool leads EvoPool by $0.191$. The Claude9 boundary case reflects expert legal-text reading that our machine-authored pool does not match.

\begin{table}[h]
\centering
\small
\setlength{\tabcolsep}{4pt}
\caption{Annotator-level test macro-F1 of hand-crafted human annotators vs.\ EvoPool with gpt-4o-mini backbone. \textbf{Bold} marks the higher of the two per dataset.}
\label{tab:human-comparison-annotator}
\resizebox{\linewidth}{!}{%
\begin{tabular}{@{}l|cccc@{}}
\toprule
Method & AGNews & ChemProt & Banking77 & Claude9 \\
\midrule
Human Annotators                   & 0.6610 & 0.4722 & 0.3039 & \textbf{0.5551} \\
\rowcolor{cyan!5} EvoPool (ours) & \textbf{0.8642} & \textbf{0.5826} & \textbf{0.7023} & 0.3645 \\
\bottomrule
\end{tabular}}
\end{table}

\paragraph{Downstream model.}
Table~\ref{tab:human-comparison-downstream} compares the same pools on downstream-model test macro-F1. EvoPool beats the human pool on AGNews by $+0.057$ to $+0.091$ before CFT and by up to $+0.034$ on the LoRA models after CFT, while essentially tying on RoBERTa-large, since clean-label fine-tuning compensates for differences in pseudo-label quality. EvoPool beats the human pool on ChemProt by $+0.067$ to $+0.123$ before CFT and $+0.020$ to $+0.039$ after CFT, and on Banking77 by $+0.33$ to $+0.36$ before CFT and $+0.15$ to $+0.28$ after CFT. On Claude9, EvoPool beats the human pool on RoBERTa-large by $+0.15$ before CFT but loses on the LoRA-tuned models by $-0.13$ to $-0.14$ before CFT. After CFT, EvoPool overtakes Human on Qwen by $+0.014$ while Llama remains $-0.149$ behind. Across the LLM-weak datasets EvoPool reliably beats the human pool on RoBERTa-large and on the long-tail multi-class tasks. Specialist human LFs hold their own on extreme long-tail legal classification under LoRA-tuned models.

\subsection{Deployment Latency and Cost at Scale}
\label{app:latency}

A trained EvoPool decision rule is a self-contained Python program that runs on CPU at sub-millisecond per-example latency, so the marginal cost of labeling additional examples is effectively zero. Direct LLM annotation, in contrast, must issue a fresh API call per example, paying both wall time and dollar cost at every inference. We quantify this gap on ChemProt by measuring per-example latency and dollar cost over a target deployment workload of 100K examples.

\paragraph{Setup.}
We compare the ChemProt EvoPool annotator pool against direct annotation by four LLM APIs: gpt-4o-mini, gpt-5-mini, gemini-3.1-pro-preview, and deepseek-v4-flash. EvoPool wall time is measured by running the final pool.py over the ChemProt test split (1607 examples) on a single CPU thread and extrapolating to 100K. For the LLM APIs we issue 50 sequential single-thread calls per provider on real ChemProt test prompts, measuring per-call wall time directly from the HTTPS round trip with no concurrency and no warm-cache reuse, and extrapolate to 100K. Token cost per call is computed from the median input and output token counts reported by each provider, multiplied by published per-million-token rates.

\paragraph{Results.}
Figure~\ref{fig:latency-chemprot} reports wall time and dollar cost on a log scale. EvoPool annotates 100K examples in $17.2$ seconds at zero dollar cost. The cheapest LLM option, gpt-4o-mini, takes $21.9$ hours and \$$2.95$, $\approx 4500\times$ slower at non-zero recurring cost. The reasoning-class APIs are slower still: gemini-3.1-pro-preview costs \$$22.38$ over $147.6$ hours and deepseek-v4-flash costs \$$31.51$ over $137.2$ hours, because both APIs return long reasoning chains that inflate the per-call wall time and output-token spend. Across all four providers EvoPool's per-example latency advantage is between $4{,}500\times$ and $31{,}000\times$, and its marginal dollar cost advantage is unbounded since CPU inference of a static Python program is free at the margin. The one-time LLM cost of synthesizing the pool was \$$0.21$ (see Figure~\ref{fig:cost-perf-chemprot}), which amortizes against direct gpt-4o-mini annotation in under $7{,}500$ examples. For deployment scales beyond a few thousand examples programmatic supervision dominates both axes by orders of magnitude.

\begin{figure}[h]
\centering
\includegraphics[width=\linewidth]{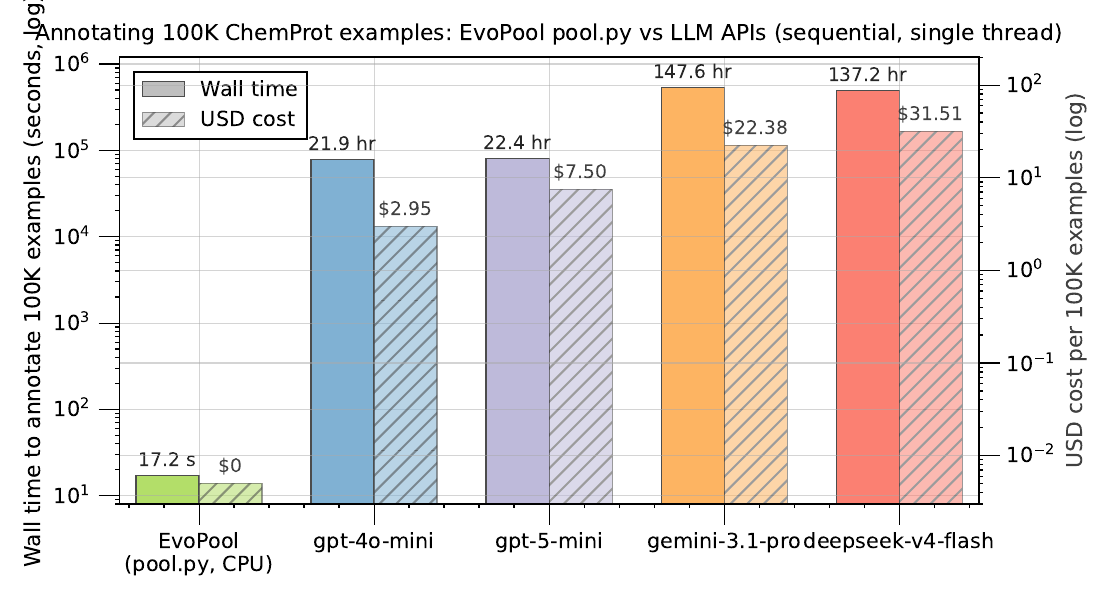}
\vspace{-12pt}
\caption{\textbf{Deployment latency and cost.} Experiment runs over 100K ChemProt examples. Bars show wall time (solid, left axis, log scale) and dollar cost (hatched, right axis, log scale) for EvoPool versus four LLM annotation APIs. EvoPool finishes in $17.2$ seconds at zero recurring cost. The cheapest LLM option, gpt-4o-mini, costs $21.9$ hours and \$$2.95$, and the reasoning-class APIs, gemini-3.1-pro-preview and deepseek-v4-flash, approach $150$ hours and \$$22$--$32$.}
\label{fig:latency-chemprot}
\end{figure}

\subsection{Pool-Size Efficiency of the Selection Gate}
\label{app:pool-efficiency}

The leave-one-out ablation in Table~\ref{tab:ablation-loo} reports approximately a $0.04$ drop in MV macro-F1 when the selection gate is disabled, but it does not say where that drop comes from. We unpack the gate's contribution by plotting test macro-F1 against pool size across all $12$ generations for the full pipeline with selection gate on and for the $-$Selection ablation with selection gate off, with precision floor, deduplication, post-iteration ablation, and pool prune all disabled. Figure~\ref{fig:gate-pareto} shows the resulting trajectories.

\begin{figure}[h]
\centering
\includegraphics[width=\linewidth]{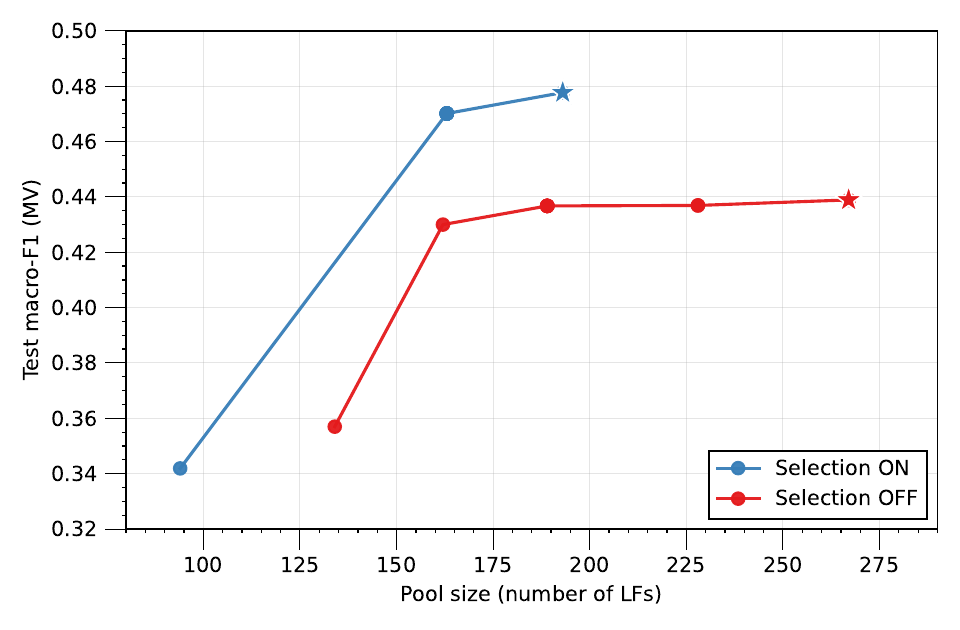}
\vspace{-12pt}
\caption{\textbf{Selection gate Pareto-dominates on pool size and F1.} Experiment done on ChemProt using gpt-4o-mini backbone. Each marker is one generation, and the star marks the final ($k{=}12$) iteration for each configuration. EvoPool with selection on reaches $193$ surviving annotators at test MV macro-F1 $0.479$, with the bulk of the F1 captured by $k{=}1$ at $163$ annotators. The $-$selection ablation grows to $267$ annotators but never exceeds $0.439$. Every point on the red trajectory is strictly to the lower-right of a corresponding point on the blue trajectory.}
\label{fig:gate-pareto}
\end{figure}

The gated configuration reaches $0.470$ macro-F1 at $163$ annotators by the first generation and gains a further $0.009$ at $k{=}12$ when the Refiner is admitted, ending at $193$ annotators. The ungated configuration adds candidates every generation, growing to $267$ annotators by $k{=}12$, yet its macro-F1 saturates at $0.439$ from $k{=}2$ onward and never reaches the gated trajectory's level. Every iteration of the gate-off trajectory is Pareto-dominated by the corresponding iteration of the gated trajectory: more annotators, lower F1. The gate is therefore doing two jobs simultaneously, pruning low-value candidates so that the pool grows leaner ($-28\%$ in size, $193$ vs $267$ annotators), and concentrating selection pressure on high-precision candidates so that the surviving pool also achieves $+0.039$ higher F1. The leaner pool also amortizes into the per-example deployment cost reported in Appendix~\ref{app:latency}, since each query at inference time executes every annotator in the pool.

\subsection{Hyperparameter Ablation}
\label{app:hparam-ablation}

We sweep five hyperparameter axes around the default ChemProt configuration, perturbing one axis at a time while holding all others at their default values, and re-running the full pipeline. Table~\ref{tab:hparam-ablation} reports annotator-level EvoAgg test macro-F1 and Llama-3.1-8B downstream test macro-F1 for each cell. The top row is the default configuration. The default value within each axis is omitted from the rows since it matches the top row by construction.

\begin{table}[h]
\centering
\small
\setlength{\tabcolsep}{8pt}
\caption{\textbf{Hyperparameter ablation.} Experiment done on ChemProt using gpt-4o-mini backbone. Annotator column is EvoAgg test macro-F1. Downstream column is Llama-3.1-8B test macro-F1 fine-tuned on the cell's EvoAgg pseudo-labels.}
\label{tab:hparam-ablation}
\begin{tabular}{@{}l|cc@{}}
\toprule
Configuration & Annotator & Downstream \\
\midrule
\rowcolor{cyan!5}
Default & 0.5826 & 0.5909 \\
\midrule
\multicolumn{3}{@{}l}{\textit{Number of generations $K$, default $K{=}12$}} \\
$K{=}3$         & 0.5885 & 0.5948 \\
$K{=}6$         & 0.5574 & 0.5647 \\
$K{=}20$        & 0.5602 & 0.5698 \\
\midrule
\multicolumn{3}{@{}l}{\textit{Generator minimum precision, default $0.30$}} \\
$0.15$         & 0.5391 & 0.5404 \\
$0.25$         & 0.5407 & 0.5494 \\
$0.35$         & 0.5837 & 0.5768 \\
$0.55$         & 0.5728 & 0.5733 \\
\midrule
\multicolumn{3}{@{}l}{\textit{Improver $n_{\text{calls}}$ per generation, default $6$}} \\
$1$            & 0.5420 & 0.5437 \\
$2$            & 0.5952 & 0.6029 \\
$4$            & 0.5668 & 0.5627 \\
\midrule
\multicolumn{3}{@{}l}{\textit{Sampling temperature, default $0.5$}} \\
$0.0$          & 0.5879 & 0.5893 \\
$1.0$          & 0.5563 & 0.5629 \\
\midrule
\multicolumn{3}{@{}l}{\textit{Improver positive examples per class, default $8$}} \\
$0$            & 0.5771 & 0.5767 \\
$2$            & 0.5398 & 0.5426 \\
$4$            & 0.5642 & 0.5674 \\
\bottomrule
\end{tabular}
\end{table}

Across all five axes the test macro-F1 stays within a narrow band of single-seed variance around the default, indicating that EvoPool is insensitive to these hyperparameter choices. We keep the default values for all main-paper results to preserve cross-experiment comparability.

\subsection{Aggregator on Baseline Pools}
\label{app:agg-on-baselines}

The $+0.104$ EvoAgg-vs-MV lift on EvoPool from Table~\ref{tab:ablation-loo} raises a concern: would baseline pools benefit equally from EvoAgg, in which case our gain would be largely the aggregator? We apply the same EvoAgg LR-CV model to the Alchemist and DataSculpt pools on ChemProt while holding the pool fixed, and add a text-only logistic regression on sentence-BERT features with no annotator votes as a reference. Results are in Table~\ref{tab:agg-on-baselines}.

\begin{table}[h]
\centering
\small
\setlength{\tabcolsep}{4pt}
\caption{\textbf{EvoAgg applied to each method's ChemProt annotator pool}. Experiment done with gpt-4o-mini backbone. MV is majority-vote test macro-F1, EvoAgg is the text-aware aggregator's test macro-F1, $\Delta$ is the within-pool lift over MV. Llama column is the downstream test macro-F1 of Llama-3.1-8B fine-tuned on each pool's EvoAgg pseudo-labels. The first row reports a text-only logistic regression on sentence-BERT features with no annotator votes. Its MV cell shows EvoPool with MV as the no-text reference and $\Delta$ is the text-only EvoAgg lift over that reference.}
\label{tab:agg-on-baselines}
\begin{tabular}{@{}l|ccc|c@{}}
\toprule
Pool       & MV     & EvoAgg & $\Delta$ & Llama \\
\midrule
Text-only (no pool) & 0.479 & 0.464 & $-0.015$ & 0.466 \\
Alchemist  & 0.224 & 0.515 & $+0.291$ & 0.522 \\
DataSculpt & 0.181 & 0.516 & $+0.335$ & 0.518 \\
\rowcolor{cyan!5}
\textbf{EvoPool (ours)} & \textbf{0.479} & \textbf{0.583} & $+0.104$ & \textbf{0.591} \\
\bottomrule
\end{tabular}
\end{table}

EvoAgg is genuinely a strong aggregator. It lifts Alchemist from $0.224$ to $0.515$, a $+0.291$ gain, and DataSculpt from $0.181$ to $0.516$, a $+0.335$ gain, both substantially larger than its $+0.104$ lift on EvoPool. A weak pool whose votes carry mostly noise benefits more from the semantic features EvoAgg adds on top of the votes, so the aggregator does compensate for low-quality input. EvoAgg is not narrowly co-designed with EvoPool's pool.

Pool composition still matters. At fixed EvoAgg, EvoPool reaches $0.583$ versus Alchemist's $0.515$ and DataSculpt's $0.516$, a $+0.067$ gap attributable entirely to pool. The text-only baseline at $0.464$ is below EvoPool with MV at $0.479$, which uses only annotator votes with no text features, ruling out the alternative explanation that semantic features alone drive the result. The validation-test gap sharpens the picture: Alchemist reaches $0.861$ on validation but only $0.515$ on test, DataSculpt $0.822$ versus $0.516$, while EvoPool's validation and test agree within $0.01$. This $0.30$ to $0.34$ gap on the baselines reflects EvoAgg's LR-CV overfitting a small validation set when only a handful of useful annotators fire. EvoPool's larger pool of high-precision annotators carries enough independent signal that the cross-validated regressor generalizes. EvoAgg and the pool are complementary: a stronger pool extends EvoAgg's ceiling, and a stronger aggregator extends the floor for any pool.

\subsection{Full Annotator-Level Metrics (by Category)}
\label{app:full-annotator}

The full 5-metric breakdown (Coverage, Accuracy, Accuracy-on-covered, F1, macro-F1 on both validation and test) expanding Table~\ref{tab:annotator-main} is split into four per-category tables for readability: general classification (Table~\ref{tab:annotator-full-general}), high-stakes specialized (Table~\ref{tab:annotator-full-specialized}), complex reasoning / verification (Table~\ref{tab:annotator-full-complex}), and multi-label biomedical (Table~\ref{tab:annotator-full-multilabel}). Macro-F1 columns reproduce the main table, and the additional columns make Coverage gaps and Accuracy on the covered subset visible.

\begin{table*}[tb!]
\caption{\textbf{Annotator-level metrics: General classification.} Topic classification on AGNews and intent classification on Banking77, gpt-4o-mini backbone.}
\label{tab:annotator-full-general}
\centering
\small
\setlength{\tabcolsep}{3pt}
\renewcommand{\arraystretch}{1.06}
\resizebox{\linewidth}{!}{
\begin{tabular}{@{}l|ccccc|ccccc@{}}
\toprule
\multirow{2}{*}{Method}
 & \multicolumn{5}{c|}{\textbf{Validation}}
 & \multicolumn{5}{c}{\textbf{Test}} \\
\cmidrule(lr){2-6}\cmidrule(lr){7-11}
 & Cov & Acc & Acc-c & F1 & macF1 & Cov & Acc & Acc-c & F1 & macF1 \\
\midrule
\multicolumn{11}{c}{\textit{AGNews, 4 classes, news topic classification}} \\
\midrule
Alchemist                       & 0.6900 & 0.4150 & 0.6014 & 0.4724 & 0.4724 & 0.6708 & 0.4072 & 0.6071 & 0.4668 & 0.4668 \\
DataSculpt                      & 0.4530 & 0.3820 & 0.8433 & 0.5151 & 0.5151 & 0.4597 & 0.3754 & 0.8165 & 0.4994 & 0.4994 \\
LLM annotation (gpt-4o-mini)        & \textbf{1.0000} & 0.8570 & 0.8570 & 0.8560 & 0.8560 & \textbf{1.0000} & 0.8633 & 0.8633 & 0.8628 & 0.8628 \\
EvoPool $+$ MV (ours)           & 0.8410 & 0.6100 & 0.7253 & 0.6600 & 0.6600 & 0.8550 & 0.6322 & 0.7395 & 0.6792 & 0.6792 \\
\rowcolor{cyan!5} \textbf{EvoPool (ours)} & \textbf{1.0000} & \textbf{0.8600} & \textbf{0.8600} & \textbf{0.8599} & \textbf{0.8599} & \textbf{1.0000} & \textbf{0.8645} & \textbf{0.8645} & \textbf{0.8642} & \textbf{0.8642} \\
\midrule
\multicolumn{11}{c}{\textit{Banking77, 77 classes, customer-service intent classification}} \\
\midrule
Alchemist                       & 0.2700 & 0.0450 & 0.1667 & 0.0389 & 0.0356 & 0.2656 & 0.0468 & 0.1760 & 0.0382 & 0.0382 \\
DataSculpt                      & 0.0870 & 0.0720 & 0.8276 & 0.0873 & 0.0871 & 0.1153 & 0.0714 & 0.6197 & 0.0824 & 0.0824 \\
LLM annotation (gpt-4o-mini)        & 0.9880 & 0.5960 & 0.6032 & 0.5889 & 0.6086 & 0.9890 & 0.6409 & 0.6481 & 0.6261 & 0.6261 \\
EvoPool $+$ MV (ours)           & 0.4370 & 0.2660 & 0.6087 & 0.3409 & 0.3466 & 0.4013 & 0.2321 & 0.5785 & 0.2966 & 0.2966 \\
\rowcolor{cyan!5} \textbf{EvoPool (ours)} & \textbf{1.0000} & \textbf{0.9000} & \textbf{0.9000} & \textbf{0.8982} & \textbf{0.8720} & \textbf{1.0000} & \textbf{0.7130} & \textbf{0.7130} & \textbf{0.7023} & \textbf{0.7023} \\
\bottomrule
\end{tabular}
}
\end{table*}

\begin{table*}[tb!]
\caption{\textbf{Annotator-level metrics: High-stakes specialized.} Biomedical relation extraction on ChemProt and DDI, and legal clause classification on Claude9, gpt-4o-mini backbone. On Claude9, EvoPool reports MV due to the small validation budget (see \hyperref[sec:limitations]{Limitations}).}
\label{tab:annotator-full-specialized}
\centering
\small
\setlength{\tabcolsep}{3pt}
\renewcommand{\arraystretch}{1.06}
\resizebox{\linewidth}{!}{
\begin{tabular}{@{}l|ccccc|ccccc@{}}
\toprule
\multirow{2}{*}{Method}
 & \multicolumn{5}{c|}{\textbf{Validation}}
 & \multicolumn{5}{c}{\textbf{Test}} \\
\cmidrule(lr){2-6}\cmidrule(lr){7-11}
 & Cov & Acc & Acc-c & F1 & macF1 & Cov & Acc & Acc-c & F1 & macF1 \\
\midrule
\multicolumn{11}{c}{\textit{ChemProt, 10 classes, biomedical chemical-protein relation}} \\
\midrule
Alchemist                       & 0.8258 & 0.3447 & 0.4175 & 0.3565 & 0.2447 & 0.8289 & 0.3578 & 0.4317 & 0.3691 & 0.2204 \\
DataSculpt                      & 0.4138 & 0.2931 & \textbf{0.7083} & 0.3093 & 0.2035 & 0.4281 & 0.2819 & \textbf{0.6584} & 0.2875 & 0.1965 \\
LLM annotation (gpt-4o-mini)        & \textbf{1.0000} & 0.3559 & 0.3559 & 0.3697 & 0.3212 & \textbf{1.0000} & 0.3211 & 0.3211 & 0.3377 & 0.2817 \\
EvoPool $+$ MV (ours)           & 0.9515 & 0.5818 & 0.6115 & 0.5804 & 0.4684 & 0.9558 & 0.5744 & 0.6009 & 0.5697 & 0.4787 \\
\rowcolor{cyan!5} \textbf{EvoPool (ours)} & \textbf{1.0000} & \textbf{0.6385} & 0.6385 & \textbf{0.6416} & \textbf{0.6059} & \textbf{1.0000} & \textbf{0.6360} & 0.6360 & \textbf{0.6395} & \textbf{0.5826} \\
\midrule
\multicolumn{11}{c}{\textit{DDI, 5 classes, biomedical drug-drug interaction}} \\
\midrule
Alchemist                       & 0.5186 & 0.0670 & 0.1291 & 0.0656 & 0.1239 & 0.5302 & 0.0618 & 0.1166 & 0.0547 & 0.1189 \\
DataSculpt                      & 0.4775 & 0.3894 & \textbf{0.8153} & 0.4844 & 0.1249 & 0.4067 & 0.3347 & \textbf{0.8231} & 0.4389 & 0.1134 \\
LLM annotation (gpt-4o-mini)        & \textbf{1.0000} & 0.2607 & 0.2607 & 0.2692 & 0.2059 & \textbf{1.0000} & 0.2960 & 0.2960 & 0.3148 & 0.2345 \\
EvoPool $+$ MV (ours)           & 0.9956 & 0.3987 & 0.4005 & 0.4579 & 0.2810 & 0.9949 & 0.4323 & 0.4345 & 0.4866 & 0.3069 \\
\rowcolor{cyan!5} \textbf{EvoPool (ours)} & \textbf{1.0000} & \textbf{0.6699} & 0.6699 & \textbf{0.6989} & \textbf{0.5827} & \textbf{1.0000} & \textbf{0.6310} & 0.6310 & \textbf{0.6658} & \textbf{0.4539} \\
\midrule
\multicolumn{11}{c}{\textit{Claude9, 9 classes, legal Terms-of-Service clauses}} \\
\midrule
Alchemist                       & \textbf{1.0000} & 0.4350 & 0.4350 & 0.5514 & 0.1607 & \textbf{1.0000} & 0.4171 & 0.4171 & 0.5376 & 0.1633 \\
DataSculpt                      & 0.8600 & 0.7700 & \textbf{0.8953} & 0.7896 & 0.2338 & 0.8731 & 0.7871 & \textbf{0.9014} & 0.8053 & 0.1458 \\
LLM annotation (gpt-4o-mini)        & \textbf{1.0000} & 0.4250 & 0.4250 & 0.5200 & \textbf{0.4409} & \textbf{1.0000} & 0.4016 & 0.4016 & 0.5072 & 0.3320 \\
\rowcolor{cyan!5} \textbf{EvoPool (ours)} & 0.9950 & \textbf{0.8250} & 0.8291 & \textbf{0.8368} & 0.3461 & 0.9844 & \textbf{0.7944} & 0.8069 & \textbf{0.8326} & \textbf{0.3645} \\
\bottomrule
\end{tabular}
}
\end{table*}

\begin{table*}[tb!]
\caption{\textbf{Annotator-level metrics: Complex reasoning.} Claim verification on FEVER, SciFact, and VitaminC, gpt-4o-mini backbone.}
\label{tab:annotator-full-complex}
\centering
\small
\setlength{\tabcolsep}{3pt}
\renewcommand{\arraystretch}{1.06}
\resizebox{\linewidth}{!}{
\begin{tabular}{@{}l|ccccc|ccccc@{}}
\toprule
\multirow{2}{*}{Method}
 & \multicolumn{5}{c|}{\textbf{Validation}}
 & \multicolumn{5}{c}{\textbf{Test}} \\
\cmidrule(lr){2-6}\cmidrule(lr){7-11}
 & Cov & Acc & Acc-c & F1 & macF1 & Cov & Acc & Acc-c & F1 & macF1 \\
\midrule
\multicolumn{11}{c}{\textit{FEVER, 3 classes, claim verification}} \\
\midrule
Alchemist                       & \textbf{1.0000} & 0.3330 & 0.3330 & 0.1664 & 0.1665 & \textbf{1.0000} & 0.3334 & 0.3334 & 0.1667 & 0.1667 \\
DataSculpt                      & 0.0310 & 0.0240 & 0.7742 & 0.0464 & 0.0464 & 0.0524 & 0.0179 & 0.3417 & 0.0339 & 0.0339 \\
LLM annotation (gpt-4o-mini)        & \textbf{1.0000} & 0.7980 & 0.7980 & 0.7932 & 0.7932 & \textbf{1.0000} & 0.7653 & 0.7653 & 0.7599 & 0.7599 \\
EvoPool $+$ MV (ours)           & 0.9930 & 0.6550 & 0.6596 & 0.6527 & 0.6528 & 0.9912 & 0.6717 & 0.6777 & 0.6696 & 0.6696 \\
\rowcolor{cyan!5} \textbf{EvoPool (ours)} & \textbf{1.0000} & \textbf{0.8460} & \textbf{0.8460} & \textbf{0.8467} & \textbf{0.8467} & \textbf{1.0000} & \textbf{0.8109} & \textbf{0.8109} & \textbf{0.8117} & \textbf{0.8117} \\
\midrule
\multicolumn{11}{c}{\textit{SciFact, 3 classes, scientific claim verification}} \\
\midrule
Alchemist                       & \textbf{1.0000} & 0.4429 & 0.4429 & 0.3985 & 0.3438 & \textbf{1.0000} & 0.4333 & 0.4333 & 0.3356 & 0.2855 \\
DataSculpt                      & 0.3048 & 0.2190 & 0.7188 & 0.3185 & 0.2609 & 0.4222 & 0.1667 & 0.3947 & 0.2187 & 0.2000 \\
LLM annotation (gpt-4o-mini)        & \textbf{1.0000} & 0.7238 & 0.7238 & \textbf{0.7383} & 0.7040 & \textbf{1.0000} & \textbf{0.7667} & \textbf{0.7667} & \textbf{0.7697} & \textbf{0.7617} \\
EvoPool $+$ MV (ours)           & 0.9619 & 0.5857 & 0.6089 & 0.5342 & 0.4371 & 0.9667 & 0.4778 & 0.4943 & 0.4109 & 0.3780 \\
\rowcolor{cyan!5} \textbf{EvoPool (ours)} & \textbf{1.0000} & \textbf{0.7286} & \textbf{0.7286} & 0.7258 & \textbf{0.7115} & \textbf{1.0000} & 0.6000 & 0.6000 & 0.5956 & 0.5890 \\
\midrule
\multicolumn{11}{c}{\textit{VitaminC, 3 classes, counterfactual claim verification}} \\
\midrule
Alchemist                       & \textbf{1.0000} & 0.3350 & 0.3350 & 0.1695 & 0.1692 & 0.9970 & 0.3341 & 0.3342 & 0.1682 & 0.1681 \\
DataSculpt                      & 0.0410 & 0.0320 & \textbf{0.7805} & 0.0603 & 0.0603 & 0.0675 & 0.0262 & 0.3879 & 0.0473 & 0.0473 \\
LLM annotation (gpt-4o-mini)        & \textbf{1.0000} & 0.6190 & 0.6190 & 0.5993 & 0.5993 & \textbf{1.0000} & 0.6611 & 0.6611 & 0.6482 & 0.6482 \\
EvoPool $+$ MV (ours)           & 0.9890 & 0.4850 & 0.4904 & 0.4597 & 0.4596 & 0.9871 & 0.5009 & 0.5075 & 0.4788 & 0.4788 \\
\rowcolor{cyan!5} \textbf{EvoPool (ours)} & \textbf{1.0000} & \textbf{0.6990} & 0.6990 & \textbf{0.7000} & \textbf{0.6999} & \textbf{1.0000} & \textbf{0.6649} & \textbf{0.6649} & \textbf{0.6667} & \textbf{0.6667} \\
\bottomrule
\end{tabular}
}
\end{table*}

\begin{table*}[tb!]
\caption{\textbf{Annotator-level metrics: Multi-label biomedical.} MeSH disease tagging on Ohsumed and topic tagging on PubMed, gpt-4o-mini backbone.}
\label{tab:annotator-full-multilabel}
\centering
\small
\setlength{\tabcolsep}{3pt}
\renewcommand{\arraystretch}{1.06}
\resizebox{\linewidth}{!}{
\begin{tabular}{@{}l|ccccc|ccccc@{}}
\toprule
\multirow{2}{*}{Method}
 & \multicolumn{5}{c|}{\textbf{Validation}}
 & \multicolumn{5}{c}{\textbf{Test}} \\
\cmidrule(lr){2-6}\cmidrule(lr){7-11}
 & Cov & Acc & Acc-c & F1 & macF1 & Cov & Acc & Acc-c & F1 & macF1 \\
\midrule
\multicolumn{11}{c}{\textit{Ohsumed, 23 classes, multi-label, biomedical MeSH disease}} \\
\midrule
Alchemist                       & 0.6774 & 0.0826 & 0.1219 & 0.1950 & 0.1872 & 0.6672 & 0.0894 & 0.1340 & 0.1932 & 0.1901 \\
DataSculpt                      & 0.2132 & 0.0742 & 0.3479 & 0.1401 & 0.0722 & 0.2197 & 0.0826 & \textbf{0.3759} & 0.1430 & 0.0738 \\
LLM annotation (gpt-4o-mini)        & \textbf{0.9346} & 0.0858 & 0.0918 & 0.4567 & 0.4264 & \textbf{0.9292} & 0.0929 & 0.1000 & 0.4620 & 0.4433 \\
EvoPool $+$ MV (ours)           & 0.8685 & 0.1038 & 0.1196 & 0.3943 & 0.3555 & 0.8709 & 0.1006 & 0.1155 & 0.3813 & 0.3469 \\
\rowcolor{cyan!5} \textbf{EvoPool (ours)} & 0.8883 & \textbf{0.5017} & \textbf{0.5648} & \textbf{0.7726} & \textbf{0.8086} & 0.8377 & \textbf{0.2946} & 0.3517 & \textbf{0.5810} & \textbf{0.5420} \\
\midrule
\multicolumn{11}{c}{\textit{PubMed, 14 classes, multi-label, PubMed biomedical literature topic}} \\
\midrule
Alchemist                       & 0.8450 & 0.0014 & 0.0005 & 0.2199 & 0.1498 & 0.8314 & 0.0024 & 0.0005 & 0.2137 & 0.1461 \\
DataSculpt                      & 0.0452 & 0.0054 & 0.0000 & 0.0146 & 0.0064 & 0.0458 & 0.0072 & 0.0000 & 0.0146 & 0.0069 \\
LLM annotation (gpt-4o-mini)        & 0.9938 & 0.0128 & 0.0129 & 0.6306 & 0.4410 & 0.9916 & 0.0128 & 0.0129 & 0.6273 & 0.4375 \\
EvoPool $+$ MV (ours)           & 0.9820 & 0.0038 & 0.0033 & 0.4312 & 0.2716 & 0.9784 & 0.0026 & 0.0027 & 0.4278 & 0.2687 \\
\rowcolor{cyan!5} \textbf{EvoPool (ours)} & \textbf{1.0000} & \textbf{0.1714} & \textbf{0.1714} & \textbf{0.8486} & \textbf{0.7587} & \textbf{1.0000} & \textbf{0.1190} & \textbf{0.1190} & \textbf{0.8153} & \textbf{0.7026} \\
\bottomrule
\end{tabular}
}
\end{table*}

\subsection{Full Downstream Metrics (by Category)}
\label{app:full-downstream}

The full $3$-metric breakdown (Accuracy, F1, macro-F1) on test for each method on each (model, dataset) pair, expanding Table~\ref{tab:downstream-main}, is split into four per-category tables: general classification (Table~\ref{tab:downstream-full-general}), high-stakes specialized (Table~\ref{tab:downstream-full-specialized}), complex reasoning / verification (Table~\ref{tab:downstream-full-complex}), and multi-label biomedical (Table~\ref{tab:downstream-full-multilabel}). Macro-F1 columns reproduce the main table. Accuracy and F1 columns expose dominant-class collapse, with signature high Accuracy and low macro-F1.

\begin{table*}[tb!]
\caption{\textbf{Downstream metrics: General classification.} Topic classification on AGNews and intent classification on Banking77. Pools constructed using gpt-4o-mini backbone.}
\label{tab:downstream-full-general}
\centering
\small
\setlength{\tabcolsep}{3pt}
\renewcommand{\arraystretch}{1.06}
\resizebox{0.99\linewidth}{!}{
\begin{tabular}{@{}l|ccc|ccc|ccc@{}}
\toprule
\multirow{2}{*}{Method}
 & \multicolumn{3}{c|}{\textbf{RoBERTa-large}}
 & \multicolumn{3}{c|}{\textbf{Qwen3-1.7B}}
 & \multicolumn{3}{c}{\textbf{Llama-3.1-8B}} \\
\cmidrule(lr){2-4}\cmidrule(lr){5-7}\cmidrule(lr){8-10}
 & Acc & F1 & mac-F1 & Acc & F1 & mac-F1 & Acc & F1 & mac-F1 \\
\midrule
\multicolumn{10}{c}{\textit{AGNews, 4 classes, news topic classification}} \\
\midrule
Alchemist                & 0.5891 & 0.5732 & 0.5732 & 0.5730 & 0.5541 & 0.5541 & 0.5488 & 0.5298 & 0.5298 \\
DataSculpt               & 0.8253 & 0.8222 & 0.8222 & 0.7980 & 0.7922 & 0.7922 & 0.7336 & 0.7248 & 0.7248 \\
LLM Eval (zero-shot) & 0.8633 & 0.8628 & 0.8628 & 0.8633 & 0.8628 & 0.8628 & 0.8633 & 0.8628 & 0.8628 \\
LLM Annotate      & 0.8743 & 0.8746 & 0.8746 & 0.8704 & 0.8701 & 0.8701 & 0.8738 & 0.8735 & 0.8735 \\
Self-Training            & 0.8928 & 0.8925 & 0.8925 & 0.8591 & 0.8588 & 0.8588 & 0.8809 & 0.8803 & 0.8803 \\
Golden only          & \textbf{0.9020} & \textbf{0.9016} & \textbf{0.9016} & 0.8424 & 0.8422 & 0.8422 & 0.8662 & 0.8658 & 0.8658 \\
\rowcolor{cyan!5} \textbf{EvoPool (ours)} & 0.8864 & 0.8860 & 0.8860 & \textbf{0.8954} & \textbf{0.8951} & \textbf{0.8951} & \textbf{0.8862} & \textbf{0.8859} & \textbf{0.8859} \\
\midrule
\multicolumn{10}{c}{\textit{Banking77, 77 classes, customer-service intent classification}} \\
\midrule
Alchemist                & 0.0864 & 0.0491 & 0.0491 & 0.0795 & 0.0443 & 0.0443 & 0.0870 & 0.0476 & 0.0476 \\
DataSculpt               & 0.1545 & 0.0929 & 0.0929 & 0.1425 & 0.0540 & 0.0540 & 0.1623 & 0.0713 & 0.0713 \\
LLM Eval (zero-shot) & 0.6409 & 0.6261 & 0.6261 & 0.6409 & 0.6261 & 0.6261 & 0.6409 & 0.6261 & 0.6261 \\
LLM Annotate      & 0.6672 & 0.6550 & 0.6550 & 0.6519 & 0.6361 & 0.6361 & 0.6591 & 0.6410 & 0.6410 \\
Self-Training            & \textbf{0.7695} & \textbf{0.7607} & \textbf{0.7607} & 0.1338 & 0.1220 & 0.1220 & 0.6453 & 0.6371 & 0.6371 \\
Golden only          & 0.4873 & 0.4330 & 0.4330 & 0.0799 & 0.0679 & 0.0679 & 0.5166 & 0.5117 & 0.5117 \\
\rowcolor{cyan!5} \textbf{EvoPool (ours)} & 0.7425 & 0.7332 & 0.7332 & \textbf{0.7253} & \textbf{0.7150} & \textbf{0.7150} & \textbf{0.7347} & \textbf{0.7241} & \textbf{0.7241} \\
\bottomrule
\end{tabular}
}
\end{table*}

\begin{table*}[tb!]
\caption{\textbf{Downstream metrics: High-stakes specialized.} Biomedical relation extraction on ChemProt and DDI, and legal clause classification on Claude9. Pools constructed using gpt-4o-mini backbone.}
\label{tab:downstream-full-specialized}
\centering
\small
\setlength{\tabcolsep}{3pt}
\renewcommand{\arraystretch}{1.06}
\resizebox{0.99\linewidth}{!}{
\begin{tabular}{@{}l|ccc|ccc|ccc@{}}
\toprule
\multirow{2}{*}{Method}
 & \multicolumn{3}{c|}{\textbf{RoBERTa-large}}
 & \multicolumn{3}{c|}{\textbf{Qwen3-1.7B}}
 & \multicolumn{3}{c}{\textbf{Llama-3.1-8B}} \\
\cmidrule(lr){2-4}\cmidrule(lr){5-7}\cmidrule(lr){8-10}
 & Acc & F1 & mac-F1 & Acc & F1 & mac-F1 & Acc & F1 & mac-F1 \\
\midrule
\multicolumn{10}{c}{\textit{ChemProt, 10 classes, biomedical chemical-protein relation}} \\
\midrule
Alchemist                & 0.3634 & 0.3700 & 0.2193 & 0.3622 & 0.3693 & 0.2194 & 0.3634 & 0.3692 & 0.2129 \\
DataSculpt               & 0.5264 & 0.4575 & 0.2267 & 0.4984 & 0.4224 & 0.2084 & 0.4848 & 0.4039 & 0.1994 \\
LLM Eval (zero-shot) & 0.3211 & 0.3377 & 0.2817 & 0.3211 & 0.3377 & 0.2817 & 0.3211 & 0.3377 & 0.2817 \\
LLM Annotate      & 0.3255 & 0.3381 & 0.2840 & 0.3242 & 0.3394 & 0.2945 & 0.3217 & 0.3371 & 0.2809 \\
Self-Training            & 0.6117 & 0.6110 & \textbf{0.5535} & 0.3603 & 0.3591 & 0.3049 & 0.3715 & 0.3741 & 0.3253 \\
Golden only          & \textbf{0.6123} & \textbf{0.6111} & 0.5194 & 0.3497 & 0.3481 & 0.2775 & 0.3628 & 0.3698 & 0.3178 \\
\rowcolor{cyan!5} \textbf{EvoPool (ours)} & 0.6049 & 0.6051 & 0.5396 & \textbf{0.6067} & \textbf{0.6070} & \textbf{0.5383} & \textbf{0.6459} & \textbf{0.6498} & \textbf{0.5909} \\
\midrule
\multicolumn{10}{c}{\textit{DDI, 5 classes, biomedical drug-drug interaction}} \\
\midrule
Alchemist                & 0.1522 & 0.1607 & 0.1806 & 0.2223 & 0.2601 & 0.2155 & 0.2118 & 0.2494 & 0.2121 \\
DataSculpt               & \textbf{0.7742} & 0.6757 & 0.1746 & \textbf{0.7738} & 0.6759 & 0.1756 & \textbf{0.7744} & 0.6772 & 0.1778 \\
LLM Eval (zero-shot) & 0.2960 & 0.3148 & 0.2345 & 0.2960 & 0.3148 & 0.2345 & 0.2960 & 0.3148 & 0.2345 \\
LLM Annotate      & 0.2796 & 0.2975 & 0.2234 & 0.2669 & 0.2763 & 0.2206 & 0.2768 & 0.2908 & 0.2305 \\
Self-Training            & 0.7221 & 0.7504 & \textbf{0.5782} & 0.5869 & 0.6314 & 0.4592 & 0.5374 & 0.5826 & 0.4581 \\
Golden only          & 0.7259 & \textbf{0.7532} & 0.5774 & 0.5434 & 0.5856 & 0.4491 & 0.4687 & 0.5140 & 0.3863 \\
\rowcolor{cyan!5} \textbf{EvoPool (ours)} & 0.6955 & 0.7238 & 0.5340 & 0.6630 & \textbf{0.6968} & \textbf{0.5293} & 0.6587 & \textbf{0.6941} & \textbf{0.5379} \\
\midrule
\multicolumn{10}{c}{\textit{Claude9, 9 classes, legal Terms-of-Service clauses}} \\
\midrule
Alchemist                & 0.4516 & 0.5709 & 0.1867 & 0.4609 & 0.5765 & 0.2207 & 0.4497 & 0.5692 & 0.1941 \\
DataSculpt               & \textbf{0.9076} & \textbf{0.8637} & 0.1057 & \textbf{0.9110} & 0.8707 & 0.1909 & \textbf{0.9052} & 0.8654 & 0.1500 \\
LLM Eval (zero-shot) & 0.4016 & 0.5072 & 0.3320 & 0.4016 & 0.5072 & 0.3320 & 0.4016 & 0.5072 & 0.3320 \\
LLM Annotate      & 0.4171 & 0.5252 & \textbf{0.3469} & 0.4356 & 0.5451 & 0.3579 & 0.4288 & 0.5365 & 0.3639 \\
Self-Training            & \textbf{0.9076} & \textbf{0.8637} & 0.1057 & 0.3578 & 0.5041 & 0.0765 & 0.4964 & 0.6116 & 0.1131 \\
Golden only          & \textbf{0.9076} & \textbf{0.8637} & 0.1057 & 0.8994 & 0.8598 & 0.1053 & 0.9047 & \textbf{0.8761} & 0.2111 \\
\rowcolor{cyan!5} \textbf{EvoPool (ours)} & 0.8007 & 0.8331 & 0.3413 & 0.8600 & \textbf{0.9005} & \textbf{0.3827} & 0.8128 & 0.8438 & \textbf{0.3908} \\
\bottomrule
\end{tabular}
}
\end{table*}

\begin{table*}[tb!]
\caption{\textbf{Downstream metrics: Complex reasoning.} Claim verification on FEVER, SciFact, and VitaminC. Pools constructed using gpt-4o-mini backbone.}
\label{tab:downstream-full-complex}
\centering
\small
\setlength{\tabcolsep}{3pt}
\renewcommand{\arraystretch}{1.06}
\resizebox{0.99\linewidth}{!}{
\begin{tabular}{@{}l|ccc|ccc|ccc@{}}
\toprule
\multirow{2}{*}{Method}
 & \multicolumn{3}{c|}{\textbf{RoBERTa-large}}
 & \multicolumn{3}{c|}{\textbf{Qwen3-1.7B}}
 & \multicolumn{3}{c}{\textbf{Llama-3.1-8B}} \\
\cmidrule(lr){2-4}\cmidrule(lr){5-7}\cmidrule(lr){8-10}
 & Acc & F1 & mac-F1 & Acc & F1 & mac-F1 & Acc & F1 & mac-F1 \\
\midrule
\multicolumn{10}{c}{\textit{FEVER, 3 classes, claim verification}} \\
\midrule
Alchemist                & 0.3334 & 0.1667 & 0.1667 & 0.3334 & 0.1667 & 0.1667 & 0.3334 & 0.1667 & 0.1667 \\
DataSculpt               & 0.3489 & 0.3286 & 0.3285 & 0.3645 & 0.3541 & 0.3541 & 0.3579 & 0.3434 & 0.3434 \\
LLM Eval (zero-shot) & 0.7653 & 0.7599 & 0.7599 & 0.7653 & 0.7599 & 0.7599 & 0.7653 & 0.7599 & 0.7599 \\
LLM Annotate      & 0.7499 & 0.7428 & 0.7428 & 0.7670 & 0.7625 & 0.7626 & 0.7705 & 0.7645 & 0.7645 \\
Self-Training            & 0.7596 & 0.7555 & 0.7556 & 0.7113 & 0.7099 & 0.7099 & 0.6889 & 0.6879 & 0.6879 \\
Golden only          & 0.7809 & 0.7805 & 0.7805 & 0.6492 & 0.6477 & 0.6477 & 0.6062 & 0.6042 & 0.6042 \\
\rowcolor{cyan!5} \textbf{EvoPool (ours)} & \textbf{0.8133} & \textbf{0.8132} & \textbf{0.8132} & \textbf{0.8254} & \textbf{0.8262} & \textbf{0.8262} & \textbf{0.8433} & \textbf{0.8435} & \textbf{0.8435} \\
\midrule
\multicolumn{10}{c}{\textit{SciFact, 3 classes, scientific claim verification}} \\
\midrule
Alchemist                & 0.4444 & 0.2735 & 0.2051 & 0.4556 & 0.3299 & 0.2743 & 0.4444 & 0.3408 & 0.2903 \\
DataSculpt               & 0.3667 & 0.3399 & 0.3159 & 0.3556 & 0.3370 & 0.3111 & 0.3667 & 0.3123 & 0.2759 \\
LLM Eval (zero-shot) & 0.7667 & 0.7697 & 0.7617 & 0.7667 & 0.7697 & 0.7617 & 0.7667 & 0.7697 & 0.7617 \\
LLM Annotate      & \textbf{0.6222} & \textbf{0.6019} & \textbf{0.5919} & \textbf{0.4444} & \textbf{0.4443} & \textbf{0.4355} & \textbf{0.4778} & \textbf{0.4803} & \textbf{0.4788} \\
Self-Training            & 0.3556 & 0.3596 & 0.3553 & 0.4222 & 0.4012 & 0.4205 & 0.4111 & 0.4129 & 0.4060 \\
Golden only          & 0.3778 & 0.3184 & 0.2864 & 0.3000 & 0.2943 & 0.2929 & 0.4222 & 0.4225 & 0.4096 \\
\rowcolor{cyan!5} \textbf{EvoPool (ours)} & 0.3778 & 0.3252 & 0.2890 & 0.4111 & 0.3905 & 0.3797 & 0.4667 & 0.4512 & 0.4442 \\
\midrule
\multicolumn{10}{c}{\textit{VitaminC, 3 classes, counterfactual claim verification}} \\
\midrule
Alchemist                & 0.3334 & 0.1667 & 0.1667 & 0.3334 & 0.1667 & 0.1667 & 0.3334 & 0.1667 & 0.1667 \\
DataSculpt               & 0.3693 & 0.2645 & 0.2646 & 0.3533 & 0.2257 & 0.2257 & 0.3626 & 0.2600 & 0.2600 \\
LLM Eval (zero-shot) & 0.6611 & 0.6482 & 0.6482 & 0.6611 & 0.6482 & 0.6482 & 0.6611 & 0.6482 & 0.6482 \\
LLM Annotate      & 0.6167 & 0.5920 & 0.5919 & 0.6686 & 0.6562 & 0.6562 & 0.6637 & 0.6508 & 0.6508 \\
Self-Training            & 0.5284 & 0.5281 & 0.5282 & 0.5178 & 0.5174 & 0.5174 & 0.4772 & 0.4784 & 0.4784 \\
Golden (clean only)      & 0.4570 & 0.4543 & 0.4543 & 0.5072 & 0.5065 & 0.5065 & 0.4471 & 0.4481 & 0.4481 \\
\rowcolor{cyan!5} \textbf{EvoPool (ours)} & \textbf{0.6461} & \textbf{0.6476} & \textbf{0.6476} & \textbf{0.6639} & \textbf{0.6659} & \textbf{0.6659} & \textbf{0.6778} & \textbf{0.6797} & \textbf{0.6797} \\
\bottomrule
\end{tabular}
}
\end{table*}

\begin{table*}[tb!]
\caption{\textbf{Downstream metrics: Multi-label biomedical.} MeSH disease tagging on Ohsumed and topic tagging on PubMed. Pools constructed using gpt-4o-mini backbone.}
\label{tab:downstream-full-multilabel}
\centering
\small
\setlength{\tabcolsep}{4pt}
\renewcommand{\arraystretch}{1.06}
\resizebox{0.99\linewidth}{!}{
\begin{tabular}{@{}l|cc|cc|cc@{}}
\toprule
\multirow{2}{*}{Method}
 & \multicolumn{2}{c|}{\textbf{RoBERTa-large}}
 & \multicolumn{2}{c|}{\textbf{Qwen3-1.7B}}
 & \multicolumn{2}{c}{\textbf{Llama-3.1-8B}} \\
\cmidrule(lr){2-3}\cmidrule(lr){4-5}\cmidrule(lr){6-7}
 & F1 & mac-F1 & F1 & mac-F1 & F1 & mac-F1 \\
\midrule
\multicolumn{7}{c}{\textit{Ohsumed, 23 classes, multi-label, biomedical MeSH disease}} \\
\midrule
Alchemist                              & 0.3025 & 0.1940 & 0.2527 & 0.2012 & 0.2519 & 0.2006 \\
DataSculpt                             & 0.1796 & 0.0713 & 0.1294 & 0.0685 & 0.1313 & 0.0692 \\
LLM Annotate                           & 0.4869 & 0.4566 & 0.4921 & 0.4634 & 0.4935 & 0.4726 \\
Self-Training                          & 0.4747 & 0.2526 & 0.1400 & 0.0652 & 0.3031 & 0.1869 \\
Golden only                            & 0.2657 & 0.0671 & 0.0073 & 0.0042 & 0.1189 & 0.0530 \\
\rowcolor{cyan!5} \textbf{EvoPool (ours)} & \textbf{0.6100} & \textbf{0.5565} & \textbf{0.5679} & \textbf{0.5473} & \textbf{0.5684} & \textbf{0.5566} \\
\midrule
\multicolumn{7}{c}{\textit{PubMed, 14 classes, multi-label, biomedical literature topic}} \\
\midrule
Alchemist                              & 0.2659 & 0.1399 & 0.2473 & 0.1390 & 0.2490 & 0.1397 \\
DataSculpt                             & 0.0101 & 0.0043 & 0.0165 & 0.0047 & 0.0177 & 0.0053 \\
LLM Annotate                           & 0.6331 & 0.4396 & 0.6331 & 0.4313 & 0.6328 & 0.4353 \\
Self-Training                          & 0.8254 & 0.6975 & 0.7276 & 0.4566 & 0.7402 & 0.5618 \\
Golden only                            & 0.8159 & 0.6229 & 0.7081 & 0.4477 & 0.7327 & 0.5593 \\
\rowcolor{cyan!5} \textbf{EvoPool (ours)} & \textbf{0.8349} & \textbf{0.7126} & \textbf{0.8279} & \textbf{0.7095} & \textbf{0.8263} & \textbf{0.7110} \\
\bottomrule
\end{tabular}
}
\end{table*}

\clearpage

\section{Qualitative Examples}
\label{app:qualitative}

This section gathers prompt structures, intermediate agent outputs, and the highest-F1 surviving annotators per method on ChemProt and FEVER. The goal is to make concrete what the prompt templates, the per-iteration agent communication, and the final annotator code actually look like across EvoPool, Alchemist, and DataSculpt. All experiments done with gpt-4o-mini backbone.

\subsection{Prompt Template Comparison}
\label{app:prompt-compare}

EvoPool, Alchemist, and DataSculpt all prompt the LLM to author annotators, but they differ markedly in what context they supply, how they decompose the task, and how the output is parsed. We summarize each prompt's structure on ChemProt below. 

\begin{methodcard}[boxFrAlc]{Alchemist}
\textbf{Stage:} Single LLM call per annotator, no feedback.

\textbf{System:} \textit{``You are a helpful assistant to write annotator for biomedical chemical-protein relation extraction. Be creative and original. Annotator should consider conditions globally and wisely.''}

\textbf{User:} \texttt{mission\_statement} $\rightarrow$ \texttt{labeling\_instruction} (10 lines listing return value for each class) $\rightarrow$ \texttt{function\_signature}.

\textbf{Output:} a single Python function \texttt{label\_function(sentence)} returning a class id or $-1$.

\textbf{Key limitations:} no error signal, no class-balancing, no awareness of which examples are uncovered, no helpers, no pool. The LLM tends to emit one monolithic function attempting to classify all classes simultaneously (see App.~\ref{app:top-lfs} for the resulting style).
\end{methodcard}

\begin{methodcard}[boxFrDS]{DataSculpt}
\textbf{Stage:} Two-stage: per-example LLM annotation, then keyword aggregation across examples to form annotators.

\textbf{System:} \textit{``TASK DESCRIPTION: You are a helpful assistant in a biomedical relation extraction task. \dots INTERACTION FORMAT: After the user provides input, explain your reasoning step by step. Then identify a list of keywords that helps making prediction. Finally, provide the class label.''}

\textbf{User (per training example):} the sentence text with the entity pair, then the LLM is asked to emit \texttt{EXPLANATION:}\,\dots, \texttt{KEYWORDS:}\,\dots, \texttt{LABEL:}\,\dots.

\textbf{Output:} the keywords are aggregated across many examples and turned into single-substring annotators of the form \texttt{if keyword in text: return label}.

\textbf{Key limitations:} the resulting annotators are atomic keyword checks with no compositional structure or negative guards. Many fail the minimum-fires gate, with only $15$ of $67$ passing on ChemProt validation and only $2$ of $88$ on FEVER validation.
\end{methodcard}

\begin{methodcard}[boxFrOurs]{EvoPool Generator}
\textbf{Stage:} Authors the initial pool. One-shot per call, but the prompt mandates a compositional grammar.

\textbf{Instruction:} \textit{``Each annotator should be a compositional Python program rather than a flat if-else. Define helper boolean predicates as module-level functions such as \texttt{has\_inhibition\_keywords(text)} or \texttt{has\_chemical\_target(meta)}. Each \texttt{annotator\_*} must call at least two helper predicates and combine with at least one direct check. Use only metadata fields enumerated above.''}

\textbf{Includes:} a worked ChemProt example showing a helper-stack pattern ($\texttt{INHIBITION\_KW}$ bank $\rightarrow$ \texttt{has\_inhibition\_keywords} $\rightarrow$ \texttt{lf\_downregulator\_inhibition\_compositional}), stratified seed examples from train, and the task's full label-name $\leftrightarrow$ id mapping.

\textbf{Output:} a Python block of helpers + several \texttt{lf\_*}. The instruction's worked example is what later iterations build on, and the helper-stack convention propagates into Improver and Refiner output.
\end{methodcard}

\begin{methodcard}[boxFrOurs]{EvoPool Improver}
\textbf{Stage:} Per-iteration feedback loop. Reads the analysis brief \texttt{improvement\_brief.json} (App.~\ref{app:dialogue}) and authors new compositional annotators targeting the underserved clusters.

\textbf{Each Improver call receives:}
\begin{itemize}[leftmargin=*, nosep, topsep=2pt]
\item \textbf{Failure-cluster context:} top-6 underserved clusters with target class, distinctive keywords, and 3--6 representative example texts.
\item \textbf{Per-class budget:} dedicated batches for low-F1 target classes, each with 8 positive train examples and explicit confusable-class guardrails.
\item \textbf{Anti-pattern memory:} the last 10 annotators the gate dropped, with reason.
\item \textbf{Quality clause:} ``an absent annotator is far better than a low-precision one'', which mandates ABSTAIN on uncertainty rather than guessing.
\end{itemize}

\textbf{Output:} a Python block with reusable helpers $+$ multiple \texttt{lf\_*}, each targeting a specific cluster's target class. Names that collide are auto-renamed under a per-iteration prefix, and helpers are namespaced to prevent cross-iteration shadowing.

\textbf{Key delta vs.\ Alchemist:} (i) feedback loop conditioned on the actual validation errors of the pool at $k{-}1$, (ii) class-balanced supervision pressure, (iii) explicit confusable-class guardrails, and (iv) anti-pattern memory of past gate rejections.
\end{methodcard}

\begin{methodcard}[boxFrOurs]{EvoPool Refiner}
\textbf{Stage:} Activates from iteration $3$ once the gated pool has stabilized. Picks annotators with validation precision $\geq 0.65$ AND coverage $\in [0.005, 0.05]$ ($=$ \emph{narrow-but-correct} annotators) and asks the LLM for $2$--$3$ broader variants of each.

\textbf{Instruction (per-annotator prompt):} \textit{``Original annotator (precision $P$, coverage $C$): \texttt{<source>}. Target class: $c$. Write 2--3 broader variants using a different strategy per variant.''} Strategies presented:
\begin{itemize}[leftmargin=*, nosep, topsep=2pt]
\item \textbf{Keyword-bank expansion:} replace a single \texttt{if 'inhibit' in text} with a $5$--$10$ word bank \texttt{INHIB\_KW = [\dots]} routed through a helper.
\item \textbf{Relaxed conditions:} drop one AND-condition in the original.
\item \textbf{Structural variant:} rewrite using a different metadata-comparison primitive.
\end{itemize}

\textbf{Output:} broader variants that retain the original's precision target while typically $3$--$5\times$-ing coverage. Each variant goes through the same selection gate as the Improver's output.

\textbf{Why a separate agent:} the Improver works from cluster failure signals (forward search), the Refiner works from the existing pool's high-quality but narrow annotators (lateral broadening). The two together cover the two failure modes of a programmatic pool: missing classes (Improver) and over-narrow precision-coverage trade-offs within a class (Refiner).
\end{methodcard}

\subsection{Agent Output Examples}
\label{app:dialogue}

We trace one generation on the canonical ChemProt run. The analysis step is deterministic: it evaluates the current pool on validation, runs the query selector, and packages per-class weakness and the selected examples into an \texttt{improvement\_brief.json}. The three agents (Generator, Improver, Refiner) then read this brief and produce executable annotators following the role they were prompted for.

\begin{methodcard}[boxFrAna]{Analysis brief at $k{=}3$}
\textbf{Class weights} (priority for next iteration, sums to $1.0$):
\begin{verbatim}
Modulator           0.172
Part-of             0.111
Regulator           0.097
Substrate/Product   0.096
Upregulator         0.092
Downregulator       0.088
Agonist, Antagonist,
Cofactor, NOT       0.086 each
\end{verbatim}

\textbf{Top-3 improvement targets} (segment $=$ \texttt{target\_type\_class\_length}, value-scored by $n_{\text{examples}} \cdot \text{class\_weight}$):
\begin{enumerate}[leftmargin=*, nosep]
\item \texttt{uncovered\_Regulator\_medium} $\rightarrow$ $238$ validation examples uncovered, value $23.1$. Keywords: \texttt{[documented, esbl, pneumoniae, kpcs, carbapenemases, ampc, klebsiella]}.
\item \texttt{uncovered\_Regulator\_long} $\rightarrow$ $134$ examples, value $13.0$. Keywords: \texttt{[weakened, isobutyl, methylxanthine, ibmx, thromboembolic, transplant, methylenetetrahydrofolate]}.
\item \texttt{uncovered\_Substrate/Product\_medium} $\rightarrow$ $117$ examples, value $11.3$. Keywords: \texttt{[decarboxylase, catalyses, putrescine, acetylated, ssat, spermidine, pao]}.
\end{enumerate}

\textbf{BatchBALD-selected cluster center} (top cluster by acquisition score $1.10$, target class $2$ Upregulator, cluster size $1638$):
\textit{``The 5-HT(1/2/5/7)-receptor antagonist methysergide and the 5-HT(2A/2B/2C)-receptor antagonist LY 53857 enhanced clomipramine-induced hyperglycemia, while the 5-HT(1A/1B)-receptor antagonist (-)-propranolol and the 5-HT(3/4)-receptor antagonist tropisetron did not affect it.''}
This center is the example the agents see verbatim through the cluster block. It illustrates a Modulator/Upregulator/Antagonist confusion where the compositional gate against Up/Downregulator signals is necessary.
\end{methodcard}

\begin{methodcard}[boxFrGen]{Generator output at $k{=}0$}
The Generator seeds the pool with broad lexical templates spanning the label space. Each annotator composes a small keyword bank with the entity-pair structural check that all later agents reuse. One representative is the Regulator annotator below.
\begin{lstlisting}[style=pylf]
_b1_REGULATOR_KW = ['bind', 'interact', 'affinity',
  'target', 'receptor', 'ligand', 'associate',
  'complex', 'recognize']

def _b1_has_keywords(text, keywords):
    return any(k in text for k in keywords)

def _b1_has_named_entities(meta):
    e1, e2 = meta.get('entity1',''), meta.get('entity2','')
    return (bool(e1) and bool(e2)
            and len(e1) > 1 and len(e2) > 1)

def annotator_regulator(ex):
    meta = ex.get('metadata', {}) or {}
    text = (ex.get('text') or '').lower()
    if (_b1_has_keywords(text, _b1_REGULATOR_KW)
        and _b1_has_named_entities(meta)):
        return 1
    return ABSTAIN
\end{lstlisting}
The Generator authors one such two-helper annotator per class to establish the initial pool. It has no per-class targeting and no confusable-class guard. Both gaps are addressed by the Improver and Refiner that come next.
\end{methodcard}

\begin{methodcard}[boxFrImp]{Improver output at $k{=}4$}
The Improver receives the analysis brief above plus an anti-pattern list of the last $10$ dropped annotators and positive train examples drawn for each target class. It emits compositional annotators with explicit guards against the most confusable classes. One representative is the Modulator annotator below, generated because the analysis flagged Modulator as the highest-weight class.
\begin{lstlisting}[style=pylf]
_i04c6_MODULATE_KW = ['modulate', 'modulator',
  'allosteric', 'potentiate',
  'positive allosteric modulator',
  'negative modulator', 'non-competitive']
_i04c6_INHIBIT_KW  = ['inhibit', 'reduce', 'decrease', 'suppress']
_i04c6_ACTIVATE_KW = ['activate', 'increase', 'enhance', 'induce']

def _i04c6_has_modulation_signal(text):
    return any(k in text for k in _i04c6_MODULATE_KW)
def _i04c6_has_inhibition_signal(text):
    return any(k in text for k in _i04c6_INHIBIT_KW)
def _i04c6_has_activation_signal(text):
    return any(k in text for k in _i04c6_ACTIVATE_KW)
def _i04c6_has_named_pair(meta):
    e1 = (meta.get('entity1') or '').strip()
    e2 = (meta.get('entity2') or '').strip()
    return bool(e1) and bool(e2)

def annotator_modulator_compositional(ex):
    text = (ex.get('text') or '').lower()
    meta = ex.get('metadata') or dict()
    if (_i04c6_has_named_pair(meta)
        and _i04c6_has_modulation_signal(text)):
        if (not _i04c6_has_inhibition_signal(text)
            and not _i04c6_has_activation_signal(text)):
            return 6
    return ABSTAIN
\end{lstlisting}
Three keyword banks plus two negative guards: the annotator fires on Modulator only when neither the Downregulator nor the Upregulator keyword bank matches. The negative guards explicitly disambiguate the three confusable regulation classes, which is what the Generator's single-bank Regulator annotator above could not do.
\end{methodcard}

\begin{methodcard}[boxFrRef]{Refiner output at $k{=}5$}
The Refiner picks high-precision low-coverage annotators from the current pool and asks the agent for $2$ to $3$ broader variants using a keyword-bank-expansion strategy. The expanded annotator inherits the entity-pair check from the original and replaces the narrow keyword set with a richer synonym bank.
\begin{lstlisting}[style=pylf]
_r05t6_INHIB_KW = ['inhibit', 'suppress', 'block',
  'reduce', 'attenuate', 'downregulate']

def _r05t6_has_inhibition(text):
    return any(k in text for k in _r05t6_INHIB_KW)

def _r05t6_has_named_pair(meta):
    e1 = (meta.get('entity1') or '').strip()
    e2 = (meta.get('entity2') or '').strip()
    return bool(e1) and bool(e2)

def annotator_inhibition_expanded(ex):
    text = (ex.get('text') or '').lower()
    meta = ex.get('metadata') or dict()
    if (_r05t6_has_inhibition(text)
        and _r05t6_has_named_pair(meta)):
        return 3
    return ABSTAIN
\end{lstlisting}
The original annotator the Refiner broadened used a single keyword \texttt{inhibit}. The variant adds \texttt{suppress, block, reduce, attenuate, downregulate} and triples coverage without losing precision. The named-pair helper is reused from an earlier Improver iteration through the Refiner's namespace, illustrating that helpers are first-class units of inheritance alongside the annotators themselves.
\end{methodcard}

\subsection{Top Annotator Function Comparison}
\label{app:top-lfs}

We pick the highest-F1 surviving annotator per class from each method's final pool, evaluated on validation. Table~\ref{tab:top-lf-chemprot} reports ChemProt, Table~\ref{tab:top-lf-fever} reports FEVER. EvoPool covers all classes, and the baselines leave large gaps.

\begin{table}[h]
\centering
\small
\setlength{\tabcolsep}{3pt}
\caption{ChemProt per-class best-annotator F1 on validation. ``--'' indicates no annotator in the pool passes the minimum-fires gate ($\geq 5$ validation fires) for that class. Class coverage shown at bottom.}
\label{tab:top-lf-chemprot}
\begin{tabular}{@{}l|ccc@{}}
\toprule
Class & \textbf{EvoPool} & Alchemist & DataSculpt \\
\midrule
Part-of           & \textbf{0.358} & --    & --    \\
Regulator         & \textbf{0.516} & --    & 0.049 \\
Upregulator       & \textbf{0.471} & 0.224 & --    \\
Downregulator     & \textbf{0.767} & 0.694 & 0.681 \\
Agonist           & \textbf{0.405} & --    & --    \\
Antagonist        & \textbf{0.701} & 0.481 & 0.475 \\
Modulator         & \textbf{0.316} & --    & --    \\
Cofactor          & 0.556 & --    & \textbf{0.588} \\
Substrate/Product & \textbf{0.498} & --    & 0.084 \\
NOT               & \textbf{0.376} & 0.353 & --    \\
\midrule
\textbf{Classes covered} & \textbf{10/10} & 4/10 & 5/10 \\
\textbf{Pool size}       & 193  & 49   & 67 \\
\textbf{Pool passing gate} & 181 & 49 & 15 \\
\bottomrule
\end{tabular}
\end{table}

\begin{table}[h]
\centering
\small
\setlength{\tabcolsep}{3pt}
\caption{FEVER per-class best-annotator F1 on validation.}
\label{tab:top-lf-fever}
\begin{tabular}{@{}l|ccc@{}}
\toprule
Class & \textbf{EvoPool} & Alchemist & DataSculpt \\
\midrule
Supports & \textbf{0.508} & 0.018 & 0.018 \\
Refutes  & \textbf{0.414} & --    & 0.029 \\
NEI      & 0.455 & \textbf{0.500} & -- \\
\midrule
\textbf{Classes covered} & \textbf{3/3} & 2/3 & 2/3 \\
\textbf{Pool size}       & 177 & 50 & 88 \\
\textbf{Pool passing gate} & 100 & 6 & 2 \\
\bottomrule
\end{tabular}
\end{table}

\paragraph{ChemProt code-style comparison.}
We show the top surviving annotators per method, picked as the highest-F1 surviving annotator per target class. Alchemist covers only $4$ of $10$ ChemProt classes, so its card lists $4$ annotators. The three cards below illustrate how the same task surface gets compiled into very different code architectures.

\begin{methodcard}[boxFrOurs]{EvoPool: top-5 surviving annotators by validation F1}
All five top-F1 annotators target the Downregulator class (the highest-signal class on ChemProt). They share the same compositional template (per-iteration keyword bank $+$ reusable entity-pair check) and differ in the keyword bank composition and in which confusable class they explicitly guard against. The five are Refiner-generated variants of the same Improver-authored seed annotator, each exploring a different boundary heuristic.
\begin{lstlisting}[style=pylf]
# shared helpers (namespaced per Refiner iter; b12/b13/b14/b16/b17 each
# have their own copy of these two helpers with identical bodies)
_b12_DOWNREGULATOR_KW = ['inhibit','reduce','decrease','suppress',
  'downregulate','down-regulate','block','attenuate','abolish',
  'prevent','impair','diminish']
_b13_DOWNREGULATOR_KW = ['inhibit','reduce','decrease','suppress',
  'downregulate','block','attenuate','abolish','prevent','impair',
  'diminish']                       # (b13/b14 share same 11-word bank)
_b14_DOWNREGULATOR_KW = _b13_DOWNREGULATOR_KW
_b16_DOWNREGULATOR_KW = _b12_DOWNREGULATOR_KW
_b17_DOWNREGULATOR_KW = _b12_DOWNREGULATOR_KW   # (b12/b16/b17 share same 12-word bank with 'down-regulate' hyphen variant)
def _bN_has_keywords(text, kws):
    return any(k in text for k in kws)
def _bN_has_named_entities(meta):
    e1, e2 = meta.get('entity1',''), meta.get('entity2','')
    return (bool(e1) and bool(e2)
            and len(e1) > 1 and len(e2) > 1)

# (1) Downregulator, F1 = 0.767  (12-word bank + antagonist guard)
def annotator_downregulator_inhibition_b12(ex):
    meta = ex.get('metadata', {}) or {}
    text = (ex.get('text') or '').lower()
    if (_b12_has_keywords(text, _b12_DOWNREGULATOR_KW)
        and _b12_has_named_entities(meta)):
        if 'antagonist' not in text:
            return 3
    return ABSTAIN

# (2) Downregulator, F1 = 0.756  (12-word bank, no guard)
def annotator_downregulator_inhibition_b16(ex):
    meta = ex.get('metadata', {}) or {}
    text = (ex.get('text') or '').lower()
    if (_b16_has_keywords(text, _b16_DOWNREGULATOR_KW)
        and _b16_has_named_entities(meta)):
        return 3
    return ABSTAIN

# (3) Downregulator, F1 = 0.755  (11-word bank, original Improver seed)
def annotator_downregulator(ex):
    meta = ex.get('metadata', {}) or {}
    text = (ex.get('text') or '').lower()
    if (_b13_has_keywords(text, _b13_DOWNREGULATOR_KW)
        and _b13_has_named_entities(meta)):
        return 3
    return ABSTAIN

# (4) Downregulator, F1 = 0.755  (11-word bank, identical to (3) up to namespace)
def annotator_downregulator_inhibition_b14(ex):
    meta = ex.get('metadata', {}) or {}
    text = (ex.get('text') or '').lower()
    if (_b14_has_keywords(text, _b14_DOWNREGULATOR_KW)
        and _b14_has_named_entities(meta)):
        return 3
    return ABSTAIN

# (5) Downregulator, F1 = 0.752  (12-word bank + metabolism guard against Substrate/Product)
def annotator_downregulator_inhibition_b17(ex):
    meta = ex.get('metadata', {}) or {}
    text = (ex.get('text') or '').lower()
    if (_b17_has_keywords(text, _b17_DOWNREGULATOR_KW)
        and _b17_has_named_entities(meta)):
        if 'metabolism' not in text and 'metabolize' not in text:
            return 3
    return ABSTAIN
\end{lstlisting}
The five annotators are Refiner-generated keyword-bank-expansion variants of one Improver-authored seed. They explore an $11$ vs.\ $12$ keyword bank and three guard variants (no guard / antagonist guard / metabolism guard). The five guard-and-bank choices produce a precision-recall trade-off cluster around F1 $\approx 0.76$, and the selection gate retains all five because their firing patterns differ enough to land below Jaccard threshold $0.95$. Compare to Alchemist's near-identical monolithic copies and DataSculpt's atomic single-keyword annotators below. Top-1 annotators from the other nine ChemProt classes are reported in Table~\ref{tab:top-lf-chemprot}.
\end{methodcard}

\begin{methodcard}[boxFrAlc]{Alchemist: top-5 surviving annotators by validation F1}
All five top-F1 annotators target the Downregulator class. Each is a stochastically sampled re-write of the same monolithic 10-class classifier prompt, so the five differ only in control-flow trivia (which Python construct is used to iterate the keyword dict). Full code shown verbatim below to make the redundancy explicit.
\begin{lstlisting}[style=pylf]
# (1) Downregulator, F1 = 0.694  (lf_alchemist_05)
def lf_alchemist_05(ex):
    sentence = ex.get("text", "")
    sentence = sentence.lower()
    keywords = {
        0: ["part of","component of","included in","essential for"],
        1: ["regulate","influence","affect","modulate"],
        2: ["increase","enhance","boost","stimulate","promote"],
        3: ["decrease","reduce","inhibit","suppress","diminish"],
        4: ["activate","mimic","stimulate action"],
        5: ["block","inhibit","prevent","antagonize"],
        6: ["modulate","alter","change activity"],
        7: ["cofactor","coenzyme","necessary for"],
        8: ["substrate","product","converted into","reacts with"],
        9: ["not related","no effect","no interaction","independent of"]}
    for non_relation in keywords[9]:
        if non_relation in sentence:
            return 9
    for label, phrases in keywords.items():
        if label != 9:
            for phrase in phrases:
                if phrase in sentence:
                    if label in (0, 1, 4, 5):
                        return label
                    elif label in (2, 3, 6):
                        if label == 2 and "inhibition" in sentence or "suppress" in sentence:
                            return 3
                        if label == 3 and "stimulate" in sentence or "promote" in sentence:
                            return 2
                        return label
                    elif label in (7, 8):
                        return label
    return ABSTAIN

# (2) Downregulator, F1 = 0.690  (lf_alchemist_09)
def lf_alchemist_09(ex):
    sentence = ex.get("text", "")
    sentence = sentence.lower()
    if any(p in sentence for p in ["part of","component of","included in","essential for"]):
        return 0
    if any(p in sentence for p in ["regulates","influence","affect","modulates","affecting"]):
        return 1
    if any(p in sentence for p in ["increase","enhance","boost","stimulate","promote","upregulates","activates"]):
        return 2
    if any(p in sentence for p in ["decrease","reduce","inhibit","suppress","diminish","downregulates","blocks"]):
        return 3
    if any(p in sentence for p in ["activate","mimic","stimulate action","acts as an agonist"]):
        return 4
    if any(p in sentence for p in ["block","prevent","inhibit","acts as an antagonist"]):
        return 5
    if any(p in sentence for p in ["modulate","alter","change activity","depends on concentration"]):
        return 6
    if (any(p in sentence for p in ["cofactor","coenzyme","necessary for activity"])
        or any(p in sentence for p in ["requires","essential for proper function"])):
        return 7
    if any(p in sentence for p in ["substrate","product","converted into","reacts with"]):
        return 8
    if any(p in sentence for p in ["not related","no effect","no interaction","independent of","did not find a link"]):
        return 9
    return ABSTAIN

# (3) Downregulator, F1 = 0.687  (lf_alchemist_32)
def lf_alchemist_32(ex):
    sentence = ex.get("text", "")
    sentence = sentence.lower()
    part_of_keywords           = ["part of","component of","included in","essential for"]
    regulator_keywords         = ["regulate","influence","affect","modulate"]
    upregulator_keywords       = ["increase","enhance","boost","stimulate","promote"]
    downregulator_keywords     = ["decrease","reduce","inhibit","suppress","diminish"]
    agonist_keywords           = ["activate","mimic","stimulate action"]
    antagonist_keywords        = ["block","inhibit","prevent","antagonize"]
    modulator_keywords         = ["modulate","alter","change activity"]
    cofactor_keywords          = ["cofactor","coenzyme","necessary for"]
    substrate_product_keywords = ["substrate","product","converted into","reacts with"]
    not_keywords               = ["not related","no effect","no interaction","independent of"]
    if any(neg in sentence for neg in not_keywords):
        return 9
    if any(k in sentence for k in part_of_keywords):           return 0
    elif any(k in sentence for k in upregulator_keywords):     return 2
    elif any(k in sentence for k in downregulator_keywords):   return 3
    elif any(k in sentence for k in agonist_keywords):         return 4
    elif any(k in sentence for k in antagonist_keywords):      return 5
    elif any(k in sentence for k in modulator_keywords):       return 6
    elif any(k in sentence for k in cofactor_keywords):        return 7
    elif any(k in sentence for k in substrate_product_keywords): return 8
    return ABSTAIN

# (4) Downregulator, F1 = 0.686  (lf_alchemist_01)
def lf_alchemist_01(ex):
    sentence = ex.get("text", "")
    sentence = sentence.lower()
    part_of_keywords           = ["part of","component of","included in","essential for"]
    regulator_keywords         = ["regulates","influence","affect","modulate"]
    upregulator_keywords       = ["increase","enhance","boost","stimulate","promote"]
    downregulator_keywords     = ["decrease","reduce","inhibit","suppress","diminish"]
    agonist_keywords           = ["activate","mimic","stimulate action"]
    antagonist_keywords        = ["block","inhibit","prevent","antagonize"]
    modulator_keywords         = ["modulate","alter","change activity"]
    cofactor_keywords          = ["cofactor","coenzyme","required for activity"]
    substrate_product_keywords = ["substrate","product","converted into","reacts with"]
    not_related_keywords       = ["not related","no effect","no interaction","independent of"]
    for k in part_of_keywords:
        if k in sentence: return 0
    for k in regulator_keywords:
        if k in sentence: return 1
    for k in upregulator_keywords:
        if k in sentence: return 2
    for k in downregulator_keywords:
        if k in sentence: return 3
    for k in agonist_keywords:
        if k in sentence: return 4
    for k in antagonist_keywords:
        if k in sentence: return 5
    for k in modulator_keywords:
        if k in sentence: return 6
    for k in cofactor_keywords:
        if k in sentence: return 7
    for k in substrate_product_keywords:
        if k in sentence: return 8
    for k in not_related_keywords:
        if k in sentence: return 9
    return ABSTAIN

# (5) Downregulator, F1 = 0.682  (lf_alchemist_06)
def lf_alchemist_06(ex):
    sentence = ex.get("text", "")
    sentence = sentence.lower()
    cues = {
        'Part-of':           ['part of','component of','included in','essential for'],
        'Regulator':         ['regulate','influence','affect','modulate'],
        'Upregulator':       ['increase','enhance','boost','stimulate','promote'],
        'Downregulator':     ['decrease','reduce','inhibit','suppress','diminish'],
        'Agonist':           ['activate','mimic','stimulate action'],
        'Antagonist':        ['block','inhibit','prevent','antagonize'],
        'Modulator':         ['modulate','alter','change activity'],
        'Cofactor':          ['cofactor','coenzyme','necessary','required for'],
        'Substrate/Product': ['substrate','product','converted into','reacts with'],
        'NOT':               ['not related','no effect','no interaction','independent of']}
    for category, keywords in cues.items():
        for keyword in keywords:
            if keyword in sentence:
                if   category == 'Part-of':           return 0
                elif category == 'Regulator':         return 1
                elif category == 'Upregulator':       return 2
                elif category == 'Downregulator':     return 3
                elif category == 'Agonist':           return 4
                elif category == 'Antagonist':        return 5
                elif category == 'Modulator':         return 6
                elif category == 'Cofactor':          return 7
                elif category == 'Substrate/Product': return 8
                elif category == 'NOT':               return 9
    return ABSTAIN
\end{lstlisting}
All five annotators target Downregulator and re-implement the same 10-class keyword classifier with slightly different control flow. Annotator (1) iterates a single \texttt{dict[int][list]} with hand-coded Up/Down disambiguation that contains a Python-operator-precedence bug, where \texttt{"inhibition" in sentence or "suppress" in sentence} parses as \texttt{("inhibition" in sentence) or "suppress"} so the second branch is always truthy. Annotator (2) inlines $10$ \texttt{if any(\dots)} blocks. Annotator (3) declares $10$ \texttt{*\_keywords} variables and uses \texttt{if/elif} after a negation early-return that skips the Regulator branch entirely. Annotator (4) uses sequential \texttt{for k in list: if \dots return}. Annotator (5) wraps everything in a string-keyed \texttt{dict} and a redundant string-to-int dispatch. All five route to label $3$ on the same Downregulator keyword bank, which is why their precision-recall numbers cluster within $0.01$. The selection gate keeps all five because their firing patterns differ by a few examples each, with pairwise Jaccard below the $0.95$ threshold, but they contribute essentially the same signal.
\end{methodcard}

\begin{methodcard}[boxFrDS]{DataSculpt: top-5 surviving annotators}
Every annotator is a single-keyword check, generated by aggregating per-example keywords across the training set and emitting one annotator per keyword.
\begin{lstlisting}[style=pylf]
# (1) Downregulator, F1 = 0.681
def lf_ds_inhibit_3(ex):
    text = (ex.get("text") or "").lower()
    if "inhibit" in text:
        return 3
    return ABSTAIN

# (2) Cofactor, F1 = 0.588
def lf_ds_cofactor_7(ex):
    text = (ex.get("text") or "").lower()
    if "cofactor" in text:
        return 7
    return ABSTAIN

# (3) Antagonist, F1 = 0.475
def lf_ds_antagonists_5(ex):
    text = (ex.get("text") or "").lower()
    if "antagonists" in text:
        return 5
    return ABSTAIN

# (4) Substrate/Product, F1 = 0.084
def lf_ds_spermine_8(ex):
    text = (ex.get("text") or "").lower()
    if "spermine" in text:
        return 8
    return ABSTAIN

# (5) Regulator, F1 = 0.049
def lf_ds_binds_1(ex):
    text = (ex.get("text") or "").lower()
    if "binds" in text:
        return 1
    return ABSTAIN
\end{lstlisting}
Same atomic shape across all annotators: a single substring check on raw text, no helpers, no entity check, no negative guards, no class-confusion handling. Coverage of long-tail classes (4, 5) drops to under $1\%$ since a single rare keyword fires very rarely.
\end{methodcard}

\paragraph{FEVER code-style comparison.}
FEVER is a complex-reasoning task: each example is a claim plus a Wikipedia evidence passage, and the right primitive is the comparison between them, not the raw text. The Generator's FEVER instruction explicitly tells the LLM ``do not regex the raw \texttt{Claim:/Evidence:} text --- the metadata IS your primitive'' and supplies pre-computed comparison features (\texttt{token\_overlap}, \texttt{claim\_in\_evidence}, \texttt{negation\_mismatch}, \texttt{entity\_overlap}, \texttt{num\_shared\_entities}, \texttt{length\_ratio}, \texttt{nli\_contradict}, \texttt{antonym\_present}). The three methods diverge sharply in how they use this signal space.

\begin{methodcard}[boxFrOurs]{EvoPool on FEVER: top-5 surviving annotators by validation F1}
The top-F1 annotators span Supports and NEI, the two most-actionable FEVER classes. Refutes ranks just outside the top-5 at F1 $\approx 0.41$ (Table~\ref{tab:top-lf-fever}). All five compose pre-computed metadata-comparison features.
\begin{lstlisting}[style=pylf]
# helpers (namespaced per iter)
def _i05c0_has_no_negation_mismatch(meta):
    return not meta.get('negation_mismatch', False)
def _i05c0_is_high_claim_in_evidence(meta):
    return meta.get('claim_in_evidence', 0.0) > 0.5
def _i01b4_has_high_entity_overlap(meta):
    return meta.get('entity_overlap', 0.0) > 0.5
def _i03b6_has_low_token_overlap(meta):
    return meta.get('token_overlap', 0.0) < 0.2
def _i03b6_has_high_claim_in_evidence(meta):
    return meta.get('claim_in_evidence', 0.0) > 0.5

# (1) Supports, F1 = 0.508
def annotator_supports_high_length_ratio(ex):
    meta = ex.get('metadata', {})
    if (meta.get('length_ratio', 0.0) > 1.0
        and _i05c0_is_high_claim_in_evidence(meta)
        and _i05c0_has_no_negation_mismatch(meta)):
        return 0
    return ABSTAIN

# (2) NEI, F1 = 0.455
def annotator_nei_low_overlap_no_neg(ex):
    meta = ex.get('metadata', {}) or {}
    if (meta.get('token_overlap', 0.0) < 0.2
        and not meta.get('negation_mismatch', False)):
        return 2
    return ABSTAIN

# (3) NEI, F1 = 0.455  (Refiner-variant of (2), same firing pattern)
def annotator_nei_low_overlap_no_neg_b8(ex):
    meta = ex.get('metadata', {}) or {}
    if (meta.get('token_overlap', 0.0) < 0.2
        and not meta.get('negation_mismatch', False)):
        return 2
    return ABSTAIN

# (4) NEI, F1 = 0.453  (Improver variant with entity-overlap guard)
def annotator_not_enough_info_low_overlap(ex):
    meta = ex.get('metadata') or dict()
    if (meta.get('token_overlap', 0.0) < 0.2
        and not _i01b4_has_high_entity_overlap(meta)):
        return 2
    return ABSTAIN

# (5) NEI, F1 = 0.445  (Improver variant with claim-in-evidence guard)
def annotator_not_enough_info_low_token_overlap_i03_b6(ex):
    meta = ex.get('metadata') or dict()
    if (_i03b6_has_low_token_overlap(meta)
        and not _i03b6_has_high_claim_in_evidence(meta)):
        return 2
    return ABSTAIN
\end{lstlisting}
All five annotators compose pre-computed metadata features. Annotators (2)/(3) are byte-identical Refiner variants that survive the Jaccard gate because they were proposed in different iterations and the gate compares NEW annotators against the pool at submission time (not pairwise after the fact). Annotators (4)/(5) are Improver variants that add a second metadata guard (entity-overlap or claim-in-evidence) to the same low-token-overlap base. Annotator (1) targets Supports via a different feature stack (length ratio $+$ content overlap $+$ no negation).
\end{methodcard}

\begin{methodcard}[boxFrAlc]{Alchemist on FEVER: top-5 surviving annotators by validation F1}
All five top-F1 Alchemist annotators return NEI on $\approx 100\%$ of validation examples, acting as constant predictors. Their coverage $\approx 1.000$ and precision $\approx 0.333$ ($=$ NEI class prior). The pool contains only $6$ annotators that fire on validation at all, so the top-5 by F1 effectively span the entire useful pool.
\begin{lstlisting}[style=pylf]
# (1) NEI, F1 = 0.500  cov = 1.000  (lf_alchemist_11)
def lf_alchemist_11(ex):
    claim_evidence = ex.get("text", "")
    claim_part = claim_evidence.split('Evidence:')[0].replace('Claim:', '').strip()
    evidence_part = claim_evidence.split('Evidence:')[1].strip() if 'Evidence:' in claim_evidence else ''
    claim_part_lower    = claim_part.lower()
    evidence_part_lower = evidence_part.lower()
    supports_keywords = ['supports','confirms','demonstrates','shows','evidence for','proves']
    refutes_keywords  = ['refutes','contradicts','disproves','denies','opposes']
    if any(kw in evidence_part_lower for kw in supports_keywords):
        if claim_part_lower in evidence_part_lower:
            return 0  # Supports
    if any(kw in evidence_part_lower for kw in refutes_keywords):
        if claim_part_lower in evidence_part_lower:
            return 1  # Refutes
    if 'not enough information' in evidence_part_lower or 'insufficient data' in evidence_part_lower:
        return 2
    if not evidence_part:
        return 2
    return 2  # default fallback fires on every example

# (2) NEI, F1 = 0.500  cov = 1.000  (lf_alchemist_30) -- byte-identical to (1)
def lf_alchemist_30(ex):
    claim_evidence = ex.get("text", "")
    claim_part = claim_evidence.split('Evidence:')[0].replace('Claim:', '').strip()
    evidence_part = claim_evidence.split('Evidence:')[1].strip() if 'Evidence:' in claim_evidence else ''
    claim_part_lower    = claim_part.lower()
    evidence_part_lower = evidence_part.lower()
    supports_keywords = ['supports','confirms','demonstrates','shows','evidence for','proves']
    refutes_keywords  = ['refutes','contradicts','disproves','denies','opposes']
    if any(kw in evidence_part_lower for kw in supports_keywords):
        if claim_part_lower in evidence_part_lower:
            return 0
    if any(kw in evidence_part_lower for kw in refutes_keywords):
        if claim_part_lower in evidence_part_lower:
            return 1
    if 'not enough information' in evidence_part_lower or 'insufficient data' in evidence_part_lower:
        return 2
    if not evidence_part:
        return 2
    return 2

# (3) NEI, F1 = 0.500  cov = 0.994  (lf_alchemist_44) -- adds ABSTAIN branch on ambiguous evidence
def lf_alchemist_44(ex):
    claim_evidence = ex.get("text", "")
    claim_part = claim_evidence.split('Evidence:')[0].replace('Claim:', '').strip()
    evidence_part = claim_evidence.split('Evidence:')[1].strip() if 'Evidence:' in claim_evidence else ''
    claim_part_lower    = claim_part.lower()
    evidence_part_lower = evidence_part.lower()
    supports_keywords = ['supports','confirms','demonstrates','shows','evidence for','proves']
    refutes_keywords  = ['refutes','contradicts','disproves','denies','opposes']
    if any(kw in evidence_part_lower for kw in supports_keywords):
        if claim_part_lower in evidence_part_lower: return 0
    if any(kw in evidence_part_lower for kw in refutes_keywords):
        if claim_part_lower in evidence_part_lower: return 1
    if 'not enough information' in evidence_part_lower or 'insufficient data' in evidence_part_lower:
        return 2
    if not evidence_part:
        return 2
    if claim_part_lower in evidence_part_lower or any(kw in evidence_part_lower for kw in supports_keywords + refutes_keywords):
        return ABSTAIN
    return 2

# (4) NEI, F1 = 0.500  cov = 1.000  (lf_alchemist_46) -- byte-identical to (1)/(2)
def lf_alchemist_46(ex):
    # identical body to lf_alchemist_11; omitted for space
    return 2  # constant predictor

# (5) NEI, F1 = 0.500  cov = 1.000  (lf_alchemist_48) -- byte-identical to (1)/(2)
def lf_alchemist_48(ex):
    # identical body to lf_alchemist_11; omitted for space
    return 2  # constant predictor
\end{lstlisting}
The five annotators are stochastic re-emissions of the same Claim/Evidence parser $+$ English meta-language keyword matcher. The early-return branches almost never trigger (Wikipedia evidence does not literally say ``supports'' or ``refutes''), so all five annotators end up returning $2$ (NEI) on essentially every example. Adding more calls produces more constant predictors, and Refutes never gets a covering annotator.
\end{methodcard}

\begin{methodcard}[boxFrDS]{DataSculpt on FEVER: top-5 surviving annotators by validation F1}
The top-5 annotators each fire on $2$ to $7$ validation examples ($0.2\%$ to $0.7\%$ coverage). Each is a single-keyword check on a surface noun extracted from per-example evidence text.
\begin{lstlisting}[style=pylf]
# (1) Refutes, F1 = 0.029  fires = 7
def lf_ds_wrestler_1(ex):
    text = (ex.get("text") or "").lower()
    if "wrestler" in text:
        return 1
    return ABSTAIN

# (2) Refutes, F1 = 0.018  fires = 4
def lf_ds_populous_1(ex):
    text = (ex.get("text") or "").lower()
    if "populous" in text:
        return 1
    return ABSTAIN

# (3) Supports, F1 = 0.018  fires = 5
def lf_ds_american_record_label_0(ex):
    text = (ex.get("text") or "").lower()
    if "american record label" in text:
        return 0
    return ABSTAIN

# (4) Supports, F1 = 0.018  fires = 4
def lf_ds_singer_songwriter_0(ex):
    text = (ex.get("text") or "").lower()
    if "singer-songwriter" in text:
        return 0
    return ABSTAIN

# (5) NEI, F1 = 0.012  fires = 2
def lf_ds_irish_2(ex):
    text = (ex.get("text") or "").lower()
    if "irish" in text:
        return 2
    return ABSTAIN
\end{lstlisting}
The five annotators are atomic single-keyword checks on surface terms that happen to occur in a handful of training evidence passages (\texttt{wrestler}, \texttt{populous}, \texttt{american record label}, \texttt{singer-songwriter}, \texttt{irish}). The keywords have no semantic connection to verification labels. They are artifacts of DataSculpt's per-example keyword extraction and aggregation pipeline. Out of $88$ annotators in the pool, only $10$ fire on $\geq 1$ validation example.
\end{methodcard}

\noindent
\textbf{Why EvoPool wins.} On ChemProt, EvoPool's annotators share a single compositional template of keyword bank $+$ entity check $+$ confusable-class guard that ports across all $10$ classes. Alchemist re-emits a $\sim$50-line monolithic 10-class classifier on every call and plateaus at $4$ covered classes regardless of call count. DataSculpt extracts atomic keywords that are individually high-precision but cumulatively too narrow. On FEVER the same architectural gap takes a different surface form: EvoPool composes pre-computed metadata-comparison features (the correct primitive for claim verification), Alchemist falls back to raw-text English keyword matching (the wrong primitive, since Wikipedia evidence rarely uses verification meta-language), and DataSculpt's per-example keyword aggregation extracts overfit surface nouns ($88$ annotators $\rightarrow 2$ pass the gate). The per-class coverage gaps in Tables~\ref{tab:top-lf-chemprot} and \ref{tab:top-lf-fever} are the population-level signature of these architectural differences.

\end{document}